\documentclass[11pt,letterpaper]{article}

\usepackage{amsmath,amsfonts,amsthm,amssymb,dsfont}
\usepackage{dsfont}
\usepackage{fullpage}
\usepackage{soul,color}
\usepackage{graphicx,float,wrapfig}
\usepackage[usenames,dvipsnames]{pstricks}
\usepackage{pst-grad} 
\usepackage{pst-plot} 
\usepackage{enumitem}

\usepackage[linesnumbered,boxed,ruled,vlined]{algorithm2e}
\title{On the Optimal Sample Complexity for Best Arm Identification}

\newtheorem{theo}{Theorem}[section]

\newtheorem{lemma}[theo]{Lemma}

\newtheorem{prop}[theo]{Proposition}
\newtheorem{cor}[theo]{Corollary}
\newtheorem{conj}[theo]{Conjecture}
\newtheorem{defi}[theo]{Definition}
\newtheorem{rem}[theo]{Remark}
\newtheorem{example}[theo]{Example}

\newenvironment{proofof}[1]{\begin{proof}[Proof of #1]}{\end{proof}}

\newcommand{\topic}[1]{\vspace{0.00cm}\noindent{\bf {#1}:}}

\newcommand{\R}{\mathbb{R}}
\newcommand{\N}{\mathbb{N}}
\newcommand{\Ex}{\mathbb{E}}
\renewcommand{\Pr}{\operatorname{Pr}}
\newcommand{\betw}{\mid}
\newcommand{\calF}{\mathcal{F}}

\newcommand{\bestarm}{\textsc{Best-$1$-Arm}}
\newcommand{\bestkarm}{\textsc{Best-$k$-Arm}}
\newcommand{\Problem}{\mathcal{P}}
\newcommand{\SIM}{\mathsf{SIM}}
\newcommand{\sign}{\textsc{Sign}-$\xi$}
\newcommand{\CORRECT}{$\delta$-correct}
\newcommand{\EXPCORRECTS}[1]{expected-$#1$-time  $\delta$-correct}
\newcommand{\WEAKCORRECTS}[1]{weakly $#1$-time  $\delta$-correct}
\newcommand{\WEAKEXPCORRECTS}[1]{weakly expected-$#1$-time $\delta$-correct}

\newcommand{\refseq}{\Lambda}
\newcommand{\distportion}{\mathfrak{C}}

\newcommand{\eat}[1]{}

\newcommand{\arm}{A}
\newcommand{\distr}{\mathcal{D}}
\newcommand{\alg}{\mathbb{A}}
\newcommand{\algp}{\mathbb{A}'}
\newcommand{\newalg}{\mathbb{B}}
\newcommand{\event}{\mathcal{E}}
\newcommand{\KL}{\mathrm{KL}}
\newcommand{\ent}{\mathrm{H}}
\renewcommand{\d}{\mathrm{d}}
\newcommand{\Normal}{\mathcal{N}}
\newcommand{\indicator}{\mathds{1}}

\newcommand{\lowb}{F}

\newcommand{\Gap}[1]{\Delta_{[#1]}}
\newcommand{\Gapepsilon}[1]{\Delta_{#1,\epsilon}}

\newcommand{\cnt}{\mathsf{cnt}}
\newcommand{\tot}{\mathsf{tot}}
\newcommand{\True}{\mathrm{True}}
\newcommand{\False}{\mathrm{False}}

\newcommand{\discdistr}{\Gamma^{\mathrm{d}}}

\newcommand{\amean}[1]{\mu_{[#1]}}
\newcommand{\hamean}[1]{\hat{\mu}_{[#1]}}

\newcommand{\granuf}{\rho}

\newcommand{\UNIFORMSAMPLING}{\textsf{UniformSample}}
\newcommand{\FRACTIONTEST}{\textsf{FractionTest}}
\newcommand{\ELIMINATION}{\textsf{Elimination}}
\newcommand{\MEDIANELIMINATION}{\textsf{MedianElim}}
\newcommand{\DELIMINATION}{\textsf{DistrBasedElim}}
\newcommand{\EXPGAPELIMINATION}{\textsf{ExpGapElim}}
\newcommand{\TESTSIGN}{\textsf{TestSign}}

\newcommand{\seqn}{\mathbb{S}}
\newcommand{\terround}{\kappa}

\renewcommand{\epsilon}{\varepsilon}

\newcommand{\lx}[1]{l_{#1}}
\newcommand{\Nsmall}{N_{\mathsf{sma}}}
\newcommand{\Pot}[1]{P_{#1}}
\newcommand{\curarm}{S_r^{=r}}
\newcommand{\curbigarm}{S_r^{>r}}
\newcommand{\Ncur}{N_{\mathsf{cur}}}
\newcommand{\Nbig}{N_{\mathsf{big}}}

\newcommand{\permute}[1]{#1_{\text{perm}}}
\newcommand{\palg}{{\permute{\alg}}}

\newcommand{\armcomp}{\mathcal{H}}
\newcommand{\bad}{\mathrm{bad}}

\newcommand{\nbag}{m}

\newcommand{\Iinit}{I_{\textsf{init}}}

\newcommand{\algone}[1]{\alg_{#1}^{1}}
\newcommand{\algtwo}[1]{\alg_{#1}^{2}}

\newcommand{\newI}{I_{\mathsf{new}}}
\newcommand{\newA}{A_{\mathsf{new}}}

\usepackage{lipsum}

\let\OLDthebibliography\thebibliography
\renewcommand\thebibliography[1]{
	\OLDthebibliography{#1}
	\setlength{\parskip}{2.5pt}
	\setlength{\itemsep}{0pt plus 0.3ex}
}

\date{}
\author{
	Lijie Chen \qquad\qquad Jian Li\\
	Institute for Interdisciplinary Information Sciences (IIIS), Tsinghua University
}

\begin{document}
	\maketitle

\vspace{-0.3cm}
	
	\begin{abstract}
	We study the best arm identification (\bestarm) problem, which is defined as follows.
	We are given $n$ stochastic bandit arms. The $i$th arm has a reward distribution $\distr_i$ with an unknown mean $\mu_{i}$. 
	Upon each play of the $i$th arm,
	we can get a reward, sampled i.i.d. from $\distr_i$. 
	We would like to identify the arm with the largest mean with probability at least $1-\delta$, using as few samples as possible.
	We provide a nontrivial algorithm for \bestarm,
	which improves upon several prior upper bounds on the same problem.
	We also study an important special case where there are only two arms, which we call the \sign\ problem. 
	We provide a new lower bound of \sign,
	simplifying and significantly extending a classical result by Farrell in 1964, with a completely new proof.
	Using the new lower bound for \sign, we obtain the first lower bound for \bestarm\ that 
	goes beyond the classic Mannor-Tsitsiklis lower bound, by an interesting reduction from \sign\ to \bestarm.
	We propose an interesting conjecture concerning the 
	optimal sample complexity of \bestarm\ 
	from the perspective of instance-wise optimality.
	\end{abstract}
	
	\vspace{-0.4cm}


\newcommand{\bB}{\mathbb{B}}

\section{Introduction}
\label{sec:intro}

\vspace{-0.2cm}

	The stochastic multi-armed bandit is a paradigmatic model  
	for capturing the exploration-exploitation tradeoff in many decision-making problems in stochastic environments.
	While the most studied goal is to maximize the cumulative rewards (or minimize
	the cumulative regret) obtained by the forecaster (see e.g.,~\cite{cesa2006prediction,bubeck2012regret}), the 
	{\em pure exploration multi-armed bandit} problems, in which the exploration phase and exploitation phase are separate, have also attracted significant 
	attentions, due to their applications in   
	several domains	such as medical trials \cite{robbins1985some,audibert2010best,chen2014combinatorial},
	communication network \cite{audibert2010best},
	crowdsourcing \cite{zhou2014optimal,cao2015top}.
	In a pure exploration problem,			
	the forecaster first performs a \emph{pure-exploration phase}, 
	by (adaptively) drawing samples from the stochastic arms, to infer 
	the optimal (or near optimal) solution, then
	keeps exploiting this solution.
	In this paper, we study the best arm identification (\bestarm) problem, 
	which is the most basic pure exploration problem 
	in stochastic multi-armed bandits.

	\begin{defi}
	 \bestarm :  We are given $n$ arms $\arm_1,\ldots, \arm_n$.
	 The $i$th arm $\arm_i$ has a reward distribution $\distr_i$ with an unknown mean $\mu_{i}\in [0,1]$. 
	 We assume that all reward distributions have $1$-sub-Gaussian tails (see Definition~\ref{defi:sub-gaussian-tail}),
	 which is a standard assumption in the stochastic multi-armed bandit literature.   
	 Upon each play of $\arm_i$,
	 we can get a reward value sampled i.i.d. from $\distr_i$. Our goal is to identify the arm with largest mean using as few samples as possible.
	 We assume here that the largest mean is strictly larger than the second largest (i.e., $\mu_{[1]}>\mu_{[2]}$)
	 to ensure the uniqueness of the solution,
	 where $\mu_{[i]}$ denotes
	 the $i$th largest mean.
	\end{defi}


	We also study the following sequential testing problem, named \sign,
	which is important for understanding \bestarm.
	\begin{defi}	
	\sign: $\xi$ is a fixed constant. We are given a single arm with unknown mean $\mu \ne \xi$.
	The goal is to decide whether $\mu > \xi$ or $\mu < \xi$. 
	Here, the gap of the problem is define to be $\Delta = |\mu - \xi|$. 
	Again, we assume that the distribution of the arm is $1$-sub-Gaussian.
	\end{defi}

	In fact, \sign\ can be viewed as a special case of \bestarm\ where there are only two arms
	and we know the mean of one arm.
	Hence, a thorough understanding of the sample complexity of \sign\ is very useful 
	for deriving tight sample complexity bounds for
	\bestarm. 
	
	

	\begin{defi}
	For a fixed value $\delta\in (0,1)$, we say that an algorithm $\alg$ for \bestarm\,\,  (or \sign) is \CORRECT, if given any \bestarm  (or \sign) instance, 
	$\alg$ returns the correct answer with probability at least $1-\delta$. 

	We say that an algorithm $\alg$ for \bestarm\ is an $(\epsilon,\delta)$-PAC algorithm, 
	if given any \bestarm\ instance 
	and any confidence level $\delta>0$, $\alg$ returns an $\epsilon$-optimal arm with probability at least $1-\delta$. 
	Here we say an arm $A_i$ is $\epsilon$-optimal if $\amean{1}-\mu_i\leq \epsilon$.
	\end{defi}

	The studies of both problems have a long history dating 
	back to 1950s \cite{bechhofer1954single,paulson1964sequential,farrell1964asymptotic}.
	We first discuss the \sign\ problem.
	It is well known that 
	for any \CORRECT\ algorithm (for constant $\delta$) $\alg$ that can distinguish two Gaussian arms
	with means $\xi+\Delta$ and $\xi-\Delta$ (the values of $\xi$ and $\Delta$ are known beforehand), 
	the expected number of samples required by $\alg$ is $\Omega(\Delta^{-2})$, 
	which is optimal (e.g., \cite{Chernoff:72}).
	This can be seen as a lower bound for \sign\ as well.
	However, a tighter lower bound of \sign\ was in fact provided  
	by Farrell in 1964~\cite{farrell1964asymptotic}.
	He showed that for any \CORRECT\ $\alg$ for \sign~(where the reward distribution is in the exponential family), 
	it holds that
	\begin{align}
	\label{eq:2armlowerbound}
	\limsup_{\Delta \to 0} \frac{T_{\alg}[\Delta]}{\Delta^{-2} \ln \ln \Delta^{-1}} > 0,
	\end{align}
	where $T_{\alg}[\Delta]$ is the expected number of samples taken by $\alg$ on an instance with gap $\Delta$.
	Farrell's result crucially relies on the {\em Law of Iterated Logarithm (LIL)},
	which roughly states that $\limsup_t \Bigl|\sum\nolimits_{i=1}^t X_i\Bigr|/\sqrt{2t\log\log t}=1$ almost surely where 
	$X_i\sim \Normal(0,1)$ for all $i$.
	Comparing with the $\Omega(\Delta^{-2})$ lower bound, the extra $\ln \ln \Delta^{-1}$ factor is caused by
	the fact that we do not known the gap $\Delta$ beforehand.
	The above result implies that $\Gap{2}^{-2} \ln \ln \Gap{2}^{-1}$ is also a lower bound for \bestarm\ (for two arms). 

	
	
	Bechhofer~\cite{bechhofer1954single} formulated the \bestarm\ problem for Gaussians in 1954.
	The early advances are summarized in the monograph~\cite{bechhofer1968sequential}.
	The last decade has witnessed a resurgence of interest in the \bestarm\ problem and its optimal sample complexity~\cite{even2002pac,gabillon2012best,kalyanakrishnan2012pac,2013arXiv1306.3917J}.
	\cite{mannor2004sample} showed that for any \CORRECT\ algorithm for \bestarm, 
	it requires $\Omega\left(\sum\nolimits_{i=2}^{n} \Gap{i}^{-2} \ln\delta^{-1}\right)$ samples in expectation for any instance.
	We note that the Mannor-Tsitsiklis lower bound is an {\em instance-wise lower bound}, 
	i.e., any \bestarm\ instance requires the stated number of samples.
	The current best known bound is
	$O\left(\sum\nolimits_{i=2}^{n} \Gap{i}^{-2} \left(\ln\ln\Gap{i}^{-1}+\ln\delta^{-1}\right)\right)$,
	due to Karnin et al.~\cite{karnin2013almost}.
	Jamieson et al.~\cite{jamieson2014lil} obtained a UCB-type algorithm (called lil'UCB), which achieves
	the same sample complexity and is also efficient in practice. 
	We refer the above bound as the KKS bound.
	See Table~1 for more previous upper bounds. 
		

	Given Farrell's $\Gap{2}^{-2} \ln \ln \Gap{2}^{-1}$ lower bound, 
	it is very attempting to 
	believe that the KKS upper bound is optimal,
	(which matches the lower bound for two arms). 
		Both \cite{jamieson2014lil} and \cite{jamieson2014best} explicitly referred the  KKS bound as ``optimal'', and
		\cite{jamieson2014lil} stated that ``{\em The procedure cannot be improved in the sense that the number
		of samples required to identify the best arm is within a constant factor of a lower bound based
		on the law of the iterated logarithm (LIL)}''.
	However, as we will demonstrate, none of the existing lower and upper bounds are optimal and the problem is more complicated (and rich) than we expected.
	The KKS bound is tight only for two arms (or $O(1)$ arms) and in the worst case sense (not instance optimal in the sense of \cite{fagin2003optimal,afshani2009instance}).


\begin{table}[]
	\centering
	\label{table:pure-exp}
	\small{
	\begin{tabular}{|l|l|}
		\hline
		Source & Sample Complexity                                                             \\ \hline
		Even-Dar et al.~\cite{even2002pac}
		& $\sum\nolimits_{i=2}^{n} \Gap{i}^{-2} \left(\ln\delta^{-1}+\ln n + \ln \Gap{i}^{-1}\right)$                 \\ \hline
		Gabillon et al.~\cite{gabillon2012best} 
		&
		$\sum\nolimits_{i=2}^{n} \Gap{i}^{-2} \left(\ln\delta^{-1}+\ln \sum\nolimits_{i=2}^{n} \Gap{i}^{-2}\right)$
			       \\ \hline
		Jamieson et al.~\cite{2013arXiv1306.3917J}
		& 
		$\sum\nolimits_{i=2}^n \Gap{i}^{-2} \left(\ln\delta^{-1}+\ln\ln \left(\sum\nolimits_{j=2}^{n} \Gap{j}^{-2}\right)\right)$
		                                                       \\ \hline
		kalyanakrishnan et al.~\cite{kalyanakrishnan2012pac}
		& 
		$\sum\nolimits_{i=2}^{n} \Gap{i}^{-2} \left(\ln\delta^{-1}+\ln \sum\nolimits_{i=2}^{n} \Gap{i}^{-2}\right)$
	                                                         \\ \hline
		Jamieson et al.~\cite{2013arXiv1306.3917J}
		& 
		$ \ln \delta^{-1} \cdot \left(\ln\ln\delta^{-1} \cdot \sum\nolimits_{i=2}^{n} \Gap{i}^{-2} + \sum\nolimits_{i=2}^{n} \Gap{i}^{-2} \ln \Gap{i}^{-1}  \right)$
	                                                          \\ \hline
		Karnin et al.\cite{karnin2013almost}, Jamieson et al.\cite{jamieson2014lil}
		& 
		$\sum\nolimits_{i=2}^{n} \Gap{i}^{-2} \left(\ln\delta^{-1}+\ln\ln\Gap{i}^{-1}\right)$
		                                                     \\ \hline
		This paper (Thm~\ref{theo:BESTARMPACD})
		& 
		$\sum\nolimits_{i=2}^{n} \Gap{i}^{-2} \left(\ln \delta^{-1} + \ln\ln\min(n,\Gap{i}^{-1}) \right) + \Gap{2}^{-2} \ln\ln \Gap{2}^{-1}$
	        \\ \hline
		This paper (clustered instances) Thm~\ref{thm:clustered}
		& 
		$\sum\nolimits_{i=2}^{n} \Gap{i}^{-2}\ln \delta^{-1} +\Gap{2}^{-2}\ln\ln\Gap{2}^{-1}$
	                             \\ \hline
	\end{tabular}
	}
	\vspace{0.1cm}
	\caption{Sample complexity upper bounds. We omit the big-O notations. 
		The definition of clustered instances can be found in 
		the supplementary material (Section~\ref{subsec:clustered})
		}.
\end{table}


	   \vspace{-0.3cm}
\subsection{Our Contributions}
\label{subsec:contributions}
	   \vspace{-0.3cm}
	   
	We need some notations to state our results formally.
	Let $\{\mu_{[1]}, \mu_{[2]},\ldots, \mu_{[n]}\}$
	be the means of the $n$ arms, sorted in the nondecreasing order
	(ties are broken in an arbitrary but consistent manner).
	We use $\arm_{[i]}$ to denote the arm with mean $\mu_{[i]}$.
	In the \bestarm\ problem, we 
	define the {\em gap} for the arm $\arm_{[i]}$ to be 
	$\Gap{i} = \amean{1}-\amean{i}$,
	which is an important quantity to measure the sample complexity.
	We use $\alg$ to denote an algorithm   
    and $T_{\alg}(I)$ to be the expected number of total arm pulls (i.e., samples) by $\alg$ on the instance $I$.

	    \subsubsection{Upper Bounds}

		First, we consider the upper bounds for the \bestarm\ problem with $n$ arms.
		We provide a novel algorithm which strictly improves the KKS bound,
		implying that it is not optimal.
		\eat{
		Finally, we return to the question whether $\sum\nolimits_{i=2}^{n} \Gap{i}^{-2} \ln\ln \Gap{i}^{-1}$
		is the right lower bound for \bestarm, which was in fact the initial motivation of our work.
		In terms of instance optimality, the answer is certainly no in light of our previous discussion of \sign.
		However, even in the sense of ``almost instance optimality'',
		the answer is still negative, unless there are only $O(1)$ arms. 
		We achieve this by providing a nontrivial (worst case) optimal algorithm,
		which performs strictly better than the KKS-JMNS bound
		for very wide class of instances (those such that $\ln\ln \Gap{i}^{-1}\ll \ln\ln n$ and the first term does not dominate).
		}
	    In particular, for any $\delta <0.1$, our algorithm is \CORRECT\ for \bestarm\
	    and it needs at most  
	    $$
	    O\Big(
	    \Gap{2}^{-2}\ln\ln \Gap{2}^{-1}+
	    \sum\nolimits_{i=2}^{n} \Gap{i}^{-2} \ln\delta^{-1}+\sum\nolimits_{i=2}^{n} \Gap{i}^{-2}\ln\ln \min(n,\Gap{i}^{-1}) 
	    \Big)
	    $$ 
	    samples in expectation. 
	    We can see that the improvement over KKS bound is mainly due to 
	    the third term. 
	    At first glance, the $\ln\ln n$ factor may seem to be an artifact of either 
	    our algorithm or analysis.
	    However, it turns out to be a fundamental quantity in the \bestarm\ problem,
	    since we can also prove a $\Omega(\sum_{i=2}^{n} \Gap{i}^{-2} \ln\ln n)$
	    lower bound (Theorem~\ref{thm:hard-case-exists}).
		Moreover, the first two terms in our upper bounds are also necessary
		(first term due to the lower bound by \cite{farrell1964asymptotic}, 
		second term due to the lower bound by \cite{mannor2004sample}). 
	
		Theorem~\ref{theo:BESTARMPACD} has a few interesting consequences we would like to point out.
   		For example, it is not possible to construct a class of infinite instances
		that requires $\Omega(n \Gap{2}^{-2} \ln\ln \Gap{2}^{-1})$ samples unless $\ln\ln \Gap{2}^{-1}=O(\ln\ln n)$. 
		This is somewhat surprising: 
		Consider a very basic family of instances in this class:
		there are $n-1$ arms with mean 0.5
		and 1 arm with mean $0.5+\Delta$. 	
		The Mannor-Tsitsiklis lower bound $\Omega\left(n\Delta^{-2}\right)$ for this instance (even when $\Delta$ is known)
		is in fact a directed sum-type result:
		roughly speaking, in order to solve \bestarm, we essentially need to solve $n-1$ independent copies of \sign\
		with gap $\Delta$.
		However, our upper bound in Theorem~\ref{theo:BESTARMPACD} indicates
		that the role of $\Gap{2}$ is different from the others $\Gap{i}$s. 
		Hence, if we want to go beyond Mannor-Tsitsiklis,
		\footnote{
			In other words, we want the lower bound to reflect the hardness caused by not knowing $\Delta_i$s.
		} 
		\bestarm\ can not be thought as $n$ independent copies of \sign. 
		In fact, from the analysis of the algorithm, we can see that the first term and the rest come
		from very different procedures: the first term is used for ``estimating'' the gap distribution 
		and the rest for ``verifying'' and ``eliminating'' suboptimal arms. 

		Our algorithm can achieve an even better upper bound 
		$
		O\left( \Gap{2}^{-2}\ln\ln \Gap{2}^{-1}+\sum\nolimits_{i=2}^{n} \Gap{i}^{-2} \ln\delta^{-1}\right)
		$
		for a special but important class of instances, which we call
		{\em clustered instances}.
		\footnote{
			We say the instances is clustered if the cardinality
			of the set $\{ \lfloor \ln \Gap{i}^{-1} \rfloor \}_{i=2}^n$ is bounded by a constant.
			See Theorem~\ref{thm:clustered} for more details.
		}
		See Section~\ref{subsec:clustered}.
		Note that this bound is almost {\em instance optimal}, since it matches the Mannor-Tsitsiklis instance-wise lower bound plus the 2-arm lower bound $\Gap{2}^{-2}\ln\ln \Gap{2}^{-1}$. 
		In fact, the aforementioned instances ($n-1$ arms with mean 0.5
		and 1 arm with mean $0.5+\Delta$) are clustered instances.
		Even for such basic instances, a tight bound is not known so far! 
		After a careful examination, we find that all previous algorithms are suboptimal on the very basic examples,
		while our algorithm can achieve the optimal bound $O(\Delta^{-2}\ln\ln \Delta^{-1}+n\Delta^{-2}\ln\delta^{-1})$.
		
		By slightly modifying our algorithm, we can easily obtain an $(\epsilon,\delta)$-PAC algorithm for \bestarm,
		which improves several prior works. See the supplementary material for the detailed information.

	\topic{Technical Novelty of Our Algorithm}
		Now, we provide a high level idea of our algorithm.
		Our algorithm is inspired by the elegant \EXPGAPELIMINATION\ algorithm in~\cite{karnin2013almost}. 
		In order to highlight the technical novelty of our algorithm,
		we provide here a very brief introduction to the \EXPGAPELIMINATION\ algorithm which runs in round.
		In the $r$th round of \EXPGAPELIMINATION, we first try to identify an $\epsilon_r$-optimal arm $A_r$
		(where $\epsilon_r=O(2^{-r})$), 
		using the classical PAC algorithm in \cite{even2006action}.
		Then, using the empirical mean of $A_r$ as a threshold, the algorithm tries to eliminate those arms with smaller means.
		In fact, comparing with the previous elimination-based algorithms, such as \cite{even2002pac,even2006action,bubeck2012multiple},
		\EXPGAPELIMINATION\ seems to be the most aggressive one, which is the main reason that \EXPGAPELIMINATION\ improves on the previous results.
		However, we show that \EXPGAPELIMINATION\ may be over-aggressive, and we may benefit from delaying the 
		elimination for some rounds, if we cannot eliminate a substantial number of arms in this round.
		To exploit this fact, we develop a procedure called \FRACTIONTEST, which, roughly speaking, 
		can inform us about the gap distribution and decide whether or not we should do elimination in this round.
		
		\vspace{-0.3cm}
		\subsubsection{Lower Bounds}
		
		\topic{Lower bound of \sign}
		First we briefly discuss our lower bound for the \sign\ problem,
		which plays a crucial role
		in the lower bound reduction for \bestarm.

		We first emphasize that Farrell's lower bound \eqref{eq:2armlowerbound}
		is not an instance-wise lower bound.\footnote{
			To the contrary, the bound $\sum\nolimits_i \Gap{i}^{-2}\ln\delta^{-1}$ 
			~\cite{mannor2004sample}
			is an instance-wise lower bound.
		}
		In particular, the $\limsup$ in \eqref{eq:2armlowerbound} merely 
		asserts that the existence of infinite number of instances 
		that require $\Delta^{-2} \ln\ln \Delta^{-2}$ samples
		(as $\Delta\rightarrow 0$), which is not enough 
		for our purpose (our reduction
		for \bestarm\ requires a stronger quantitative lower bound).	
		Moreover, we note that it is impossible to obtain an $\Omega(\Delta^{-2} \ln\ln \Delta^{-2})$
		lower bound for every instance, since we can design an algorithm 
		that uses $o(\Delta^{-2} \ln\ln \Delta^{-2})$ samples for infinite number of 
		instances (see the supplementary material Section~\ref{app:sign}). 
		
		Comparing to Farrell's lower bound, ours is a quantitative
		one.	
		Let the target lower bound be
		$
		\lowb(\Delta)=c\Delta^{-2}\ln\ln \Delta^{-1},
		$
		where $c$ is a small universal constant.
		For simplicity, assume that all the reward distributions are Gaussian with $\sigma=1$.
		Let $T_{\alg}(\Delta) = \max(T_{\alg}(A_{\xi+\Delta}),T_{\alg}(A_{\xi-\Delta}))$, 
		in which $A_{\xi+\Delta}$ and $A_{\xi-\Delta}$ denote the arms with means $\xi+\Delta$ and $\xi-\Delta$, 
		respectively.
		For an algorithm $\alg$, if $T_{\alg}(\Delta)\geq \lowb(\Delta)$, we say $\Delta$ is a ``{\em slow point}''
		(or $\alg$ is slow at $\Delta$),
		otherwise, it is a ``{\em fast point}''.
		\footnote{
			    We sometimes mention the sample complexity of an algorithm and its running time interchangeably,
			    since for all of our algorithms
			    the running time is at most a constant times the number of samples.
			    Hence, sometimes when we informally speak that an algorithm must be ``slow'', which also means it requires
			    many samples. 
			    }	
		Roughly speaking, our lower bound asserts that for any \CORRECT\ algorithm for \sign,
		there must be slow points in almost all intervals 
		$[e^{-i}, e^{-i+1})$, $i\in \mathbb{Z}^+$.
		The precise statement and the proof can be found in supplementary
		material (Theorem~\ref{lm:FRACTION-BOUND-DISC} in Section~\ref{sec:lbsign}).

		Our proof is very different from, and much simpler than the complicated proof in \cite{farrell1964asymptotic}.
		Furthermore, we note that \cite{farrell1964asymptotic} assumes the reward distributions are from the exponential family,
		while our proof only utilizes the KL divergence between different reward distributions, which is more general and applies to non-exponential family as well.
		
		\topic{Lower bound for \bestarm}
		Now, we discuss our new lower bound for \bestarm.
		Note that our current knowledge 
		(in particular the lower bounds of 
		Farrell and Mannor-Tsitsiklis) does not rule out an
		$O\left(\sum\nolimits_{i=2}^{n}\Gap{i}^{-2} \ln\delta^{-1} + \Gap{2}^{-2}\ln\ln\Gap{2}^{-1}\right)$ algorithm for \bestarm.
		If such a result exists, it is clearly a very satisfying answer.
		However, we show it is impossible by 
		by presenting an $\Omega\left(\sum\nolimits_{i=2}^{n}\Gap{i}^{-2} \ln\ln n\right)$ lower bound.

		This is the first lower bound that surpasses  
		$\Omega(\sum\nolimits_{i=2}^{n} \Gap{i}^{-2} \ln\delta^{-1})$ 
		\cite{mannor2004sample, kaufmann2014complexity}
		for general \bestarm. 
		The proof of the theorem is also interesting in its own right.
		We provide a nontrivial reduction from the \sign\ problem to the \bestarm\ problem,
		and utilize our previous lower bound for \sign\ to obtain the desired lower bound for \bestarm. 
		More concretely, we construct a class of instances for \bestarm, 
		and show that if there is an algorithm $\alg$ that can solve those instances faster than the target lower bound,
		we can construct an algorithm $\bB$ (calling
		$\alg$ as a subroutine) to solve a nontrivial proportion of a class of \sign\ instances faster than $\Delta^{-2}\ln\ln \Delta^{-1}$ time,
		which leads to a contradiction to our lower bound on \sign.
		Note that the old lower bound in \cite{farrell1964asymptotic} cannot be used here since it does not preclude the existence of 
		such an algorithm $\bB$ for \sign.
				
		\topic{On Instance Optimality}
		Instance optimality (\cite{fagin2003optimal,afshani2009instance}) is arguably the strongest possible notion of optimality. Loosely speaking,
		an algorithm $\alg$ is instance optimal if the running time 
		of $\alg$ on instance $I$ is at most $O(L(I))$, where $L(I)$
		is the lower bound required to solve the instance for any 
		algorithm.	We propose an intriguing conjecture
		concerning the instance optimality of \bestarm.		
		The conjecture concerns the sample complexity of every \bestarm\ instance,
		and provides a concrete formula for it. Interestingly, the formula
		involves an entropy-like term, which we call {\em gap entropy}.
		The new $\ln\ln n$ factor appearing in both 
		our new upper and lower bounds 
		is in fact a tight bound of the gap entropy. The proofs of our new results
		also provide strong evidence for the conjecture.
		The details can be found in the supplementary material.


\vspace{-0.2cm}
\subsection{Other Related Work}
\vspace{-0.2cm}

\topic{\bestarm}
 Kaufmann et al.~\cite{kaufmann2014complexity} provided an $\Omega(\sum_{i=2}^{n} \Gap{i}^{-2} \ln \delta^{-1})$ lower bound for \bestarm, with a better constant factor than in \cite{mannor2004sample}.
 Garivier and Kaufmann~\cite{garivier2016optimal} obtained a complete resolution of the 
 asymptotic sample complexity
 of \bestarm\ in the regime where $\delta \to 0$ (treating $\Delta_i$s as fixed).
 However, our work focus on the regime where all $\Delta_i$s, and $\delta$
 are variables that can approach to $0$.
 In fact, when we allow $\Delta_i$ to approach to $0$ and maintain $\delta$ fixed, their lower bound is not tight. \footnote{
 	Their lower bound is of the form $T^{*}(I) \cdot \ln \delta^{-1}$
 	for instance $I$ (see \cite{garivier2016optimal} for the definition of $T^*$). 
 	In fact, we can see $T^{*}(I)$ is upper bounded by $O(\sum_{i=2}^{n} \Gap{i}^{-2})$ by existing upper bounds. 
 	When $\delta$ is some constant, by
 	our new lower bound Theorem~\ref{thm:hard-case-exists},  $\Omega(\sum_{i=2}^{n}\Gap{i}^{-2} \ln \delta^{-1})$ is not tight. } 

\topic{\sign\ and A/B testing}
The \sign\ problem is closely related to the A/B testing problem in the literature, in which 
we have two arms with unknown means and the goal is to decide which one is larger.
It is easy to see that a lower bound for \sign\ is also a lower bound for the the A/B testing problem. 
Kaufmann et al.~\cite{kaufmann2014complexity}
studied the optimal sample complexity for the A/B testing problem. However, their focus is on
the limiting behavior of the sample complexity when $\delta\rightarrow 0$,
while we are interested in the case where the gap $\Delta$ approaches to zero but $\delta$ is a constant
(in their case, the $\ln\ln \Delta^{-1}$ factor is absorbed by the $O(\ln \delta^{-1})$ factor).


\topic{\bestkarm}
One natural generalization of \bestarm\ is the \bestkarm\ problem, which asks for 
the top-$k$ arms instead of just the top-$1$. 
The \bestkarm\ problem has also been studied extensively for the last few years \cite{kalyanakrishnan2010efficient,  gabillon2012best, gabillon2011multi, kalyanakrishnan2012pac, bubeck2012multiple, kaufmann2013information, zhou2014optimal, kaufmann2014complexity}.
Most lower and upper bounds for \bestkarm\ are variants of those for \bestarm,
and the bounds also depend on the gap parameters. But in this case, the gaps are typically defined to 
be the distance to $\mu_{[k]}$ or $\mu_{[k+1]}$.
Chen et al.~\cite{chen2014combinatorial}
and Chen et al.~\cite{chen2016matroid} study the combinatorial pure exploration problem, which
generalizes the cardinality constraint in \bestkarm\ to more general combinatorial constraints.


\topic{PAC learning}
The worst case sample complexity of \bestarm\ in the PAC setting is also well studied. 
There is a matching lower and upper bounds
obtained by $\Omega(n \ln\delta^{-1}/\epsilon^2)$ in~\cite{even2002pac, even2006action, mannor2004sample}. 
The worst case sample complexity for \bestkarm\ in the PAC setting has also
been well studied by many authors during the last few years~\cite{kalyanakrishnan2010efficient,kalyanakrishnan2012pac,zhou2014optimal,cao2015top}.


\newcommand{\lnmin}[2]{\ln[\min(#1,#2)]}
\newcommand{\armset}[1]{U^{#1}}
\newcommand{\bigarmset}[1]{U^{\ge#1}}
\newcommand{\smaarmset}[1]{U^{\le#1}}
\newcommand{\maxs}{\mathrm{max}_s}

\vspace{-0.4cm}
\section{An Improved Upper Bound for \bestarm}
\label{sec:bestarmupperboundsketch}
\vspace{-0.2cm}

In this section, we present our new algorithm, which achieves 
the improved upper bound in Theorem~\ref{theo:BESTARMPACD}.
Due to space constraint, all proofs and some details are deferred to the supplementary material.
Our final algorithm builds on several useful components.

\topic{1. Uniform Sampling}
The first building block is the simple uniform sampling algorithm. Given two parameters $\epsilon,\delta$ and a set of arms $S$, it takes from each arm $a \in S$ 
$2\epsilon^{-2}\ln(2\cdot \delta^{-1})$ samples. 
Let $\hamean{a}$ be the empirical mean of arm $a$.
We denote the algorithm as $\UNIFORMSAMPLING(S,\varepsilon,\delta)$.
We have the following lemma 
which is a simple consequence of Hoeffding's inequality.

\begin{lemma}
	For each arm $a \in S$, we have that
	$
	\Pr\left[|\amean{a}-\hamean{a}|\ge \epsilon\right] \le \delta.
	$
\end{lemma}

\topic{2. Median Elimination}
We need the median elimination algorithm developed in \cite{even2006action}, which is a classic $(\epsilon, \delta)$-PAC algorithm for \bestarm.
The algorithm takes parameters $\epsilon,\delta > 0$ and a set $S$ of arms,
and returns an $\epsilon$-optimal arm with probability $1-\delta$.
The algorithm runs in rounds. In each round, it samples every remaining arm a uniform number of times,
and then discard half of the arms with lowest empirical mean (thus the name {\em median elimination}). 
It outputs the final arm that survives.
We denote the procedure by \MEDIANELIMINATION$(S,\epsilon,\delta)$.
We use this algorithm in a black-box manner and
its performance is summarized in the following lemma.

\begin{lemma} 
	Let $\amean{1}$ be the maximum mean value. 
	\MEDIANELIMINATION$(S,\epsilon,\delta)$ returns an $\varepsilon$-optimal arm (i.e., with mean at least $\amean{1} - \epsilon$) 
	with probability at least $1-\delta$, using a budget of at most $O(|S|\log(1/\delta)/\epsilon^{2})$ samples. 
\end{lemma}

\topic{3. Fraction test}	
\FRACTIONTEST\ is a simple estimation procedure, which can be used to gain some information about the distribution of the arms.
It plays a key role in the final algorithm.
The algorithm takes six parameters $(S,c_l,c_r,\delta,t,\epsilon)$,
where $S$ is the set of arms, $\delta$ is the confidence level, $c_l<c_r$ are real numbers called {\em range parameters}, $t\in (0,1)$
is the threshold, and $\epsilon$ is a small positive constant.
Typically, $c_l$ and $c_r$ are very close.
The goal of the algorithm, roughly speaking, is to distinguish whether there are still many arms in $S$ which  have small means 
(w.r.t. $c_r$)
or the majority of arms already have large means (w.r.t. $c_l$).

The algorithm runs in $\ln(2\cdot\delta^{-1}) (\epsilon/3)^{-2}/2$ iterations.
In each iteration, it samples an arm $a_i$ uniformly from $S$, 
and takes $O(\ln \epsilon^{-1} (c_r-c_l)^{-2})$ independent samples from $a_i$.
Then, we maintain a counter $\cnt$ which is initially 0,
and counts 
the fraction of iterations in which the empirical mean of $a_i$ is smaller than $(c_l+c_r)/2$.
If the fraction is larger than $t$, the algorithm returns $\True$. Otherwise, it returns $\False$.

For ease of notation, 
we define $S^{\ge c}:= \{\amean{a} \ge c \betw a \in S \}$, i.e.,
all arms in $S$ with means at least $c$. 
Similarly, we can define $S^{>c},S^{\le c}$ and $S^{< c}$. 

\begin{lemma} 
	Suppose $\epsilon < 0.1$ and $t \in (\epsilon,1-\epsilon)$. With probability $1-\delta$,
	the following hold:
	\vspace{-0.2cm}
	\begin{itemize}[leftmargin=*]
		\item If \FRACTIONTEST\ outputs $\True$, then $|S^{>c_r}| < (1-t+\epsilon) |S|$
		(or equivalently $|S^{\le c_r}| > (t-\epsilon)|S|$).
		\item If \FRACTIONTEST\ outputs $\False$, then $|S^{<c_l}| < (t+\epsilon) |S|$
		(or equivalently $ |S^{\ge c_l}| > (1-t-\epsilon)|S|$).
	\end{itemize}
	\vspace{-0.2cm}
	Moreover, the number of samples taken by the algorithm is $O(\ln\delta^{-1}\epsilon^{-2}\Delta^{-2} \ln \epsilon^{-1})$, in which $\Delta = c_r-c_l$.
	If $\epsilon$ is a fixed constant, then the number of samples is simply $O(\ln \delta^{-1} \Delta^{-2})$.
\end{lemma}

\topic{4. Eliminating arms}
The last ingredient is an elimination procedure, which can be used to eliminate most arms below a given threshold.
The procedure takes four parameters $(S,c_l,c_r,\delta)$ as input, 
where $S$ is a set of arms, $c_l<c_r$ are the range parameters, and $\delta$
is the confidence level.
It outputs a subset of $S$
and guarantees that upon termination, most of the remaining arms have means at least $c_l$ with probability $1-\delta$.

Now, we describe the procedure \ELIMINATION\, which runs in iterations.
It maintains the current set $S_r$ of arms, which is initially $S$.
In each iteration, it first applies  
\FRACTIONTEST$(S_r,c_l,c_m,\delta_r,0.075,0.025)$ on $S_r$,
where $c_m=(c_r+c_r)/2$.
If \FRACTIONTEST\ returns $\True$, which means
that there are at least 5\% fraction of arms with 
means smaller than $c_m$ in $S_r$, we sample all arms in $S_r$ uniformly by calling \UNIFORMSAMPLING$(S_r,(c_r-c_m)/2,\delta_r)$
where $\delta_r=\delta/(10\cdot 2^r)$,
and retain those with empirical means at least $(c_m+c_r)/2$.
If \FRACTIONTEST\ returns $\False$ (meaning that 90\% arms have means at least $c_l$) 
the algorithm terminates and returns the remaining arms.
The guarantee of \ELIMINATION\ is summarized in the following lemma.

\begin{lemma} 
	Suppose $\delta < 0.1$.
	Let $S' = \ELIMINATION(S,c_l,c_r,\delta)$. 
	Let $A_1$ be the best arm among $S$, with mean $\amean{A_1} \ge c_r$.
	Then with probability at least $1-\delta$, the following statements hold
	\vspace{-0.2cm}
	\begin{enumerate}[leftmargin=*]
		\item $A_1 \in S'$ \,\, (the best arm survives);
		\item $|S'^{\le c_l}| < 0.1 |S'|$ \,\, (only a small fraction of arms have means less than $c_l$);
		\item The number of samples is $O(|S| \ln \delta^{-1} \Delta^{-2})$, in which $\Delta = c_r-c_l$.
	\end{enumerate}
	\vspace{-0.2cm}
\end{lemma}

\topic{Our Final Algorithm \DELIMINATION}
Now, everything is ready
to describe our algorithm \DELIMINATION\
for \bestarm.
We provide a high level description here.
All detailed parameters can be found in Algorithm~\ref{algo:BESTARM-m}.
The algorithm runs in rounds.
It maintains the current set $S_r$ of arms.
Initially, $S_1$ is the set of all arms $S$.
In round $r$, the algorithm tries to eliminate a set of suboptimal arms, while makes sure 
the best arm is not eliminated.
First, it applies the \MEDIANELIMINATION\ procedure to find an 
$\epsilon_{r}/4$-optimal arm, where $\epsilon_{r}=2^{-r}$.
Suppose it is $a_r$. Then, we take a number of samples from $a_r$ to estimate its mean
(denote the empirical mean by $\hamean{a_r}$).
Unlike previous algorithms in~\cite{even2002pac,karnin2013almost},
which eliminates either a fixed fraction of arms or those arms with mean much less than $a_r$,
we use a \FRACTIONTEST\ to see whether there are many arms with mean much less than $a_r$.
If it returns $\True$, we apply the \ELIMINATION\ procedure to eliminate those arms
(for the purpose of analysis, we need to use \MEDIANELIMINATION\ again, but with a tighter confidence level,
to find an $\epsilon_{r}/4$-optimal arm $b_r$).
If it returns $\False$, the algorithm decides that it is not judicious to do elimination in this round
(since we need to spend a lot of samples, but only discard very few arms, which is wasteful),
and simply sets $S_{r+1}$ to be $S_r$, and then proceeds to the next round.

\begin{algorithm}[H]
	\LinesNumbered
	\setcounter{AlgoLine}{0}
	\caption{\DELIMINATION($S,\delta$)}\label{algo:BESTARM}
	\label{algo:BESTARM-m}
	$h \leftarrow 1$
	
	$S_1 \leftarrow S$
	
	\For{$r = 1$ to $+\infty$}{
		\uIf{$|S_r| = 1$}{
			\(\bf Return\) the only arm in $S_r$ 
		}
		$\epsilon_{r} \leftarrow 2^{-r}$ 
		
		$\delta_{r} \leftarrow \delta/50r^2$
		
		$a_r \leftarrow \MEDIANELIMINATION(S_r,\epsilon_r/4,0.01)$. \label{line:ME1} 
		
		$\hamean{a_r} \leftarrow \textrm{\UNIFORMSAMPLING}(\{a_r\},\epsilon_r/4,\delta_r)$ \label{line:US1}
		
		\uIf{ {\em \FRACTIONTEST}$(S_r,\hamean{a_r}-1.5\epsilon_r,\hamean{a_r}-1.25\epsilon_r,\delta_r,0.4,0.1)$} { \label{line:FT}
			$\delta_h \leftarrow \delta/50h^2$
			
			$b_r \leftarrow \MEDIANELIMINATION(S_r,\epsilon_r/4,\delta_h)$ \label{line:ME2}
			
			$\hamean{b_r} \leftarrow \UNIFORMSAMPLING(\{b_r\},\epsilon_r/4,\delta_h)$ \label{line:US2}
			
			$S_{r+1} \leftarrow \ELIMINATION(S_r,\hamean{b_r}-0.5\epsilon_r,\hamean{b_r}-0.25\epsilon_r,\delta_h)$ \label{line:EL} 
			
			$h \leftarrow h + 1$ \label{line:INC-H}
		}\Else{
		
		$S_{r+1} \leftarrow S_{r}$
	}
}
\end{algorithm}

 \begin{theo}\label{theo:BESTARMPACD}
    	For any $\delta <0.1$, there is a $\delta$-correct algorithm for \bestarm\
    	which needs at most  
    	$
    	O\Big(
    	\Gap{2}^{-2}\ln\ln \Gap{2}^{-1}+
    	\sum_{i=2}^{n} \Gap{i}^{-2} \ln\delta^{-1}+\sum_{i=2}^{n} \Gap{i}^{-2}\ln\ln \min(n,\Gap{i}^{-1}) 
    	\Big)
    	$ 
    	samples in expectation. 
 \end{theo}
\vspace{-0.2cm}

It is not difficult to verify that \DELIMINATION\ returns the best arm with probability
at least $1-\delta$.
However, the analysis of the running time of our algorithm is much more challenging.
The rough idea is 
to consider the number of arms in each interval
$\armset{s} = \{ a \mid 2^{-s} \le \Gap{a} < 2^{-s+1}\}$
and see how it changes in very round. 
When we execute \ELIMINATION\ in round $r$, we can 
eliminate a substantial fraction or arms in $\armset{1},\ldots, \armset{r}$.
However, there are still some remaining and we need to keep track of them over
the ensuing rounds. 
For that purpose, we need to carefully choose a potential function to amortize the costs over different iterations.
 		
Our algorithm can achieve an even better upper bound 
$
O\left( \Gap{2}^{-2}\ln\ln \Gap{2}^{-1}+\sum\nolimits_{i=2}^{n} \Gap{i}^{-2} \ln\delta^{-1}\right)
$
for {\em clustered instances}.
Furthermore, by slightly modifying the algorithm, we can obtain an improved $(\epsilon,\delta)$-PAC algorithm for \bestarm. See the supplementary material.


\vspace{-0.4cm}
\section{A New Lower Bound for \bestarm}
\vspace{-0.2cm}

In this section, we provide a sketch proof of the following new lower bound for \bestarm.
From now on, $\delta$ is a fixed constant such that $0 < \delta < 0.005$. 
Throughout this section, we assume the distributions of all the arms are Gaussian with variance $1$.

\begin{theo}\label{thm:hard-case-exists}
	There exist constants $c,c_1 > 0$ and $N \in \N$ such that, for any $\delta < 0.005$ and any \CORRECT\ algorithm $\alg$, 
	and any $n \ge N$, there exists an $n$ arms instance $I$ such that 
	$T_{\alg}[I] \ge c \cdot \sum_{i=2}^{n} \Gap{i}^{-2} \ln\ln n$. 
	Furthermore, $\Gap{2}^{-2} \ln\ln \Gap{2}^{-1} < \frac{c_1}{\ln n}\cdot \sum_{i=2}^{n} \Gap{i}^{-2} \ln\ln n$.  
\end{theo}
\vspace{-0.2cm}

The second statement of the theorem says that 
$\sum\nolimits_{i=2}^{n} \Gap{i}^{-2} \ln\ln n$ 
dominates the 2-arm lower bound $\Gap{2}^{-2} \ln\ln \Gap{2}^{-1}$ 
(so that the theorem is not vacant).
In order to prove Theorem~\ref{thm:hard-case-exists},
we need a new lower bound for \sign\ to serve as the basis for our reduction to \bestarm.
Let $\algp$ denote an algorithm for \sign, 
$A_\mu$ be an arm with mean $\mu$ (i.e., with distribution $\Normal(\mu,1)$), and we define
$T_{\algp}(\Delta) = \max(T_{\algp}(A_{\xi+\Delta}), T_{\algp}(A_{\xi-\Delta}))$. Then we have the following new lower bound for \sign.
The proof is deferred to the supplementary material (Section~\ref{sec:lbsign}).

\begin{lemma}~\label{lm:prev-result-sk}
	For any $\delta'$-correct algorithm $\algp$ for \sign\ with $\delta' \le 0.01$, 
	there exist constants $N_0 \in \N$ and $c_1 > 0$ such that for all $N \ge N_0$,
	$
	|\{ T_{\algp}(\Delta) < c_1 \cdot \Delta^{-2} \ln N \betw \Delta = 2^{-i}, i \in [0,N] \}|  \le 0.1(N-1).
	$
\end{lemma}

\vspace{-0.6cm}
\begin{proofof}{Theorem~\ref{thm:hard-case-exists}}	(sketch)
	Without loss of generality, we can assume $N_0$ in Lemma~\ref{lm:prev-result-sk} is an even integer, 
	and $N_0 > 10$ such that $2\cdot 4^{N_0} \ge \frac{4}{3} \cdot 4^{N_0} + N_0 + 2$. Let $N=2\cdot 4^{N_0}$.
	For every $n \ge N$, we pick the largest even integer $\nbag$ such that $2\cdot 4^{\nbag} \le n$. Consider the following \bestarm\ instance $I_{\mathrm{init}}$ with $n$ arms:
	(1) There is a single arm with mean $\xi$.
	(2) For each $k \in [0,\nbag]$, there are $4^{\nbag-k}$ arms with mean $\xi - 2^{-k}$.
	(3) There are $n - \sum_{k=0}^{\nbag} 4^k - 1$ arms with mean $\xi - 2$.

	Now we define a class of \bestarm\ instances $\{I_{S}\}$
	where each $S \subseteq \{0,1,\ldots, \nbag\}$.
	Each $I_{S}$ is formed as follows:
	for every $k \in S$, we add one more arm with mean $\xi - 2^{-k}$ to $I_{\mathrm{init}}$; 
	finally we remove $|S|$ arms with mean $\xi - 2$
	(by our choice of $\nbag$ there are enough such arms to remove). 
	Obviously, there are still $n$ arms in every instance $I_{S}$.
	
	For a \bestarm\ instance $I$, let $n(I)$ be the number of arms in $I$, and $\Gap{i}(I)$ be the gap $\Gap{i}$ according to $I$. 
	We denote $\armcomp(I) = \sum_{i=2}^{n(I)} \Gap{i}(I)^{-2}$. 
	Now we claim that for any $\delta$-correct algorithm $\alg$ for \bestarm,
	there must exist an instance $I_{S}$ such that 
	$
	T_\palg(I_S) > c \cdot \armcomp(I_S) \cdot \ln \nbag =\Omega(\armcomp(I_S) \ln\ln n),
	$
	for some universal constant $c>0$, where $\palg$ first randomly permutes the arms and then simulates $\alg$.
	
	Suppose for contradiction that there exists a $\delta$-correct $\alg$ such that 
	$T_{\palg}(I_S) \le c \cdot \armcomp(I_S) \cdot \ln \nbag$ for all $S$.
	Let $U= \{ I_{S} \betw |S| = \nbag/2 \}$, $V = \{ I_S \betw |S| = \nbag/2 + 1 \}$ be two sets of \bestarm\ instances. 
	Notice that $|U| = |V| = \binom{\nbag+1}{\nbag/2}$ (since $\nbag$ is even). 
	
	Fix $S \in U$. 
	Consider the problem \sign,
	in which the given 
	instance is a single arm $A$ with unknown mean $\mu$, and we would like to decide whether 
	$\mu>\xi$ or $\mu<\xi$.
	Now, we construct an algorithm $\alg_S$ for \sign.
	First consider the following two algorithms
	for \sign, which call $\palg$ as a subprocedure.
	(1) $\algone{S}$: 
		We first create a \bestarm\ instance 
		instance $\newI$ by
		replacing one arm with mean $\xi-2$ in $I_S$ with $A$. Then run $\palg$ on $\newI$. 
		We output $\mu > \xi$ if $\palg$ selects $A$ as the best arm. Otherwise, we output $\mu < \xi$.
	(2) $\algtwo{S}$: We first construct an artificial arm $\newA$ with mean $2\xi - \mu$ from $A$.
		\footnote{That is, whenever the algorithm pulls $\newA$, we pull $A$ to get a reward $r$, and return $2\xi-r$ as the reward for $\newA$. 
			Note although we do not know $\mu$, $\newA$ is clearly an arm with mean $2\xi - \mu$.
		} 
		Create a \bestarm\ instance $\newI$
		by replacing one arm with mean $\xi-2$ in $I_S$ with $\newA$. Then run $\palg$ on $\newI$.
		We output $\mu < \xi$ if $\palg$ selects $\newA$ as the best arm. Otherwise,
		we output $\mu > \xi$.
	$\alg_S$ simulates $\algone{S}$ and $\algtwo{S}$ simultaneously: 
	Each time it takes a sample from the input arm, and feeds it to both $\algone{S}$ and $\algtwo{S}$. If $\algone{S}$ ($\algtwo{S}$ resp.) terminates first, it returns the output of $\algone{S}$ ($\algtwo{S}$ resp.). It is not hard to see that $\alg_S$ is $2\delta$-correct for \sign.
	
		Now we analyze the expected total number of
	samples taken by $\alg_S$ on arm $A$ with mean $\mu$ and gap $\Delta=|\xi-\mu|=2^{-k}$. 
	Suppose $k \notin S$. A key observation is the following:
	if $\mu < \xi$, then the instance constructed in $\algone{S}$ is exactly $I_{S \cup \{k\}}$; otherwise $\mu > \xi$, since $2\xi - (\xi+\Delta) = \xi - \Delta = \xi - 2^{-k}$, the instance constructed in $\algtwo{S}$ is exactly $I_{S \cup \{k\}}$.
	Hence, $T_{\alg_S}(A) \le \min(T_{\algone{S}}(A),T_{\algtwo{S}}(A))\le a_{S \cup \{k\}}^k \cdot 4^k$,
	where $a_S^k$ is so defined that $a_S^k \cdot 4^k$ is the expected number of samples taken from an arm with gap $2^{-k}$ by $\palg(I_S)$.
	
	Moreover, since $T_{\palg}(I_S) \le c \cdot \armcomp(I_S) \cdot \ln \nbag$ for all $S$, 
	we can show that for any $S$, there are at most $0.1$ fraction of elements in $\{a_S^k\}_{k=0}^{\nbag}$ satisfying $a_S^k \ge c_1 \cdot \ln \nbag$ (letting $c=c_1/30$ will suffice).
	Intuitively, this implies $T_{\alg_S}(A)\leq c_1 4^k\ln \nbag$ from for $\Delta=2^{-k}, k\in [0,m]$.	Indeed, by a careful counting argument, we can show that there exists an $S\in U$ such that $\{ T_{\alg_S}(\Delta) < c_1 \cdot \Delta^{-2} \ln \nbag \betw \Delta = 2^{-i}, i \in [0,\nbag] \} \geq 0.4(\nbag+1)$, 
	which is a contradiction to Lemma~\ref{lm:prev-result-sk}.
\end{proofof}

		\vspace{-0.5cm}
		\section{Concluding Remarks}
		\vspace{-0.2cm}
		
		The most interesting open problem 
		from this paper is 
		to obtain an almost instance optimal algorithm
		for \bestarm, in particular
		to prove (or disprove) Conjecture~\ref{conj:optimal}.
		Note that for the clustered instances, and the instances
		where the gap entropy is $\Omega(\ln\ln n)$, 
		we already have such an algorithm.
		Our techniques may be helpful for obtaining better bounds for the \bestkarm\ problem, or even 
		the combinatorial pure exploration problem.
		In an ongoing work, we already have some partial results on applying 
		some of the ideas in this paper to obtain improved upper and lower bounds for \bestkarm.
		
	\eat{	
		We have made no attempt to optimize the constant 
		hidden in the big-O in our upper bound. 
		However, we believe that it is possible to obtain
		a competitive algorithm in practice by fine-tuning 
		some parameters in our algorithm. 
		Some of the our ideas (in particular, the fraction test and the elimination procedure based
		on the distribution of the gaps)
		may be easily combined with other practical heuristics. 
		Simplifying our algorithm and making it more practical 
		are left as interesting future directions. 
	}	
		\eat{
			\section{Acknowledgment}
			We would like to thank Anupum Gupta for several interesting discussions
			in the beginning of this work, and helpful comments on an earlier version of the paper.
			In fact, the question whether $\sum\nolimits_i \Gap{i}^{-2}\ln\ln \Gap{i}^{-1}$ 
			is the instance-wise lower bound for \bestarm\ was raised during a discussion with him.
		}
		
		\newpage
		\begin{small}
		\bibliographystyle{abbrv}
		\bibliography{team} 
		\end{small}

		\appendix


\clearpage



\section*{Supplementary Material for ``On the Optimal Sample Complexity for Best Arm
Identification''}

The supplementary material is organized as follows:
\begin{enumerate}	[leftmargin=*]
\item (Section~\ref{sec:prel}) We provide some preliminary knowledge for our later developments.
\item (Section~\ref{sec:bestarmupperbound}) We provide all details of our new algorithm \DELIMINATION\ and the proof of Theorem~\ref{theo:BESTARMPACD}.
In Section~\ref{subsec:clustered}, we define the clustered instances and show 
our algorithm is almost instance optimal for such instances. 
In Section~\ref{subsec:improvePAC},
we slightly modify the algorithm and obtain an improved $(\epsilon,\delta)$-PAC algorithm
for \bestarm.
\item (Section~\ref{sec:bestarmlowerbound}) We provide the detailed proof of Theorem~\ref{thm:hard-case-exists}, our new lower bound for \bestarm.
\item (Section~\ref{sec:lbsign}) We provide our new lower bound for \sign, which is the basis
of our lower bound reduction in Section~\ref{sec:bestarmlowerbound}.
\item (Section~\ref{sec:instanceopt}) We propose to investigate \bestarm\ 
from the perspective of instance optimality, and propose a conjecture concerning 
the fundamental sample complexity for every instance of \bestarm.
We also discuss how our new results are related with the conjecture and why we think the conjecture 
is likely to be true.
\item (Section~\ref{app:missingpf-upper}) This section contains some missing technical proofs from Section~\ref{sec:bestarmupperbound}.	
\item (Section~\ref{app:sign}) We present a class of 
\CORRECT\ algorithms for \sign\ which 
needs  $o(\Delta^{-2}\ln\ln \Delta^{-1})$ samples for 
infinite instances. It is useful in discussing the instance optimality in Section~\ref{sec:instanceopt}.
\item (Section~\ref{sec:simult}) We provide a transformation that turns
an algorithm with only conditional expected sample complexity upper bound
to one with asymptotically the same unconditional expected sample complexity upper bound, under mild conditions.
\end{enumerate}

\section{Preliminaries}
\label{sec:prel}

    
    \begin{defi}\label{defi:sub-gaussian-tail}
    Let $R > 0$, we say a distribution $\distr$ on $\R$ has $R$-sub-Gaussian tail 
    (or $\distr$ is $R$-sub-Gaussian)
    if for the random variable $X$ drawn from $\distr$ and any $t \in \R$, 
    we have that $\Ex[\exp(t X - t \Ex[X])] \le \exp(R^2t^2/2)$.
    \end{defi}
    
    It is well known that the family of $R$-sub-Gaussian distributions contains all distributions with support on $[0,R]$ as well as many unbounded distributions such as Gaussian distributions with variance $R^2$.    
    Then we recall a standard concentration inequality for $R$-sub-Gaussian random variables.
    
    \begin{lemma}(Hoeffding's inequality)\label{lm:hoeff}
    Let $X_1,\dotsc,X_n$ be $n$ i.i.d. random variables drawn from an $R$-sub Gaussian distribution $\distr$.
    Let $\mu = \Ex_{x \sim \distr}[x]$. 
    Then for any $\epsilon > 0$, we have that
    \[
    \Pr\left[ \left|\frac{1}{n}\sum_{i=1}^{n} X_i - \mu\right| \ge \epsilon \right] \le 2 \exp\left(-\frac{n\epsilon^2}{2 R^2}\right).
    \]
    \end{lemma}
    
    For simplicity of exposition, we assume all reward distributions are $1$-sub-Gaussian in the paper.

	Suppose $\alg$ is an algorithm for \bestarm (or \sign).
	Let the given instance be $I$. 
	Let $\event$ be an event and 
	$\Pr_{\alg,I}[\event]$ be the probability that the event $\event$ happens when running $\alg$ on instance $I$. 
	When $\alg$ is clear from the context, we omit the subscript $\alg$ and simply write  $\Pr_{I}[\event]$. 
	Similarly, if $X$ is a random variable, we use $\Ex_{\alg,I}[X]$ to denote the expectation of $X$ 
	when running $\alg$ on instance $I$. Sometimes, $\alg$ takes an additional confidence parameter $\delta$, 
	and we write $\alg_{\delta}$ to denote the algorithm $\alg$ with the fixed confidence parameter $\delta$.

	Let $\tau_i$ be the random variable that denotes the number of pulls from arm $i$
	(when the algorithm and the problem instance are clear from the context)
	and $\Ex_I[\tau_i]$ be its expectation. 
	Let $\tau = \sum_{i=1}^{n} \tau_i$ be the total number of samples taken by $\alg$. 

	The Kullback-Leibler (KL) divergence of any two distributions $p$ and $q$
	is defined to be 
	$$
	\KL(p, q)= \int \log \left(\frac{\d p}{\d q}(x) \right) \d p(x)  \quad \text{if}\quad q\ll p
	$$
	where $q\ll p$ means that $\d p(x)=0$ whenever $\d q(x)=0$. 
	For any two real numbers $x,y \in (0,1)$,
	let $\ent(x,y) = x\log(x/y) + (1-x) \log((1-x)/(1-y))$ be
	the relative entropy function.
	
	Many lower bounds in the bandit literature rely on certain ``changes of distributions'' argument.
	The following version (Lemma 1 in \cite{kaufmann2014complexity})
	is crucial to us.
	
	\begin{lemma}(Change of distribution)
	\label{lm:CHANGEDIST} \cite{kaufmann2014complexity}
	We use an algorithm $\alg$ for a bandit problem with $n$ arms.	
	\footnote{
		We make no assumption on the behavior of $\alg$ in this lemma.
		For example, $\alg$ may even output incorrect answers with high probability.
	}
	Let $I$ (with arm distributions $\{\distr_i\}_{i\in [n]}$) and 
	$I'$ (with arm distributions $\{\distr'_i\}_{i\in [n]}$) be two instances. 
	Let $\event$ be an event, 
	\footnote{
		More rigorously, $\event$ should be in the $\sigma$-algebra $\calF_\tau$ where $\tau$ is a stopping time with 
		respect to the filtration $\{\mathcal{F}_t\}_{t\geq 0}$.
	}
	such that $0 < \Pr_{\alg, I}(\event) < 1$.
	Then, we have
	\[
	\sum_{i=1}^{n} \Ex_{I}[\tau_i] \KL(\distr_i,\distr'_i) \ge \ent(\Pr_{\alg, I}(\event),\Pr_{\alg, I'}(\event)).
	\]
	\end{lemma}
	
	
	\noindent
	$\Normal(\mu,\sigma^2)$ denotes the Gaussian distribution with mean $\mu$ and standard deviation $\sigma$. 
	We also need the following well known fact about the KL divergence between two Gaussian distributions.

	\begin{lemma}
	\label{lm:NORMALKL}
	\[
	\KL(\Normal(\mu_1,\sigma^2),\Normal(\mu_2,\sigma^2)) = \frac{(\mu_1-\mu_2)^2}{2\sigma^2}.
	\]
	\end{lemma}

	
		
\section{An Improved Upper Bound for \bestarm}
\label{sec:bestarmupperbound}

In this section we prove Theorem~\ref{theo:BESTARMPACD}
by presenting an algorithm for \bestarm. 
Our final algorithm builds on several useful components.

\subsection{Useful Building Blocks}
\label{subsec:buildingblocks}

\topic{1. Uniform Sampling}
The first building block is the simple uniform sampling algorithm.

\vspace{0.2cm}
\begin{algorithm}[H]
\LinesNumbered
\setcounter{AlgoLine}{0}
	\caption{\UNIFORMSAMPLING($S,\epsilon,\delta$)}
	\label{algo:UNIFORM-SAMPLE-PROCEDURE}
	\KwData{Arm set $S$, approximation level $\epsilon$, confidence level $\delta$.}
	\KwResult{For each arm $a$, output the empirical mean $\hamean{a}$.}
	\smallskip
	For each arm $a \in S$, sample it $2\epsilon^{-2}\ln(2\cdot \delta^{-1})$ times. Let $\hamean{a}$ be the empirical mean.
\end{algorithm}
\vspace{0.2cm}

\noindent
We have the following Lemma for Algorithm~\ref{algo:UNIFORM-SAMPLE-PROCEDURE},
which is a simple consequence of Lemma~\ref{lm:hoeff}.

\begin{lemma}\label{lm:UNIFORM-SAMPLE-PROCEDURE}
	For each arm $a \in S$, we have that
	$
	\Pr\left[|\amean{a}-\hamean{a}|\ge \epsilon\right] \le \delta.
	$
\end{lemma}

\topic{2. Median Elimination}
We need the \MEDIANELIMINATION\ algorithm in \cite{even2006action}, which is a classic $(\epsilon, \delta)$-PAC algorithm for \bestarm.
The algorithm takes parameters $\epsilon,\delta > 0$ and a set $S$ of $n$ arms,
and returns an $\epsilon$-optimal arm with probability $1-\delta$.
The algorithm runs in rounds. In each round, it samples every remaining arm a uniform number of times,
and then discard half of the arms with lowest empirical mean (thus the name {\em median elimination}). 
It outputs the final arm that survives.
We denote the procedure by \MEDIANELIMINATION$(S,\epsilon,\delta)$.
We use this algorithm in a black-box manner and
its performance is summarized in the following lemma.

\begin{lemma}
	\label{lm:MEDIANELIM}
	Let $\amean{1}$ be the maximum mean value. 
	\MEDIANELIMINATION$(S,\epsilon,\delta)$ returns an arm with mean at least $\amean{1} - \epsilon$ 
	with probability at least $1-\delta$, using a budget of at most $O(|S|\log(1/\delta)/\epsilon^{2})$ pulls. 
\end{lemma}

\topic{3. Fraction test}	
\FRACTIONTEST\ is an estimation procedure, which can be used to gain some information about the distribution of the arms.
The algorithm takes six parameters $(S,c_l,c_r,\delta,t,\epsilon)$,
where $S$ is the set of arms, $\delta$ is the confidence level, $c_l<c_r$ are real numbers called {\em range parameters}, and $t\in (0,1)$
is the threshold, and $\epsilon$ is a small positive constant.
Typically, $c_l$ and $c_r$ are very close.
The goal of the algorithm, roughly speaking, is to distinguish whether there are still many arms in $S$ which  have small means 
(w.r.t. $c_r$)
or the majority of arms already have large means.
The precise guarantee the algorithm can achieve can be found in Lemma~\ref{lm:ESTIMATE-PROCEDURE}.

The algorithm runs in $\ln(2\cdot\delta^{-1}) (\epsilon/3)^{-2}/2$ iterations.
In each iteration, it samples an arm $a_i$ uniformly from $S$, 
and takes $O(\ln \epsilon^{-1} (c_r-c_l)^{-2})$ independent samples from $a_i$.
Then, we look at the fraction of iterations in which the empirical mean of $a_i$ is smaller than $(c_l+c_r)/2$.
If the fraction is larger than $t$, the algorithm returns $\True$. Otherwise, it returns $\False$.

\begin{algorithm}[t]
\LinesNumbered
\setcounter{AlgoLine}{0}
	\caption{\FRACTIONTEST($S,c_l,c_r,\delta,t,\epsilon$)}
	\label{algo:ESTIMATE-PROCEDURE}
	\KwData{Arm set $S$, range parameters $c_l,c_r$, confidence level $\delta$, threshold $t$, approximate parameter $\epsilon$.}
	\smallskip
	$\cnt \leftarrow 0$ 
	
	$\tot \leftarrow \ln(2\cdot\delta^{-1}) (\epsilon/3)^{-2}/2$
	
	\For{$i = 1$ to $\tot$}{
		Pick a random arm $a_i \in S$ uniformly.
		
		$\hamean{a_i} \leftarrow \textrm{\UNIFORMSAMPLING}(\{a_i\},(c_r-c_l)/2,\epsilon/3)$
		
		\lIf{$\hamean{a_i} < (c_l+c_r)/2$}{
			$\cnt \leftarrow \cnt+1$
		}
	}
	\uIf{$\cnt/\tot > t$}{
		{\bf Return} $\True$ 
	}\Else{
	{\bf Return} $\False$ 
}
\end{algorithm}

For ease of notation, 
we define $S^{\ge c}:= \{\amean{a} \ge c \betw a \in S \}$, i.e.,
all arms in $S$ with means at least $c$. 
Similarly, we can define $S^{>c},S^{\le c}$ and $S^{< c}$. 

\begin{lemma}\label{lm:ESTIMATE-PROCEDURE}
	Suppose $\epsilon < 0.1$ and $t \in (\epsilon,1-\epsilon)$. With probability $1-\delta$,
	the following hold:
	\begin{itemize}[leftmargin=*]
		\item If \FRACTIONTEST\ outputs $\True$, then $|S^{>c_r}| < (1-t+\epsilon) |S|$
		(or equivalently $|S^{\le c_r}| > (t-\epsilon)|S|$).
		\item If \FRACTIONTEST\ outputs $\False$, then $|S^{<c_l}| < (t+\epsilon) |S|$
		(or equivalently $ |S^{\ge c_l}| > (1-t-\epsilon)|S|$).
	\end{itemize}
	Moreover, the number of samples taken by the algorithm is $O(\ln\delta^{-1}\epsilon^{-2}\Delta^{-2} \ln \epsilon^{-1})$, in which $\Delta = c_r-c_l$.
	If $\epsilon$ is a fixed constant, then the number of samples is simply $O(\ln \delta^{-1} \Delta^{-2})$.
\end{lemma}

The proof can be found in Section~\ref{app:missingpf-upper}.

\begin{algorithm}[t]
	\LinesNumbered
	\setcounter{AlgoLine}{0}
	\caption{\ELIMINATION($S,c_l,c_r,\delta$)}
	\label{algo:ELIMINATION-PROCEDURE}
	\KwData{Arm set $S$, range parameters $c_l,c_r$, confidence level $\delta$.}
	\KwResult{A set of arms after elimination.}
	$S_1 \leftarrow S$
	
	$c_m \leftarrow (c_l+c_r)/2$
	
	\For{$r = 1$ to $+\infty$}{
		$\delta_{r} = \delta/(10\cdot2^{r})$
		
		\uIf{{\em \FRACTIONTEST}$(S_r,c_l,c_m,\delta_r,0.075,0.025)$}{
			\UNIFORMSAMPLING$(S_r,(c_r-c_m)/2,\delta_r)$
			
			$S_{r+1} \leftarrow \Bigl\{  a \in S_r \mid \hamean{a} > (c_m+c_r)/2 \Bigr\} $
		}\Else{
		\bf Return $S_r$
	}
}
\end{algorithm}

\topic{4. Eliminating arms}
The final ingredient is an elimination procedure, which can be used to eliminate most arms below a given threshold.
The procedure takes four parameters $(S,c_l,c_r,\delta)$ as input, 
where $S$ is a set of arms, $c_l<c_r$ are the range parameters, and $\delta$
is the confidence level.
It outputs a subset of $S$
and guarantees that upon termination, most of the remaining arms have means at least $c_l$ with probability $1-\delta$.

Now, we describe the procedure \ELIMINATION\, which runs in iterations.
It maintains the current set $S_r$ of arms, which is initially $S$.
In each iteration, it first applies  
\FRACTIONTEST$(S_r,c_l,c_m,\delta_r,0.075,0.025)$ on $S_r$,
where $c_m=(c_r+c_r)/2$.
If \FRACTIONTEST\ returns $\True$, which means
that there are at least 5\% fraction of arms with small means in $S_r$, we sample all arms in $S_r$ uniformly by calling \UNIFORMSAMPLING$(S_r,(c_r-c_m)/2,\delta_r)$
where $\delta_r=\delta/(10\cdot 2^r)$,
and retain those with empirical means at least $(c_m+c_r)/2$.
If \FRACTIONTEST\ returns $\False$ (meaning that 90\% arms have means at least $c_l$) 
the algorithm terminates and returns the remaining arms.
The guarantee of \ELIMINATION\ is summarized in the following lemma.
The proof can be found in Section~\ref{app:missingpf-upper}.

\begin{lemma}\label{lm:ELIMINATION-PROCEDURE}
	Suppose $\delta < 0.1$.
	Let $S' = \ELIMINATION(S,c_l,c_r,\delta)$. 
	Let $A_1$ be the best arm among $S$, with mean $\amean{A_1} \ge c_r$.
	Then with probability at least $1-\delta$, the following statements hold
	\begin{enumerate}[leftmargin=*]
		\item $A_1 \in S'$ \,\, (the best arm survives);
		\item $|S'^{\le c_l}| < 0.1 |S'|$ \,\, (only a small fraction of arms have means less than $c_l$);
		\item The number of samples is $O(|S| \ln \delta^{-1} \Delta^{-2})$, in which $\Delta = c_r-c_l$.
	\end{enumerate}
\end{lemma}

\subsection{Our Algorithm}
\label{subsec:bestarmalgorithm}

Now, everything is ready
to describe our algorithm \DELIMINATION\
for \bestarm.
We provide a high level description here.
All detailed parameters can be found in Algorithm~\ref{algo:BESTARM}.
The algorithm runs in rounds.
It maintains the current set $S_r$ of arms.
Initially, $S_1$ is the set of all arms $S$.
In round $r$, the algorithm tries to eliminate a set of suboptimal arms, while makes sure 
the best arm is not eliminated.
First, it applies the \MEDIANELIMINATION\ procedure to find an 
$\epsilon_{r}/4$-optimal arm, where $\epsilon_{r}=2^{-r}$.
Suppose it is $a_r$. Then, we take a number of samples from $a_r$ to estimate its mean
(denote the empirical mean by $\hamean{a_r}$).
Unlike previous algorithms in~\cite{even2002pac,karnin2013almost},
which eliminates either a fixed fraction of arms or those arms with mean much less than $a_r$,
we use a \FRACTIONTEST\ to see whether there are many arms with mean much less than $a_r$.
If it returns $\True$, we apply the \ELIMINATION\ procedure to eliminate those arms
(for the purpose of analysis, we need to use \MEDIANELIMINATION\ again, but with a tighter confidence level,
to find an $\epsilon_{r}/4$-optimal arm $b_r$).
If it returns $\False$, the algorithm decides that it is not judicious to do elimination in this round
(since we need to spend a lot of samples, but only discard very few arms, which is wasteful),
and simply sets $S_{r+1}$ to be $S_r$, and then proceeds to the next round.

\begin{algorithm}[t]
\LinesNumbered
\setcounter{AlgoLine}{0}
	\caption{\DELIMINATION($S,\delta$)}\label{algo:BESTARM}
	\label{algo:BESTARM}
	\KwData{Arm set $S$, confidence level $\delta$.}
	\KwResult{The best arm.}
\smallskip	
	$h \leftarrow 1$
	
	$S_1 \leftarrow S$
	
	\For{$r = 1$ to $+\infty$}{
		\uIf{$|S_r| = 1$}{
			\(\bf Return\) the only arm in $S_r$ 
		}
		$\epsilon_{r} \leftarrow 2^{-r}$ 
		
		$\delta_{r} \leftarrow \delta/50r^2$
		 
		$a_r \leftarrow \MEDIANELIMINATION(S_r,\epsilon_r/4,0.01)$. \label{line:ME1} 
		
		$\hamean{a_r} \leftarrow \textrm{\UNIFORMSAMPLING}(\{a_r\},\epsilon_r/4,\delta_r)$ \label{line:US1}
		
		\uIf{ {\em \FRACTIONTEST}$(S_r,\hamean{a_r}-1.5\epsilon_r,\hamean{a_r}-1.25\epsilon_r,\delta_r,0.4,0.1)$} { \label{line:FT}
			$\delta_h \leftarrow \delta/50h^2$
			
			$b_r \leftarrow \MEDIANELIMINATION(S_r,\epsilon_r/4,\delta_h)$ \label{line:ME2}
			 
			$\hamean{b_r} \leftarrow \UNIFORMSAMPLING(\{b_r\},\epsilon_r/4,\delta_h)$ \label{line:US2}
			
			$S_{r+1} \leftarrow \ELIMINATION(S_r,\hamean{b_r}-0.5\epsilon_r,\hamean{b_r}-0.25\epsilon_r,\delta_h)$ \label{line:EL} 
			
			$h \leftarrow h + 1$ \label{line:INC-H}
		}\Else{
	
		$S_{r+1} \leftarrow S_{r}$
		}
	}
\end{algorithm}

We devote the rest of the section to prove that Algorithm~\ref{algo:BESTARM} indeed solves the \bestarm\ problem 
and achieves the sample complexity stated in Theorem~\ref{theo:BESTARMPACD}. 
To simplify the argument, we first describe some event we will condition on for the rest of the proof,
and show the algorithm indeed finds the best arm under the condition.

\begin{lemma}\label{lm:GOODEVENT}
	Let $\event_G$ denote the event that all procedure calls in line~\ref{line:US1},~\ref{line:FT},~\ref{line:ME2},~\ref{line:US2},~\ref{line:EL} return correctly for all rounds.
	$\event_G$ happens with probability at least $1-\delta$.
	Moreover, conditioning on $\event_G$, the algorithm outputs the correct answer.
\end{lemma}
\begin{proof}
	By Lemma~\ref{lm:UNIFORM-SAMPLE-PROCEDURE}, Lemma~\ref{lm:ESTIMATE-PROCEDURE} and Lemma~\ref{lm:ELIMINATION-PROCEDURE}, 
	we can simply bound the total error probability with a union bound over all procedure calls in all rounds:
	\[
	\sum_{r=1}^{+\infty} 2\delta_r +¡¡\sum_{h=1}^{+\infty} 3\delta_h \le \delta\cdot5\sum_{i=1}^{+\infty} 1/50i^2 \le \delta.
	\]
	
	To prove the correctness, it suffices to show that the best arm $A_1$ is never eliminated in line~\ref{line:EL}.
	Conditioning on event $\event_G$, for all $r$, we have $\amean{b_r} \le \amean{A_1}$ and $|\hamean{b_r} - \amean{b_r}| < \epsilon_r/4$, 
	thus $\hamean{b_r} < \amean{A_1}+\epsilon_r/4$. Clearly, this means $\amean{A_1} \ge \hamean{b_r}-0.25\epsilon_r$.
	Then by Lemma~\ref{lm:ELIMINATION-PROCEDURE}, we know that $A_1$ has survived round $r$.
\end{proof}

Note that the correctness of \MEDIANELIMINATION\ (line~\ref{line:ME1}) is not included in event $\event_G$.

\subsection{Analysis of the running time}
\label{subsec:bestarmrunningtime}

We use $\alg$ to denote Algorithm~\ref{algo:BESTARM}.
Let 
$$
T(\delta,I) = \sum_{i=2}^{n} \Gap{i}^{-2} \left(\ln\delta^{-1}+\ln\ln \min(n,\Gap{i}^{-1})\right) + \Gap{2}^{-2}\ln\ln \Gap{2}^{-1}
$$
be the target upper bound we want to prove.
We need to prove $\Ex_{\alg_{\delta},I}[\tau\mid \event_G]=O(T(\delta,I))$ for $\delta < 0.1$.
In the rest of the proof, we condition on the event $\event_G$, unless state otherwise. Let $A_1$ denote the best arm.

We need some additional notations. First, for all $s \in \mathbb{N}$, define the sets of arms $\armset{s}$, $\bigarmset{s}$ and $\smaarmset{s}$ as:
\[
\armset{s} = \{ a \mid 2^{-s} \le \Gap{a} < 2^{-s+1}\},\quad
\bigarmset{s} = \bigcup_{r=s}^{+\infty} \armset{r}, \quad
\smaarmset{s} = \bigcup_{r=1}^{s} \armset{r}
\]
Note that the best arm is not in any of the above set.
Let $\epsilon_{s} = 2^{-s}$. It is also convenient to use the following 
equivalent definitions for $\bigarmset{s}$ and $\smaarmset{s}$:
\[
\bigarmset{s} = \{ a \mid \amean{A_1} - 2\epsilon_{s} < \amean{a} < \amean{A_1} \},\quad
\smaarmset{s} = \{ a \mid \amean{a} \le \amean{A_1} - \epsilon_{s} \}
\]
Let $\maxs$ be the maximum $s$ such that $\armset{s}$ is not empty.

We start with a lemma which concerns the ranges of  $\hamean{a_r}$ and $\hamean{b_r}$.

\begin{lemma}\label{lm:a-r-b-r-approx}
	Conditioning on $\event_G$, the following statements hold:
	\begin{enumerate} [leftmargin=*]
		\item For any round $r$, $\hamean{a_r} < \amean{A_1} + \epsilon_r/4$. In addition, 
		if  \MEDIANELIMINATION\ (line \ref{line:ME1}) returns an $\epsilon_r/4$-approximation correctly in that round, 
		we also have $\hamean{a_r} > \amean{A_1} - 2\epsilon_r/4$.
		\item For any round $r$, if $b_r$ exists (the algorithm enters line~\ref{line:ME2}), then $\hamean{b_r} < \amean{A_1} + \epsilon_r/4$ and $\hamean{b_r} > \amean{A_1} - 2\epsilon_r/4$.
	\end{enumerate}
\end{lemma}
\begin{proof}	
	Conditioning on $\event_G$, 
	we have that $|\hamean{a_r} - \amean{a_r}| < \epsilon_r/4$. Clearly $\amean{a_r} \le \amean{A_1}$. Hence, $\hamean{a_r} < \amean{A_1} + \epsilon_r/4$. If $a_r$ is an $\epsilon_r/4$-approximation of $A_1$, we have $\amean{a_r} \ge \amean{A_1}-\epsilon_r/4$.
	Then, we can see $\hamean{a_r} > \amean{A_1} - 2\epsilon_r/4$. Note that conditioning on $\event_G$, $b_r$ is always an  $\epsilon_r/4$-approximation of $A_1$.
	Hence the second claim follows exactly 
	in the same way.
\end{proof}

Then we give an upper bound of $h$ (updated in line~\ref{line:INC-H}). 
Note that $h$ indicates how many times we pass the 
\FRACTIONTEST\ and 
execute \ELIMINATION\ (line~\ref{line:EL}).
Indeed, in the following analysis, we only need an upper bound of $h$ during the first $\maxs$ rounds. We introduce the definition first.

\newcommand{\hI}{h_I}

\begin{defi}
	Given an instance $I$, 
	conditioning on $\event_G$, 
	we denote the maximum value of $h$ during the first $\maxs$ rounds as $\hI$. 
	It is easy to see that $\hI \le \maxs$.
\end{defi}

\begin{lemma}\label{lm:GOOD-PROPERTYS2}
	$\hI = O(\ln n)$ for any instance $I$.
\end{lemma}

\begin{proof}	
	Suppose we enter line \ref{line:ME2} at round $r$. 
	By Lemma~\ref{lm:a-r-b-r-approx}, we have $\hamean{a_r} < \amean{A_1} + \epsilon_r/4$ and $\hamean{b_r} > \amean{A_1} - 2\epsilon_r/4$. Hence $\hamean{b_r} > \hamean{a_r} - 0.75\epsilon_r$.
	
	By Lemma~\ref{lm:ESTIMATE-PROCEDURE}, 
	we know there is at least $0.3$ fraction of arms in $S_r$ with means $\le \hamean{a_r}-1.25\epsilon_r$. 
	But by Lemma~\ref{lm:ELIMINATION-PROCEDURE}, we know that after executing
	line \ref{line:EL}, there are at most $0.1$ fraction of arms in $S_r$ 
	with means $\le \hamean{b_r}-0.5\epsilon_r$.
	By noting that $\hamean{b_r}-0.5\epsilon_r > \hamean{a_r}-1.25\epsilon_r$. 
	we can see that 
	$|S_r|$ drops by at least a constant fraction whenever we enter line~\ref{line:EL}. 
	Therefore, $h$ can increase by 1 for at most $O(\ln n)$ times.
\end{proof}

\begin{rem}
	Conditioning on $\event_G$, for all round $r \le \maxs$, we can see $h \le r$. Thus, $h \le \min(\hI,r)$, and $\ln \delta_h^{-1} = O(\ln \delta^{-1} +\lnmin{\hI}{r})$. 
\end{rem}

Now, we describe some behaviors of \MEDIANELIMINATION\ and \ELIMINATION\ before we analyze the sample complexity.

\begin{lemma}\label{lm:if-output-true}
	If \FRACTIONTEST\ (line~\ref{line:FT}) outputs $\True$, we know 
	that there are $>0.3$ fraction of arms with means $\le \amean{A_1} - \epsilon_r$ in $S_r$. 
	In other words, $|\smaarmset{r} \cap S_r| > 0.3 |S_r|$. Moreover, we have that $|\smaarmset{r} \cap S_{r+1}| \le 0.1 |S_{r+1}|$ (i.e., we can eliminate a significant portion
	in this round). 
\end{lemma}

\begin{proof}
	By Lemma~\ref{lm:a-r-b-r-approx}, 
	we can see that $\hamean{a_r} < \amean{A_1} + \epsilon_r/4$, $\hamean{b_r} > \amean{A_1} - 2\epsilon_r/4$. Now consider the parameters for \FRACTIONTEST. Let  $c_r=\hamean{a_r}-1.25\epsilon_r$. Then 
	we have that $c_r < \amean{A_1} - \epsilon_r$.
	
	By Lemma~\ref{lm:ESTIMATE-PROCEDURE},
	when \FRACTIONTEST\ outputs $\True$, we know that there are $>(0.4-0.1)=0.3$ fraction of arms 
	with means $\le c_r \le \amean{A_1} - \epsilon_r$ in $S_r$. Clearly, $\amean{a} \le \amean{A_1}-\epsilon_r$ is equivalent to $a \in \smaarmset{r}$ for an arm $a$.
	
	Now consider the parameters for \ELIMINATION. Let $c_l =\hamean{b_r}-0.5\epsilon_r$. Then $c_l > \amean{A_1}- \epsilon_r$. We also note that $\amean{a} \le \amean{A_1}-\epsilon_r$ is equivalent to $a \in \smaarmset{r}$ for an arm $a$.
	Then by Lemma~\ref{lm:ELIMINATION-PROCEDURE}, after the elimination, we have
	$|\smaarmset{r} \cap S_{r+1}| \le | S_{r+1}^{\le c_l} | \le  0.1 |S_{r+1}|$.
\end{proof}

\begin{lemma}\label{lm:CORRECT-MEDELIM}
	Consider a round $r$. Suppose \MEDIANELIMINATION\ (line~\ref{line:ME1}) returns a correct $\epsilon_r/4$-approximation $a_r$. 
	Then, the following statements hold:
	\begin{enumerate}[leftmargin=*]
		\item If \FRACTIONTEST\ (line~\ref{line:FT}) outputs $\True$, 
		we know there are $>0.3$ fraction of arms with means $\le \amean{A_1} - \epsilon_r$ in $S_r$. 
		In other words, $|\smaarmset{r} \cap S_r| > 0.3 |S_r|$. 
		Moreover, we have that $|\smaarmset{r} \cap S_{r+1}| \le 0.1 |S_{r+1}|$.
		\item If it outputs $\False$, we know there are at least $0.5$ fraction of arms with 
		means at least $\amean{A_1} - 2\epsilon_r$ in $S_r$. In other words, $|\bigarmset{r} \cap S_r|+1 >0.5 |S_r|$.\
	\end{enumerate}
\end{lemma}

\begin{proof}	
	Since \MEDIANELIMINATION\ (line~\ref{line:ME1}) returns the correctly, by Lemma~\ref{lm:a-r-b-r-approx}, $\hamean{a_r} > \amean{A_1} - 2\epsilon_r/4$ and $\hamean{a_r} < \amean{A_1} + \epsilon_r/4$. Now consider the parameters for $\textrm{\FRACTIONTEST}$, let $c_l=\hamean{a_r} -1.5\epsilon_r$ and $c_r=\hamean{a_r}-1.25\epsilon_r$ , It is easy to see that
	$c_l > \amean{A_1} - 2\epsilon_r$, and $c_r < \amean{A_1} - \epsilon_r$.
	
	The first claim just follows from Lemma~\ref{lm:if-output-true} (note that Lemma~\ref{lm:if-output-true} does not require the output of \MEDIANELIMINATION\ (line~\ref{line:ME1}) being correct).
	
	By Lemma~\ref{lm:ESTIMATE-PROCEDURE}, if \FRACTIONTEST\ outputs $\False$, we know there are at least $(1-0.4-0.1)=0.5$ fraction of arms with means $\ge c_l > \amean{A_1} - 2\epsilon_r$ in $S_r$.  
	For an arm $a$, 
	$\amean{a} > \amean{A_1}-2\epsilon_r$ is equivalent to $a \in \bigarmset{r}$ or 
	$a$ is the best arm $A_1$ itself.
\end{proof}

We also need the following lemma describing the behavior of the algorithm when $r > \maxs$.

\begin{lemma}\label{lm:big-r}
	For each round $r > \maxs$, the algorithm terminates if \MEDIANELIMINATION\  returns an $\epsilon_r/4$-approximation correctly,
	which happens with probability at least $0.99$.
\end{lemma}

\begin{proof}
	If we already have $|S_r| = 1$ at the beginning of round $r$, then there is nothing to prove since it halts immediately. 
	So we can assume $|S_r| > 1$.
	
	Suppose in round $r$, \MEDIANELIMINATION\ returns a correct $\epsilon_r/4$-approximation. 
	Conditioning on $\event_G$, 
	\FRACTIONTEST\ must output $\True$. 
	Since if it outputs $\False$, then by Lemma~\ref{lm:CORRECT-MEDELIM}, 
	$|\bigarmset{r} \cap S_r| + 1 > 0.5 |S_r|$. But $\bigarmset{r} = \emptyset$ from $r > \maxs$. So $1 >0.5|S_r|$, which implies $|S_r| = 1$, rendering a contradiction.
	
	Then, by Lemma~\ref{lm:CORRECT-MEDELIM}, $|\smaarmset{r} \cap S_{r+1}| \le 0.1 |S_{r+1}|$, which is equivalent to
	the fact that
	$|\bigarmset{r+1} \cap S_{r+1}|+1 \ge 0.9 |S_{r+1}|$. 
	Note that $|\bigarmset{r+1} \cap S_{r+1}| = 0$ as $\bigarmset{r+1} = \emptyset$. 
	So it holds that $1\ge 0.9|S_{r+1}|$, or equivalently $|S_{r+1}|=1$. Thus it terminates right after round $r$.
	
	Finally, as \MEDIANELIMINATION\ returns correctly an $\epsilon_r/4$-approximation with probability at least $0.99$, the proof is completed.
\end{proof}

We analyze the expected number of samples used for each subprocedure separately. 
We first consider \FRACTIONTEST\ (line~\ref{line:FT}) and \UNIFORMSAMPLING\ (line~\ref{line:US1},~\ref{line:US2}) 
and prove the following simple lemma.

\begin{lemma}\label{lm:PROCGROUP2}
	Conditioning on event $\event_G$, the expected number of samples incurred by \FRACTIONTEST\ (line~\ref{line:FT}) and \UNIFORMSAMPLING\ (line~\ref{line:US1},~\ref{line:US2}) is 
	$$
	O\left( \Gap{2}^{-2} (\ln \delta^{-1}+\ln\ln \Gap{2}^{-1} ))\right).
	$$
\end{lemma}

\begin{proof}
	By Lemma~\ref{lm:big-r}, for any round $r > \maxs$, the algorithm halts w.p. at least $0.99$. So we can bound the expectation of samples incurred by \FRACTIONTEST\ and \UNIFORMSAMPLING\ by:
	\begin{align*}
	\sum_{r=1}^{\maxs} c_4\cdot\ln\delta_r^{-1} \epsilon_r^{-2} + \sum_{r=\maxs+1}^{+\infty} c_4 \cdot (0.01)^{r-\maxs-1} \ln\delta_r^{-1} \epsilon_r^{-2}
	\end{align*}
	Here, $c_4$ is a constant large enough such that in round $r$ the number of samples taken by \FRACTIONTEST\ and \UNIFORMSAMPLING\ together is bounded by $c_4 \cdot \ln\delta_r^{-1}\epsilon_r^{-2}$. It is not hard to see the first sum is dominated by the last term while the second sum is dominated by the first. So the bound is 
	$O(\Gap{2}^{-2}( \ln \delta^{-1} + \ln\ln \Gap{2}^{-1}))$ since $\maxs$ is $\Theta(\ln \Gap{2}^{-1})$.
\end{proof}

Next, we analyze {\em\MEDIANELIMINATION} (line~\ref{line:ME1},~\ref{line:ME2}) and {\em\ELIMINATION} (line~\ref{line:EL}). In the following, we only consider samples due to these two procedures.

\begin{lemma}\label{lm:PROCGROUP1}
	Conditioning on event $\event_G$, the expected number of samples incurred by \MEDIANELIMINATION\ (line~\ref{line:ME1},~\ref{line:ME2}) and \ELIMINATION\ (line~\ref{line:EL}) is 
	\[
	O\left(\sum_{i=2}^{n} \Gap{i}^{-2} (\ln \delta^{-1} + \lnmin{\hI}{\ln \Gap{i}^{-1}})\right).
	\]
\end{lemma}

We devote the rest of the section to the proof of Lemma~\ref{lm:PROCGROUP1}, which is more involved.
We first need a lemma which provides an upper bound on the number of samples for one round.

\begin{lemma}\label{lm:PULLING-BOUND}
	Let $c_3$ be a sufficiently large constant.
	The number of samples in round $r \le \maxs$ can be bounded by	
	\[
	\begin{cases}
	c_3 \cdot |S_r| \epsilon_r^{-2} & \quad \text{if \FRACTIONTEST\ outputs $\False$.}\\
	c_3 \cdot |S_r| \epsilon_r^{-2} (\ln\delta^{-1}+\lnmin{\hI}{r}) & \quad \text{if \FRACTIONTEST\ outputs $\True$.}\\
	\end{cases}
	\]
	If $r > \maxs$, the number of samples can be bounded by
	\[
	c_3 \cdot |S_r| \epsilon_r^{-2} (\ln\delta^{-1}+\ln(\hI+r-\maxs)).
	\]
\end{lemma}
\begin{proof}
	Note that conditioning on event $\event_G$, \ELIMINATION\ always returns correctly.    
	Let $c_3$ be a constant such that $\MEDIANELIMINATION(S_r,\epsilon_r/4,0.01)$ takes no more than $c_3/3\cdot |S_r|\epsilon_r^{-2}$ samples, and $\MEDIANELIMINATION(S_r,\epsilon_r/4,\delta_h)$ and 
	$\ELIMINATION(S_r,\hamean{a_r}-1.5\epsilon_r,\hamean{a_r}-1.25\epsilon_r,\delta_h)$ both take no more than
	$c_3/3\cdot|S|\epsilon_r^{-2} (\ln\delta^{-1}+\lnmin{\hI}{r})$ samples conditioning on $\event_G$. 
	The later one is due to the fact that $\ln \delta_h^{-1}=O(\ln\delta^{-1} + \lnmin{\hI}{r})$ for $r \le \maxs$. If $r > \maxs$, we have $h \le \hI +r - \maxs$, then the bounds follow from a simple calculation.
\end{proof}

\begin{proofof}{Lemma~\ref{lm:PROCGROUP1}}
	We prove the lemma inductively. 
	Let $T(r,\Nsmall)$ denote the maximum expected total number of samples the algorithm takes at and after round $r$, 
	when $|S_r \cap \smaarmset{r-1}| \le \Nsmall$. In other words, it is an upper bound of the expected number of samples 
	we will take further, provided that we are at the beginning of round $r$ and there are at most $\Nsmall$ ``small" arms left. 
	By definition, $T(1,0)$ is the final upper bound for the total expected number of samples taken by the algorithm.
	
	Let $c_1=4c_3,c_2=60c_3$.
	\footnote{
		Although these constants are chosen somewhat arbitrarily, they need to satisfy certain relations (will be clear from the proof)
		and it is necessary to make them explicit.
	}
	For ease of notation, we let $\lx{s} = \ln(\min(\hI,s))$. 
	
	We first consider the case where $r = \maxs+1$ and prove the following bound of $T(r, \Nsmall)$:
	\begin{align}
	\label{eq:recurrence1} 
	T(r,\Nsmall) \le (\ln\delta^{-1}+\ln{\hI})c_1 \cdot \Nsmall\cdot\epsilon_{r}^{-2}.
	\end{align}
	
	Clearly there is nothing to prove for the base case $\Nsmall = 0$.
	So we consider $\Nsmall \ge 1$.
	Now, suppose the first round after $r$ in which \MEDIANELIMINATION\ (line~\ref{line:ME1}) 
	returns correctly an $\epsilon_r/4$-approximation is $r' \ge r$. 
	Clearly, this happens with probability at most $0.01^{r'-r}$ (all rounds in between fail).
	By Lemma~\ref{lm:big-r}, the algorithm terminates after round $r'$.
	Moreover, we have $|\bigarmset{r} \cap S_r| = 0$ since $\bigarmset{r} = \emptyset$. 
	So $S_r$ consists of the single best arm $A_1$ and $\Nsmall$ arms in $\smaarmset{r-1}$.
	By Lemma~\ref{lm:PULLING-BOUND}, the number of samples is bounded by 
	$$
	\sum_{i=r}^{r'} c_3 (\ln\delta^{-1} + \ln[\hI+i-\maxs])(1+\Nsmall)\epsilon_{i}^{-2}.
	$$
	
	\noindent
	Hence, we can bound $T(r,\Nsmall)$ as follows:
	\begin{align*}
	T(r,\Nsmall) \le
	&\sum_{r'=r}^{+\infty} (0.01)^{r'-r} \cdot \sum_{i=r}^{r'} c_3 \Big(\ln\delta^{-1} + \ln[\hI+i-\maxs])(1+\Nsmall)\epsilon_{i}^{-2}\\
	\le& 2c_3 \Nsmall \sum_{r'=r}^{+\infty} (0.01)^{r'-r} \cdot \sum_{i=r}^{r'}(\ln\delta^{-1} + \ln[\hI+i-\maxs])\epsilon_{i}^{-2}\\
	\le& 3c_3 \Nsmall \sum_{r'=r}^{+\infty}  (0.01)^{r'-r} \cdot (\ln\delta^{-1} + \ln[\hI+r'-\maxs])\epsilon_{r'}^{-2}\\
	\le& 4c_3 \Nsmall (\ln\delta^{-1} + \ln \hI)\epsilon_{r}^{-2}.
	\end{align*}
	
	\noindent
	Now, we analyze the more challenging case where $r \leq \maxs$.
	For ease of notation, we let
	\[
	C_{r,s} = (\ln \delta^{-1} + \lx{s}) \sum_{i=r}^{s} \epsilon_{i}^{-2}.
	\]
	When $r>s$, we let $C_{r,s} = 0$. 
	We also define 
	$$
	\Pot{r} = c_2 \cdot\left(\sum_{s=r}^{+\infty} C_{r,s} |\armset{s}| + C_{r,\maxs}\right).
	$$ 
	$\Pot{r}$ can be viewed as a potential function. 
	Notice that when $r>\maxs$, we have $\Pot{r}=0$.
	We are going to show inductively that 
	\begin{equation}
	\label{eq:recurrence}
	T(r,\Nsmall) \le (\ln\delta^{-1}+\lx{r})\cdot c_1 \cdot \Nsmall\cdot\epsilon_{r}^{-2} + \Pot{r}.
	\end{equation}
	We note that \eqref{eq:recurrence1} is in fact consistent with \eqref{eq:recurrence}
	(when $r>\maxs$, we have $\Pot{r}=0$ and $\lx{r} = \ln \hI$).
	
	The induction hypothesis assumes that the inequality holds for $r+1$ and all $\Nsmall$. 
	We need to prove that it also holds for $r$ and all $\Nsmall$. 
	Now we are at round $r$.
	Conditioning on event $\event_G$, there are three cases we need to consider. We state the following lemmas, each analyzes one case, which together imply our time bound. Their proofs are not difficult but somewhat tedious. So we defer them to  Section~\ref{app:missingpf-upper}. 
	
	
	\begin{lemma}\label{lm:case-1}
	Suppose that \MEDIANELIMINATION\ (line~\ref{line:ME1}) returns an $\epsilon_r/4$-approximation of the best arm $A_1$, and \FRACTIONTEST\ outputs $\True$. The expected number of samples taken at and after round $r$ is bounded by
	$$
	(\ln\delta^{-1} +  \lx{r})c_3 \Nsmall\epsilon_r^{-2} + \Pot{r}.
	$$
	\end{lemma}
	
	\begin{lemma}\label{lm:case-2}
	Suppose that
	\MEDIANELIMINATION\ (line~\ref{line:ME1}) returns an $\epsilon_r/4$-approximation of the best arm $A_1$, and \FRACTIONTEST\ outputs $\False$. The expected number of samples taken at and after round $r$ is bounded by $\Pot{r}$.
	\end{lemma}
	
	\begin{lemma}\label{lm:case-3}
	Suppose that
	\MEDIANELIMINATION\ (line~\ref{line:ME1}) returns an arm which is not an $\epsilon_r/4$-approximation of the best arm $A_1$. The expected number of samples taken at and after round $r$ is bounded by
	$$
	(\ln\delta^{-1}+\lx{r})(c_3+5c_1) \Nsmall\epsilon_{r}^{-2} + \Pot{r}.
	$$
	\end{lemma}
	
	Note that the bound in Lemma~\ref{lm:case-3} is larger than we need to prove (in particular, the constant is larger).
	So, we need to combine three cases together as follows:
	
	Recall that \MEDIANELIMINATION\ (line~\ref{line:ME1}) returns correctly an $\epsilon_r/4$-approximation with probability $p$ ($p \ge 0.99$).
	By Lemma~\ref{lm:case-1}, Lemma~\ref{lm:case-2} and Lemma~\ref{lm:case-3}, we have that
	\begin{align*}
	T(r,\Nsmall) &\le (\ln\delta^{-1}+\lx{r})(p\cdot (c_3\cdot \Nsmall\cdot\epsilon_r^{-2}
	) + (1-p) \cdot (c_3+5c_1) \cdot \Nsmall\epsilon_{r}^{-2} ) + \Pot{r} \\
	&\le (\ln\delta^{-1}+\lx{r})((c_3 + 0.05c_1)\cdot \Nsmall \cdot \epsilon_r^{-2}) + \Pot{r} \tag{$c_3+ (1-p)\cdot5c_1 \le c_3 +0.05c_1$} \\
	&\le (\ln\delta^{-1}+\lx{r})c_1\cdot \Nsmall \cdot \epsilon_r^{-2} + \Pot{r} \tag{$c_3 +0.05c_1 \le c_1$}
	\end{align*}
	This finishes the proof of \eqref{eq:recurrence}.
	Hence, the number of samples is bounded by $T(1,0) \le \Pot{1}$. 
	Note that $C_{1,s} \le 2(\ln \delta^{-1} + \lx{s})\epsilon_{s}^{-2}$ for any $s$. 
	By a simple calculation of $\Pot{1}$, we can see that the overall sample complexity for \MEDIANELIMINATION\ and \ELIMINATION\ is 
	$$
	T(1,0) \le \Pot{1}=O\left(\sum_{i=2}^{n} \Gap{i}^{-2} (\ln \delta^{-1} + \lnmin{\hI}{\ln\Gap{i}^{-1}})\right).
	$$
	This finally completes the proof of Lemma~\ref{lm:PROCGROUP1}.
\end{proofof}
Putting Lemma~\ref{lm:PROCGROUP1} and Lemma~\ref{lm:PROCGROUP2} together, 
%
we have the following corollary.

\begin{cor}\label{cor:final-time-bound}
	With probability $1-\delta$, Algorithm~\ref{algo:BESTARM}
	returns the correct answer for \bestarm\ and take at most $O(T)$ sample in expectation, where
	$$T=\sum_{i=2}^{n} \Gap{i}^{-2} (\ln \delta^{-1} + \lnmin{\hI}{\ln\Gap{i}^{-1}}) + \Gap{2}^{-2} \ln\ln \Gap{2}^{-1}.$$
\end{cor}

The time bound in Theorem~\ref{theo:BESTARMPACD} is an immediate consequence of 
the above corollary and Lemma~\ref{lm:GOOD-PROPERTYS2}, which asserts $\hI = O(\ln n)$.

However, there is one subtlety:
we only provide 
a bound on the running time conditioning on $\event_G$
(i.e., $\Ex[T_{\alg}\mid \event_G]$).
With probability at most $\delta$
(when the event $\event_G$ fails),
we do not have any bound on the running time (or the number of samples) of the algorithm
(the algorithm may not terminate).
So strictly speaking, the overall expected
running time of \DELIMINATION\ is not bounded.
In fact, several previous algorithms in the literature
\cite{even2006action,karnin2013almost,gabillon2012best} 
have the same problem. 
However, it is possible to transform such an algorithm $\alg$ to
another $\alg'$ which
succeeds with probability $1-\delta$ and whose overall expected running time is bounded as $\Ex[T_{\alg'}]=O(\Ex[T_{\alg}\mid \event_G])$.
We are not aware such a transformation in the literature
and we provide one in Section~\ref{sec:simult}.

\eat{
\begin{proof}
	For a big constant $C$,
	it is easy to see $
	C\cdot\left(\sum_{i=2}^{n} \Gap{i}^{-2} (\ln \delta^{-1} + \lnmin{\hI}{\ln\Gap{i}^{-1}}) + \Gap{2}^{-2} \ln\ln \Gap{2}^{-1} \right)$ is a good time bound function according to Definition~\ref{defi:GOOD-FUNCTION}. Then it just follows from Theorem~\ref{theo:TRANSFORM1}.
\end{proof}
}


\subsection{Almost Instance Optimal Bound for Clustered Instances}
\label{subsec:clustered}

Recall $\armset{s} = \{ a \mid 2^{-s} \le \Gap{a} < 2^{-s+1}\}$.
\begin{defi}
We say 
an instance $I$ is {\em clustered}
if $|\{ \armset{i}\ne \emptyset \betw 1 \le i \le \maxs \}|$ is bounded by a constant. 
\end{defi}
In this case, we can obtain an almost instance optimal algorithm for such instances.
For this purpose, we only need to establish a tighter bound on $\hI$.

\begin{lemma}~\label{lm:another-bound}
	$\hI \le 2\cdot|\{ \armset{i}\ne \emptyset \betw 1 \le i \le \maxs \}|$.
\end{lemma}
\begin{proof}
	Let $s$ be an index such that $\armset{s}$ is not empty. Let $s'$ be the largest index $<s$ such that $\armset{s'}$ is not empty.
	If such index does not exist, let $s'=0$. We show that during rounds $s'+1,s'+2,\dotsc,s-1$, 
	we can only call \ELIMINATION\ (or equivalently, increase $h$) once.
	
	Suppose otherwise, we call \ELIMINATION\ in round $r$ and $r'$ such that $s'<r<r'<s$. We further assume that 
	there is no other call to \ELIMINATION\ between round $r$ and $r'$. Then clearly $S_{r'}=S_{r+1}$.
	
	Now by Lemma~\ref{lm:if-output-true}, we have $|\smaarmset{r} \cap S_{r+1}| \le 0.1 |S_{r+1}|$, but this means $|U_{r'} \cap S_{r'}| \le 0.1 |S_r'{}|$, as $U_{r'} = \smaarmset{r}$ and $S_{r'}=S_{r+1}$. Again by Lemma~\ref{lm:if-output-true}, this contradicts the fact that on round $r'$, \FRACTIONTEST\ outputs true.
	
	As $\armset{\maxs}$ is not empty, we can partition the rounds $1,2,\dotsc,\maxs$ into at most 
	$2\cdot|\{ \armset{i}\ne \emptyset \betw 1 \le i \le \maxs \}|$ groups. Each group is either a single round which corresponds to a non-empty set $\armset{s}$, 
	or the rounds between two rounds corresponding to two adjacent non-empty sets $\armset{s'}$ and $\armset{s}$. 
	Therefore, $h$ can increase at most by $1$ in each group, which concludes the proof.
\end{proof}

The following theorem is an immediate consequence of the above lemma and Corollary~\ref{cor:final-time-bound}.
\begin{theo}
	\label{thm:clustered}
	There is an \CORRECT\ algorithm for clustered instances, with 
	expected sample complexity 
	$$
	T(\delta, I)=O\left(\sum\nolimits_{i=2}^{n} \Gap{i}^{-2}\ln \delta^{-1} +\Gap{2}^{-2}\ln\ln\Gap{2}^{-1}\right).
	$$
\end{theo}

\begin{example}
	Consider a very simple yet important instance where 
	there are $n-1$ arms with mean $0.5$, and a single arm with mean $0.5+\Delta$. 
	In fact, a careful examination of all previous algorithms 
	(including \cite{jamieson2014lil,karnin2013almost}) 
	shows that they all require $\Omega\left(n\Delta^{-2}(\ln \delta^{-1}+\ln\ln\Delta^{-1})\right)$ samples
	even in this particular instance.
	However, our algorithm only requires $O\left(n\Delta^{-2}\ln\delta^{-1} + \Delta^{-2}\ln\ln\Delta^{-1} \right)$
	samples. Our bound is almost instance optimal, since the first term matches 
	the instance-wise lower bound $n\Delta^{-2}\ln\delta^{-1}$. 
	This is the best bound for such instances we can hope for.
\end{example}

\subsection{An Improved PAC Algorithm}
\label{subsec:improvePAC}

\begin{table}[]
	\centering
	\label{table:pac}
	\small{
		\begin{tabular}{|l|l|}
			\hline
			Source & Sample Complexity                                                           \\ \hline
			Even-Dar et al.
			\cite{even2002pac}
			& $n\epsilon^{-2} \cdot \ln \delta^{-1}$                \\ \hline
			Gabillon et al. \cite{gabillon2012best} 
			& $ \sum\nolimits_{i=2}^{n} \Gapepsilon{i}^{-2}  \left( 
			\ln \delta^{-1}+\ln \sum\nolimits_{i=2}^{n} \Gapepsilon{i}^{-2}  \right) $            \\ \hline
			Kalyanakrishnan et al. \cite{kalyanakrishnan2012pac}
			& $ \sum\nolimits_{i=2}^{n} \Gapepsilon{i}^{-2}  \left( 
			\ln \delta^{-1}+ \ln \sum\nolimits_{i=2}^{n} \Gapepsilon{i}^{-2}  \right) $ 
			\\ \hline 
			Karnin et al. \cite{karnin2013almost}
			& 
			$\sum\nolimits_{i=2}^{n} \Gapepsilon{i}^{-2} \left(\ln\delta^{-1}+\ln\ln\Gapepsilon{i}^{-1}\right)$
			\\ \hline
			This paper (Thm~\ref{thm:pacimproved})
			& 
			$\sum\nolimits_{i=2}^{n}
			\Gapepsilon{i}^{-2} \left(\ln \delta^{-1} + \ln\ln\min(n,\Gapepsilon{i}^{-1})\right) + 
			\Gapepsilon{2}^{-2} \ln\ln \Gapepsilon{2}^{-1} 
			$
			\\
			\hline
		\end{tabular}
	}
	\vspace{0.1cm}
	\caption{Sample complexity upper bounds for $(\epsilon, \delta)$-PAC algorithms. 
		Here, $\Gapepsilon{i} = \max(\Gap{i},\epsilon)$}
\end{table}

Finally, we discuss how to convert our algorithm to an $(\epsilon,\delta)$-PAC algorithm for \bestarm.
Our result improves several previous PAC algorithm, which we summarize
in Table~2.

\begin{theo}
	\label{thm:pacimproved} 
	For any $\epsilon < 0.5$ and $\delta <0.1$, 	
	there exists an 
	$(\epsilon,\delta)$-PAC algorithm for \bestarm, 
	with expected sample complexity
	$$
	T(\delta, I)=O\left(\sum_{i=2}^{n}
	\Gapepsilon{i}^{-2} (\ln \delta^{-1} + \ln\ln\min(n,\Gapepsilon{i}^{-1})) + 
	\Gapepsilon{2}^{-2} \ln\ln \Gapepsilon{2}^{-1} 
	\right),
	$$ 
	where $\Gapepsilon{i} = \max(\Gap{i},\epsilon)$.
\end{theo}

\begin{proof}
	Given parameters $\epsilon,\delta$, we run \DELIMINATION\ with confidence $\delta/2$ only 
	for the first $\lceil \ln \epsilon^{-1} \rceil$ rounds. 
After that, we invoke \MEDIANELIMINATION\ with confidence $\delta/2$ 
	to find an $\epsilon$-optimal arm among $S_{\lceil \log_2 \epsilon^{-1} \rceil}$. Clearly we are correct with probability at least $1-\delta$. 
	The analysis for the sample complexity is exactly the same as the original \DELIMINATION. 
\end{proof}

\section{A New Lower bound for \bestarm}
\label{sec:bestarmlowerbound}	

In this section,
we provide a new lower bound for \bestarm.
In particular,  we prove Theorem~\ref{thm:hard-case-exists}. 
From now on, $\delta$ is a fixed constant such that $0 < \delta < 0.005$. 
We use $[0,N]$ to denote the set of integers $\{0,1,\dotsc,N\}$. Throughout this section, we assume the distributions of all the arms are Gaussian with variance $1$.

Our proof for Theorem~\ref{thm:hard-case-exists} consists of two ingredients. The first one is a non-trivial lower bound for \sign, which we state here but defer its proof to the next section. 
The second one is a novel reduction from \sign\ to \bestarm, which turns the lower bound for \sign\ into the desired lower bound for \bestarm.

For stating the lower bound for \sign\, we introduce some notations first.
Let $\algp$ denote an algorithm for \sign, 
$A_\mu$ be an arm with mean $\mu$ (i.e., with distribution $\Normal(\mu,1)$), and we define
$T_{\algp}(\Delta) = \max(T_{\algp}(A_{\xi+\Delta}), T_{\algp}(A_{\xi-\Delta}))$. Then we have the following lower bound for \sign.
The proof is deferred to Section~\ref{sec:lbsign}.

\begin{lemma}~\label{lm:prev-result}
	For any $\delta'$-correct algorithm $\algp$ for \sign\ with $\delta' \le 0.01$, 
	there exist constants $N_0 \in \N$ and $c_1 > 0$ such that for all $N \ge N_0$:
	\[
	|\{ T_{\algp}(\Delta) < c_1 \cdot \Delta^{-2} \ln N \betw \Delta = 2^{-i}, i \in [0,N] \}|  \le 0.1 (N+1).
	\]
\end{lemma}

\eat{
\begin{proof}
	Let $\granuf(t) = t \cdot \ln 2$.
	So $e^{-\granuf(t)} = 2^{-t}$. 
	Clearly, $\granuf$ is a convex granularity function. 
	Note that $\delta' \le 0.01$.
	So by Theorem~\ref{thm:DENSITYSIGN-DIST-CONVEX}, there exists $c_1$ such that:
	\[
	\lim_{N \to +\infty} \discdistr_{N}(\granuf)(\{ \Delta \betw T_{\alg}(\Delta) < 2c_1 \cdot \Delta^{-2}\ln\ln\Delta^{-1} \}) = 0.
	\]
	This implies that there exists $N'\in \N$, such that for all $N \ge N'$:
	\[
	\discdistr_{N}(\granuf)(\{ \Delta \betw T_{\alg}(\Delta) < 2c_1 \cdot \Delta^{-2}\ln\ln\Delta^{-1} \}) \le 0.05.
	\]
	Therefore, we can pick $N_0 \ge \max(N',1000)$, such that for all $N \ge N_0$:
	\begin{align*}
	&\discdistr_{N}(\granuf)( \{ \Delta \betw T_{\alg}(\Delta) < c_1 \Delta^{-2} \ln N \} ) \\
	\le&\,\discdistr_{N}(\granuf)( \{ \Delta \betw T_{\alg}(\Delta) < c_1 \Delta^{-2} \ln N \ \wedge\ \ln \Delta^{-1} \ge \sqrt{N}  \} ) +\discdistr_{N}(\granuf)( \{ \Delta \betw \ln \Delta^{-1} < \sqrt{N}  \} ) \\
	\le&\,\discdistr_{N}(\granuf)( \{ \Delta \betw T_{\alg}(\Delta) < 2c_1 \Delta^{-2}\ln\ln\Delta^{-1} \} ) +\discdistr_{N}(\granuf)( \{ \Delta \betw \ln \Delta^{-1} < \sqrt{N}  \} ) \\
	\le&\, 0.05+ 0.05 \le 0.1
	\end{align*}	
	In fact, the second inequality follows from $2c_1 \Delta^{-2}\ln\ln\Delta^{-1} \ge 2 c_1 \cdot \Delta^{-2} \ln \sqrt{N} = c_1 \cdot \Delta^{-2} \ln N$. 
	The third inequality follows from the fact that $\Delta=e^{-\granuf(t)}$ and $\ln \Delta^{-1} < \sqrt{N}$
	corresponds to $\granuf(t) = t \cdot \ln 2 < \sqrt{N}$.
\end{proof}
}

For an algorithm $\alg$ for \bestarm, let $\palg$ be the algorithm which first randomly permutes the input arms, then runs $\alg$. 
More precisely, given an arm instance $I$ with $n$ arms, $\palg$ first chooses a random permutation $\pi$ on $n$ elements in $I$ uniformly, 
then simulates $\alg$ on the instance $\pi \circ I$
and returns what $\alg$ returns.
It is not difficult to see that the running time of $\palg$ only depends on the set $\{\distr_i\}$ of reward distributions  of the instance, 
not	their particular order.

Clearly, if $\alg$ is a $\delta$-correct algorithm for any instance of \bestarm, 
so is $\palg$. 
Furthermore, we have the following simple lemma, which says we only need to prove a lower bound for $\palg$.

\begin{lemma}\label{lm:random-permute}
	For any instance $I$, there exists a permutation $\pi$ such that $T_{\alg}(\pi \circ I) \ge T_{\palg}(I)$.
\end{lemma}
\begin{proof}
	By the definition of $\palg$, we can see that
	$$
	T_{\palg}(I) = \frac{1}{n!} \sum_{\pi \in \mathrm{Sym}(n)}{T_{\alg}(\pi \circ I)},
	$$ 
	where $\mathrm{Sym}(n)$ is the set of all $n!$ permutations of $\{1,\ldots, n\}$.
\end{proof}

Now we prove theorem~\ref{thm:hard-case-exists}. 
The high-level idea is to construct some ``balanced'' instances for \bestarm, 
and show that if an algorithm $\alg$ is ``fast'' on those instances, 
we can construct a fast algorithm for \sign, 
which leads to a contradiction to Lemma~\ref{lm:prev-result}.

\begin{proofof}{Theorem~\ref{thm:hard-case-exists}}	
	In this proof, we assume all distributions are Gaussian random variables with $\sigma = 1$. 
	Without loss of generality, we can assume $N_0$ in Lemma~\ref{lm:prev-result} is an even integer, 
	and $N_0 > 10$. So we have $2\cdot 4^{N_0} \ge \frac{4}{3} \cdot 4^{N_0} + N_0 + 2$. Let $N=2\cdot 4^{N_0}$.
	
	For every $n \ge N$, we pick the largest even integer $\nbag$ such that $2\cdot 4^{\nbag} \le n$. Clearly $\nbag \ge N_0 > 10$ and $\sum_{k=0}^{\nbag} 4^{k} + \nbag + 2 \le \frac{4}{3}\cdot 4^{\nbag} + \nbag + 2 \le 2 \cdot 4^{\nbag}$. Also, by the choice of $\nbag$, we have $2 \cdot 4^{\nbag+2} > n$,  hence $4^{\nbag} > \frac{n}{8}$.
	
	Consider the following \bestarm\ instance $\Iinit$ with $n$ arms:
	\begin{enumerate}[leftmargin=*]
		\item There is a single arm with mean $\xi$.
		
		\item For each integer $k \in [0,\nbag]$, there are $4^{\nbag-k}$ arms with mean $\xi - 2^{-k}$.
		
		\item There are $n - \sum_{k=0}^{\nbag} 4^k - 1$ arms with mean $\xi - 2$.
	\end{enumerate}

	For a \bestarm\ instance $I$, let $n(I)$ be the number of arms in $I$, and $\Gap{i}(I)$ be the gap $\Gap{i}$ according to $I$. 
	We denote $\armcomp(I) = \sum_{i=2}^{n(I)} \Gap{i}(I)^{-2}$. 
	
	Now we define a class of \bestarm\ instances $\{I_{S}\}$
	where each $S \subseteq \{0,1,\ldots, \nbag\}$.
	Each $I_{S}$ is formed as follows:
	for every $k \in S$, we add one more arm with mean $\xi - 2^{-k}$ to $\Iinit$; 
	finally we remove $|S|$ arms with mean $\xi - 2$
	(by our choice of $\nbag$ there are enough such arms to remove). 
	Obviously, there are still $n$ arms in every instance $I_{S}$.
	
	Let $c$ be a universal constant to be specified later (in particular $c$ does not dependent on $n$). 
	Now we claim that for any $\delta$-correct algorithm $\alg$ for \bestarm,
	there must exist an instance $I_{S}$ such that 
	$$
	T_\palg(I_S) > c \cdot \armcomp(I_S) \cdot \ln \nbag =\Omega(\armcomp(I_S) \ln\ln n).
	$$
	
	Suppose for contradiction that there exists a $\delta$-correct $\alg$ such that 
	$T_{\palg}(I_S) \le c \cdot \armcomp(I_S) \cdot \ln \nbag$ for all $S$.
	
	Let $U= \{ I_{S} \betw |S| = \nbag/2 \}$, $V = \{ I_S \betw |S| = \nbag/2 + 1 \}$ be two sets of \bestarm\ instances. 
	Notice that $|U| = |V| = \binom{\nbag+1}{\nbag/2}$ (since $\nbag$ is even). 
	
	Fix $S \in U$. 
	Consider the problem \sign,
	in which the given 
	instance is a single arm $A$ with unknown mean $\mu$, and we would like to decide whether 
	$\mu>\xi$ or $\mu<\xi$.
	Consider the following two algorithms
	for \sign, which call $\palg$ as a subprocedure.

	\begin{enumerate}[leftmargin=*]
		\item $\algone{S}$: 
		We first create a \bestarm\ instance 
		instance $\newI$ by
		replacing one arm with mean $\xi-2$ in $I_S$ with $A$. Then run $\palg$ on $\newI$. 
		We output $\mu > \xi$ if $\palg$ selects $A$ as the best arm. Otherwise, we output $\mu < \xi$.
		\item $\algtwo{S}$: We first construct an artificial arm $\newA$ with mean $2\xi - \mu$ from $A$
		\footnote{That is, whenever the algorithm pulls $\newA$, we pull $A$ to get a reward $r$, and return $2\xi-r$ as the reward for $\newA$. 
		Note although we do not know $\mu$, $\newA$ is clearly an arm with mean $2\xi - \mu$.
		} 
		, and create a \bestarm\ instance $\newI$
		by replacing one arm with mean $\xi-2$ in $I_S$ with $\newA$. Then run $\palg$ on $\newI$.
		We output $\mu < \xi$ if $\palg$ selects $\newA$ as the best arm. Otherwise,
		we output $\mu > \xi$.
	\end{enumerate}
	Since $\palg$ is $\delta$-correct for \bestarm, $\algone{S}$ and $\algtwo{S}$ are both $\delta$-correct for \sign. 
	
	Now, consider the algorithm $\alg_S$ for \sign\ which runs as follows: It simulates $\algone{S}$ and $\algtwo{S}$ simultaneously. 
	Each time it takes a sample from the input arm, and feeds it to both $\algone{S}$ and $\algtwo{S}$. If $\algone{S}$ ($\algtwo{S}$ resp.) terminates first, it returns the output of $\algone{S}$ ($\algtwo{S}$ resp.).
	In case of a tie, it
	returns the output of $\algone{S}$.
	
	First, we can see that 
	if both $\algone{S}$ and $\algtwo{S}$ are correct, then $\alg_S$ must be correct. Therefore, $\alg_S$ is $2\delta$-correct for \sign.
	
	Then we are going to show that there exists some particular $S$ such that the algorithm $\alg_S$ runs too fast for way too many points in $\{\Delta = 2^{-i}\}_{i \in [0,\nbag]}$ for \sign\, hence rendering a contradiction to Lemma~\ref{lm:prev-result}.
	
	For a \bestarm\ instance $I_S$ and an integer $k \in [0,m]$, 
	we use $N_S^k$ to denote the number of arms with gap $2^{-k}$. 
	Let $a_S^k \cdot 4^k$ be the expected number of samples taken from an arm with gap $2^{-k}$ by $\palg$.  
	Then we have that	
	\[
	\sum_{k=0}^{\nbag} N_S^k (4^k \cdot a_S^k) \le T_\palg(I_S) \le c \cdot \armcomp(I_S) \cdot \ln \nbag.
	\]
	Since $4^{\nbag-k} \le N_S^k \le 4^{\nbag-k} + 1$, we can see that 
	\[
	\armcomp(I_S) = \sum_{k=0}^{\nbag} N_S^k \cdot 4^k + (n-\sum_{k=0}^{\nbag} 4^k - 1 - |S|) \cdot 2^{-2} \le 2\sum_{k=0}^{\nbag} 4^{\nbag} + \frac{1}{4}\cdot n .
	\]	
	Thus we have that	
	\[
	\sum_{k=0}^{\nbag} 4^{\nbag-k} (4^k\cdot a_S^k) \le\sum_{k=0}^{\nbag} N_S^k (4^k\cdot a_S^k) \le c \cdot \armcomp(I_S)\cdot \ln \nbag \le c \left(2\sum_{k=0}^{\nbag} 4^{\nbag} + \frac{1}{4}n\right) \cdot \ln \nbag.
	\]
	Simplifying it a bit and noting that $\frac{n}{8} < 4^{\nbag}$, we get that
	\[
	\sum_{k=0}^{\nbag} 4^{\nbag} \cdot a_S^k \le c \cdot 2(\nbag+2) 4^{\nbag} \ln \nbag, 
	\]
	which is equivalent to
	\[
	\sum_{k=0}^{\nbag} a_S^k \le 2c \cdot (\nbag+2)\ln \nbag \le 3c \cdot (\nbag+1)\ln \nbag.
	\]
	The last inequality holds since $\nbag \ge N_0 > 10$.
	
	Now we set $c = \frac{c_1}{30}$, in which $c_1$ is the constant in Lemma~\ref{lm:prev-result}. 
	Since $\sum_{k=0}^{\nbag} a_S^k \le 3c \cdot (\nbag+1)\ln \nbag = \frac{c_1}{10} \cdot (\nbag+1)\cdot \ln \nbag$, we can see for any $S$, there are at most $0.1$ fraction of elements in $\{a_S^k\}_{k=0}^{\nbag}$ satisfying $a_S^k \ge c_1 \cdot \ln \nbag$.
	
	Then for $S \in U$, let $\bad_{S} = \{ k\not\in S \ \wedge\ a_{S \cup \{ k \}}^{k} \ge c_1 \ln \nbag \betw k \in [0,\nbag] \}$. 
	We have that
	\[
	\sum_{S \in U} | \bad_S | \le \sum_{S \in V} \sum_{k=0}^{\nbag} \indicator\{ a_{S}^k \ge c_1 \ln \nbag \} \le \frac{\nbag+1}{10} |V| = \frac{\nbag+1}{10}|U|.
	\]
	
	\newcommand{\goodS}{S}
	
	By an averaging argument, there exists $\goodS \in U$ such that $|\bad_{\goodS}| \le \frac{\nbag+1}{10}$. We will show for that particular $S$, the algorithm $\alg_S$ for \sign\ contradicts Lemma~\ref{lm:prev-result}.
	
	Now we analyze the expected total number of
	samples taken by $\alg_S$ on arm $A$ with mean $\mu$ and gap $\Delta=|\xi-\mu|=2^{-k}$. 
	Suppose $k \notin \goodS$. A key observation is the following:
	if $\mu < \xi$, then the instance constructed in $\algone{\goodS}$ is exactly $I_{\goodS \cup \{k\}}$; otherwise $\mu > \xi$, since $2\xi - (\xi+\Delta) = \xi - \Delta = \xi - 2^{-k}$, the instance constructed in $\algtwo{\goodS}$ is exactly $I_{\goodS \cup \{k\}}$
	(the order of arms in the constructed instance and $I_{\goodS \cup \{k\}}$ may differ, but as $\palg$ randomly permutes the arms beforehand, it does not matter). 
	Hence, either $T_{\algone{\goodS}}(A)=a_{\goodS \cup \{k\}}^k \cdot 4^k$ or $T_{\algtwo{\goodS}}(A)=a_{\goodS \cup \{k\}}^k \cdot 4^k$. Since $\alg_{\goodS}$ terminates as soon as either one of them terminate, we clearly have $T_{\alg_{\goodS}}(A) \le \min(T_{\algone{\goodS}}(A),T_{\algtwo{\goodS}}(A))\le a_{\goodS \cup \{k\}}^k \cdot 4^k$ for arm $A$ with gap $2^{-k}$ when $k \not\in \goodS$.
	
	So for all $k \in [0,\nbag] \setminus (\goodS \cup \bad_{\goodS})$,  we can see 
	$T_{\alg_{\goodS}}(2^{-k})  \le a_{\goodS \cap \{ k \}}^{k} \cdot 4^k < c_1 \cdot 4^k \ln \nbag$. But this implies that	
	\[
	\frac{ \{ T_{\alg_{\goodS}}(\Delta) < c_1 \cdot \Delta^{-2} \ln \nbag \betw \Delta = 2^{-i}, i \in [0,\nbag] \} }{\nbag+1} 
	\ge \frac{|[0,\nbag] \setminus (\goodS \cup \bad_{\goodS})|  }{\nbag+1} \ge 0.4,
	\]
	which contradicts Lemma~\ref{lm:prev-result}. 
	So there must exist $I_S$ such that $T_{\palg}(I_{S}) > c\cdot \armcomp(I_{S}) \cdot \ln \nbag$.
	
	\eat{
		Since $n \ge 4^{\nbag} > 4^{10}$ we can see $4^{\nbag} > \frac{n}{8} \ge \sqrt{n}$. Then $\ln 4 \cdot \nbag > \frac{1}{2} \ln n$, taking logarithm of each side, we have $\ln\ln 4 + \ln \nbag > \ln\ln n - \ln 2$. Which means $\ln \nbag > \frac{1}{2} \ln\ln n$.
		
		So we have:
		\[
		T_{\alg_p}(I_S) \ge \frac{c}{2} \sum_{i=2}^{n} \Delta_{i}^{-2} \ln\ln n.
		\]
	}
	By Lemma~\ref{lm:random-permute}, there exists a permutation $\pi$ on $I_{S}$ such that 
	$T_{\alg}(\pi \circ I_{S}) \ge \frac{c}{2} \sum_{i=2}^{n} \Delta_{i}^{-2} \ln\ln n $.
	This finishes the first part of the theorem.
	
	To see that $\Delta_2^{-2} \ln\ln \Delta_2^{-1}$ is not the dominating term,
	simply notice that 
	$$
	\Delta_2^{-2} \ln\ln \Delta_2^{-1}  = 4^{\nbag} \ln (\nbag\cdot \ln 2) \le  4^{\nbag} \ln \nbag \le \frac{1}{\nbag}\sum_{k=0}^{\nbag} N_{\goodS}^k \cdot 4^k \ln \nbag \le \frac{2\cdot \ln 4}{\ln n} \sum_{i=2}^{n} \Delta_{i}^{-2} \ln\ln n.
	$$
	This proves the second statement of the theorem.
\end{proofof}


\section{A New Lower Bound for \sign}
\label{sec:lbsign}

\newcommand{\powconst}{\gamma}

In this section, we prove
a new lower bound of \sign\ (Theorem~\ref{lm:FRACTION-BOUND-DISC}), from which
Lemma~\ref{lm:prev-result} follows easily.

We introduce some notations first. Recall that the distributions of all the arms are Gaussian with variance $1$. Fix an algorithm $\alg$ for \sign, for a random event $\event$, let
$\Pr_{\alg,A_\mu}[\event]$ (recall that $A_{\mu}$ denotes an arm with mean $\mu$) denote the probability that $\event$ happens if we run $\alg$ on arm $A_\mu$.
For notational simplicity, when $\alg$ is clear from the context, 
we abbreviate it as $\Pr_\mu[\event]$.
Similarly, we write $\Ex_\mu[X]$ as a short hand notation for $\Ex_{\alg,A_\mu}[X]$, which denotes the expectation of random variable $X$ when running $\alg$ on arm $A_{\mu}$.

We use $\indicator\{\text{expr}\}$ to denote the indicator function 
which equals 1 when $\text{expr}$ is true and 0 otherwise, and we define
$\lowb(\Delta) = \Delta^{-2} \cdot \ln\ln \Delta^{-1}$, which is the wanted lower bound for \sign.
Finally, for an integer $N \in \mathbb{N}$,
a $\delta$-correct algorithm $\alg$ for \sign,
and a function $g : \R \to \R$, 
define
$$
\distportion(\alg,g,N) = \sum_{i=1}^{N} \indicator\left\{ \text{There exists some } \Delta \in [e^{-i},e^{-i+1}) \text{ such that: } T_{\alg}(\Delta) < g(\Delta) \right\}.
$$ 
Intuitively, it is the number of intervals $[e^{-i},e^{-i+1})$ among the first $N$ intervals that contains a fast point with respect to $g$.

\begin{theo}\label{lm:FRACTION-BOUND-DISC}
	For any $\powconst>0$, we have a constant $c_1>0$ (which depends on $\powconst$) such that for any $0 < \delta <0.01$,
	\[
	\lim_{N \to +\infty} \sup_{\alg' \text{ is } \delta\text{-correct} } \frac{\distportion(\alg',c_1 \lowb, N)}{N^\powconst} = 0.
	\]
	In other words, the fraction of the intervals containing fast points with respect to $\Omega(\lowb)$ can be smaller than any inverse polynomial. 
\end{theo}

Before proving Lemma~\ref{lm:FRACTION-BOUND-DISC}, we show it implies Lemma~\ref{lm:prev-result} as desired. We restate Lemma~\ref{lm:prev-result} here for convenience.

\vspace{0.2cm}
\noindent
{\bf Lemma~\ref{lm:prev-result}	}
(restated)
{\em	
For any $\delta'$-correct algorithm $\algp$ for \sign\ with $\delta' \le 0.01$, 
there exist constants $N_0 \in \N$ and $c_1 > 0$ such that for all $N \ge N_0$:
\[
 |\{ T_{\algp}(\Delta) < c_1 \cdot \Delta^{-2} \ln N \betw \Delta = 2^{-i}, i \in [0,N] \}|  \le 0.1 (N+1).
\]
}

\begin{proofof}{Lemma~\ref{lm:prev-result}}
Let $\gamma = 1/2$. Applying Lemma~\ref{lm:FRACTION-BOUND-DISC}, we can see that
there exist an integer $M_0$ and a constant $c_2$ such that, for any integer $M \ge M_0$, it holds that
$$
\distportion(\algp,c_2 F,M) \le 0.05 \cdot \sqrt{M},
$$
for any $\delta'$-correct algorithm $\algp$ for \sign.

Then for any $N \ge M_0/\ln 2$, we can see that
$$
 |\{ T_{\algp}(\Delta) < c_2 \cdot F(\Delta) \betw \Delta = 2^{-i}, i \in [0,N] \}| \le \distportion(\algp,c_2\cdot F,\lceil \ln 2 \cdot N \rceil) \cdot 2 \le 0.1 \cdot \sqrt{N},
$$
since $e^{\lceil \ln 2 \cdot N\rceil} \ge 2^{N}$, and each interval $[e^{-i},e^{-i+1})$ can contain at most $2$ values of the form $\Delta = 2^{-k}$.

For $\Delta = 2^{-i}$ ($i\ge \sqrt{N}$), we have that
$F(\Delta) = \Delta^{-2} \ln\ln \Delta^{-1} = \Delta^{-2} \ln i \ge \Delta^{-2} \ln N/2$. Therefore, by letting $c_1 = c_2/2$, 
we can bound the cardinality of the set as follows
\begin{align*}
&|\{ T_{\algp}(\Delta) < c_1 \cdot \Delta^{-2} \ln N \betw \Delta = 2^{-i}, i \in [0,N] \}|\\
\le& \sqrt{N} + |\{ T_{\algp}(\Delta) < c_1 \cdot \Delta^{-2} \ln N \betw \Delta = 2^{-i}, i \in [\sqrt{N},N] \}|\\
\le& \sqrt{N} + |\{ T_{\algp}(\Delta) < c_2 \cdot F(\Delta) \betw \Delta = 2^{-i}, i \in [\sqrt{N},N] \}|\\
\le& 1.1 \cdot \sqrt{N}.
\end{align*}
This completes the proof of the lemma.
\end{proofof}

\subsection{Proof for Theorem~\ref{lm:FRACTION-BOUND-DISC}}

The rest of this section is devoted to prove Theorem~\ref{lm:FRACTION-BOUND-DISC}. 
From now on, $0<\delta<0.01$ is a fixed constant.
We first show a simple but convenient lemma based on Lemma~\ref{lm:CHANGEDIST}.

\begin{lemma}\label{lm:BIGHURTSMALL}
	Let $I_1$ (with reward distribution $\distr_1$) and $I_2$ (with reward distribution $\distr_2$)
	be two instances of \sign.  
	Let $\event$ be a random event and 
	$\tau$ be the total number of samples taken by $\alg$.
	Suppose $\Pr_{\distr_1}[\event] \ge \frac{1}{2}$.
	Then, we have	
	\[
	\Pr_{\distr_2}[\event] \ge \frac{1}{4}\exp\bigl(-2\Ex_{\distr_1}[\tau]\KL(\distr_1,\distr_2)\bigr).
	\]
\end{lemma}

\begin{proof}
	Applying Lemma~\ref{lm:CHANGEDIST}, we have that
	\[
	\Ex_{\distr_1}[\tau] \cdot \KL(\distr_1,\distr_2) \ge \ent(\Pr_{\distr_1}[\event],\Pr_{\distr_2}[\event]) 
	\ge 
	\ent\left(\frac{1}{2},\Pr_{\distr_2}[\event]\right) 
	\ge \frac{1}{2} \cdot \ln\left(\frac{1}{4\Pr_{\distr_2}[\event](1-\Pr_{\distr_2}[\event])}\right).
	\]
	Hence, we can see that 
	$
	4\Pr_{\distr_2}[\event](1-\Pr_{\distr_2}[\event]) \ge \exp(-2\Ex_{\distr_1}[\tau]\cdot \KL(\distr_1,\distr_2)),
	$
	from which the lemma follows easily.
\end{proof}

From now on, suppose $\alg$ is a $\delta$-correct algorithm for \sign.
We define two events:
$$
\event_U = [\alg \text{ outputs }``\mu>\xi"],
$$
$$
\event(\Delta) = \event_U \wedge [ d\Delta^{-2} \le \tau \le 5 T_{\alg}(\Delta) ],
$$
where 
$\tau$ is the number of samples taken by $\alg$ and
$d$ is a universal constant to be specified later. The following lemma is a key tool, which can be used to partition the event $\event_U$ into several disjoint parts. 

\begin{lemma}
	\label{lm:EVENTEXT}
	For any $\Delta > 0$ and $d < \ent(0.2,0.01)/2$, we have that
	$$
	\Pr_{\xi+\Delta}[\event(\Delta)]=\Pr_{\alg, A_{\xi+\Delta}}[\event(\Delta)] \ge \frac{1}{2}.$$
\end{lemma} 

\begin{proof} 	
	First, we can see that $\Pr_{\xi+\Delta}[\event_U] \ge 1-\delta \ge 0.99$
	since $\alg$ is $\delta$-correct and ``$\mu>\xi$" is the right answer.
	
	Now, we claim $\Pr_{\xi+\Delta}[\tau < d\Delta^{-2}] < 0.25$.
	Suppose to the contrary that
	$\Pr_{\xi+\Delta}[\tau < d\Delta^{-2}] \ge 0.25$. 
	We can see that $\Pr_{\xi+\Delta}[\event_U \wedge \tau < d\Delta^{-2}] \ge 0.25 - \delta  > 0.2$.
	
	Consider the following algorithm $\alg'$:
	$\alg'$ simulates $\alg$ for $d \Delta^{-2}$ steps. 
	If $\alg$ halts, $\alg'$ outputs what $\alg$ outputs, 
	otherwise $\alg'$ outputs nothing. 
	
	Let $\event_V$ be the event that $\alg'$ outputs $\mu > \xi$.
	Clearly, we have $\Pr_{\alg',\xi+\Delta}[\event_V] > 0.2$. 
	On the other hand, $\Pr_{\alg',\xi-\Delta}[\event_V] < \delta$,
	since $\alg$ is a $\delta$-correct algorithm. 
	So by Lemma~\ref{lm:CHANGEDIST}, we have that
	\[
	\Ex_{\alg',\xi+\Delta}[\tau] \KL(N(\xi+\Delta,\sigma),N(\xi-\Delta,\sigma)) = 
	\Ex_{A',\xi+\Delta}[\tau] 2\Delta^{2} \ge \ent(0.2,\delta) \ge \ent(0.2,0.01).
	\]

	Since $d \Delta^{-2} \ge \Ex_{\alg',\xi+\Delta}[\tau]$, we have $d \ge \ent(0.2,0.01)/2$.
	But this contradicts the condition of the lemma. 
	Hence, we must have $\Pr_{\xi+\Delta}[\tau < d\Delta^{-2}] < 0.25$.
	
	Finally, we can see that
	\begin{align*}
	\Pr_{\xi+\Delta}[\event(\Delta)] &\ge \Pr_{\xi+\Delta}[\event_U] - \Pr_{\xi+\Delta}[\tau < d\Delta^{-2}] - \Pr_{\xi+\Delta}[\tau > 5 T_\alg(\Delta)] \\
	&\ge 1-0.01-0.25 - \Pr_{\xi+\Delta}[\tau > 5 \Ex_{\alg,\xi+\Delta}[\tau]]\\
	&\ge 1-0.01-0.25-0.2 \ge 0.5
	\end{align*}
	where the first inequality follows from the union bound and the second from Markov inequality.
\end{proof}


\begin{lemma}
	\label{lm:sequence-bound}
	For any $\delta$-correct algorithm $\alg$, and any finite sequence $\{\Delta_i\}_{i=1}^{n}$ 
	such that 
	\begin{enumerate}
		\item the events $\{\event(\Delta_i)\}$ are disjoint, 
		\footnote{
			More concretely, the intervals $[d\Delta_i^{-2}, 5T_{\alg}(\Delta_i)]$ are disjoint.
		}
		and $0 < \Delta_{i+1} < \Delta_i$ for all $1 \le i \le n-1$;
		\item there exists a constant $c>0$ such that $T_{\alg}(\Delta_i)\leq c \cdot \lowb(\Delta_i) $ for all $1\le i \le n$,
	\end{enumerate}
	it must hold that:
	\[
	\sum_{i=1}^{n} \exp\{-2c\cdot \lowb(\Delta_i)\cdot \Delta_i^2\} \le 4\delta.
	\]
\end{lemma}

\begin{proof}	
	Suppose for contradiction that 
	$\sum_{i=1}^{n} \exp\{-2c\cdot \lowb(\Delta_i)\cdot \Delta_i^2\} 
	> 4 \delta$.
	
	Let $\alpha = \frac{1}{5}\Delta_n$. 
	By Lemma~\ref{lm:NORMALKL}, 
	we can see that $\KL(\Normal(\xi+\Delta_i,\sigma),\Normal(\xi-\alpha,\sigma)) 
	= \frac{1}{2\sigma^2} (\Delta_i+\alpha)^2 \le \frac{1}{2}(1.2\Delta_i)^2 \le  \Delta_i^2$.
	By Lemma~\ref{lm:BIGHURTSMALL}, we have:
	\[
	\Pr_{\xi-\alpha}[\event_U] \ge \sum_{i=1}^{n} \Pr_{\xi-\alpha}[\event(\Delta_i)] 
	\ge \frac{1}{4} \sum_{i=1}^{n} \exp\{-2 \Ex_{\xi+\Delta_i}[\tau]\, \Delta_i^2\}
	\ge \frac{1}{4} \sum_{i=1}^{n} \exp\{-2 c\cdot \lowb(\Delta_i)\cdot \Delta_i^2\}>\delta.
	\]
	Note that we need Lemma~\ref{lm:EVENTEXT}(1) (i.e., $\Pr_{\xi+\Delta_i}[\event(\Delta_i)]\geq 1/2$ ) 
	in order to apply Lemma~\ref{lm:BIGHURTSMALL} for the second inequality.
	The above inequality means that    
	$\alg$ outputs a wrong answer for instance $\xi-\alpha$ with probability $>\delta$,
	which contradicts that $\alg$ is $\delta$-correct.
\end{proof} 



Now, we try to utilize Lemma~\ref{lm:sequence-bound} on a carefully constructed sequence $\{\Delta_i\}$.
The construction of the sequence $\{\Delta_i\}$ requires quite a bit calculation.
To facilitate the calculation, we provide a sufficient condition for the disjointness of the sequence, as follows.

\begin{lemma}\label{lm:DISJOINT-EVENTS}
	$\alg$ is any $\delta$-correct algorithm for \sign.
	$c>0$ is a universal constant.
	The sequence $\{\Delta_i\}_{i=0}^N$ satisfies the following properties: 
	\begin{enumerate}
		\item $1/e>\Delta_1>\Delta_2>\dotsc>\Delta_N\ge\alpha>0$.
		\item For all $i\in [N]$, we have that $T_{\alg}(\Delta_i) \le c\cdot \lowb(\Delta_i)$.
		\item Let $L_i = \ln \Delta_i^{-1}$.
		We have $L_{i+1} - L_{i} > \frac{1}{2}\ln\ln\ln\alpha^{-1} + c_1$, in which $c_1 = \frac{\ln c + \ln 5 - \ln d}{2}$.
	\end{enumerate}
	Then, the events $\{\event(\Delta_1), \event(\Delta_2), \ldots, \event(\Delta_N)\}$  are disjoint. 
\end{lemma}

\begin{proof}
	We only need to show the intervals for each $\event(\Delta_i)$ are disjoint. 
	In fact, it suffices to show it holds for two adjacent events $\event(\Delta_i)$ and $\event(\Delta_{i+1})$. 
	Since $5T_{\alg}(\Delta_i) \le 5c\cdot\lowb(\Delta_i)$, we only need to show	
	$
	5c\cdot\lowb(\Delta_i) < d \Delta_{i+1}^{-2},
	$
	which is equivalent to
	\[
	\ln c + \ln 5 + 2 L_{i} + \ln\ln L_{i} < \ln d + 2 L_{i+1}.
	\]
	By simple manipulation, this is further equivalent to 
	$L_{i+1} - L_{i} > (\ln c + \ln 5 - \ln d + \ln\ln L_{i})/2$.	
	Since $\alpha \le \Delta_i$, we have $\ln\ln\ln \alpha^{-1} \ge \ln\ln L_i$, which concludes the proof. 
\end{proof}

Now, everything is ready to prove Theorem~\ref{lm:FRACTION-BOUND-DISC}.

\begin{proofof}{Theorem~\ref{lm:FRACTION-BOUND-DISC}}
	Suppose for contradiction, for any $c_1 > 0$, the limit is not zero.
	This is equivalent to
	\begin{align}
	\label{eq:limsup1}
	\limsup_{N \to +\infty} \sup_{\alg' \text{ is } \delta-\text{correct} } \frac{\distportion(\alg',c_1 \lowb, N)}{N^\powconst} > 0.
	\end{align}
	We claim that for $c_1 = \frac{\powconst}{4}$, the above can lead to a contradiction.
	
	
	
	First, we can see that \eqref{eq:limsup1} is equivalent to 
	the existence of an infinite increasing sequence $\{N_i\}_i$ and a positive number $\beta>0$
	such that
	\[
	\sup_{\alg' \text{ is } \delta-\text{correct} } \frac{\distportion(\alg',c_1 \lowb, N_i)}{N_i^\powconst} > \beta, \quad \text{ for all }i.
	\]
	
	Consider some large enough $N_i$ in the above sequence.
	The above formula implies that there exists a $\delta$-correct algorithm $\alg$ 
	such that $\distportion(\alg,c_1\lowb, N_i) \ge \beta N_i^{\powconst}$.
	
	We maintain a set $S$, which is initially empty.
	For each $2\le j\le N_i$, if there exists $\Delta \in [e^{-j},e^{-j+1})$ such that
	$T_{\alg}(\Delta) \le c_1 F(\Delta)$, 
	then we add one such $\Delta$ into the set $S$. 
	We have $|S| \ge \distportion(\alg,c_1\lowb,N_i) -1 \ge \beta N_i^\powconst - 1$
	($-1$ comes from that $j$ starts from 2).
	Let 
	$$
	b = \left\lceil \frac{\ln c_1 + \ln 5 - \ln d +\ln\ln N_i}{2} + 1 \right\rceil.
	$$ 
	We keep only the $1$st, $(1+b)$th, $(1+2b)$th, $\dotsc$ elements in $S$, and remove the rest.
	With a slight abuse of notation, rename the elements in $S$ by
	$\{ \Delta_{i} \}_{i=1}^{|S|}$, sorted in decreasing order. 
	
	It is not difficult to see that $\frac{1}{e}>\Delta_1 >\Delta_2>\dotsc>\Delta_{|S|} \ge e^{-N_i}>0$.
	By the way we choose the elements, for $1 \le i < |S|$, 
	we have 
	$$
	\ln \Delta_{i+1}^{-1} - \ln\Delta_{i}^{-1} > \frac{\ln c_1 + \ln 5 - \ln d}{2} +  \frac{1}{2}\ln\ln N.
	$$ 
	Recall that we also have $T_{\alg}(\Delta_i) \le c_1 \lowb(\Delta_i)$ for all $i$. 
	Hence, we can apply Lemma~\ref{lm:DISJOINT-EVENTS} and conclude that all events $\{\event(\Delta_i)\}$ are disjoint.
	
	We have $|S|\ge (\beta N_i^\powconst-1)/b$, for large enough $N_i$ (we can choose such $N_i$ since $\{N_i\}$ approaches to
	infinity), it implies $|S| \ge \beta N_i^\powconst/\ln\ln N_i$. Then, we can get
	\begin{align*}
	\sum_{j=1}^{|S|} \exp\{-2c_1\cdot \lowb(\Delta_j)\cdot\Delta_j^{-2}\} 
	=&\sum_{j=1}^{|S|} \exp\{-\powconst\cdot \ln\ln \Delta_j^{-1}/2\}\\
	=&\sum_{j=1}^{|S|} (\ln \Delta_j^{-1})^{-\powconst/2}
	\ge |S| \cdot N_i^{-\powconst/2}\\
	\ge& \beta N_i^\powconst/\ln\ln N_i \cdot N_i^{-\powconst/2} = \beta N_i^{\powconst/2}/\ln\ln N_i.
	\end{align*}
	The inequality in the second line holds since $\Delta_j\geq e^{-N_i}$ for all $j\in [|S|]$.
	Since $\powconst>0$, we can choose $N_i$ large enough such that $\beta N_i^{\powconst/2}/\ln\ln N_i > 4\delta$, which renders a contradiction to Lemma~\ref{lm:sequence-bound}.
\end{proofof}

\begin{rem}
	It would be transparent from our proof why the $\ln\ln \Delta^{-1}$ term is essential: 
	The main reason is that $\Delta$ is not known beforehand (if $\Delta$  is known, \sign\ can be solved in $O(\Delta^{-2}\ln \delta^{-1})$ time). 
	Intuitively, an algorithm has to ``guess'' and ``verify'' (in some sense) the true $\Delta$ value.
	As a result, if the algorithm is ``lucky'' in allocating the time for verifying the right guess of $\Delta$,  
	it may stop earlier and thus be faster than $\lowb(\Delta)$ for some $\Delta$s. 
	But as we will demonstrate, if an algorithm stops earlier on 
	larger $\Delta$, it would hurt the accuracy for smaller $\Delta$, and there is no way to be always lucky. 
	This is the only factor accounting for the $\ln\ln \Delta^{-1}$. 
	While Farrell's proof attributes the $\ln\ln \Delta^{-1}$ factor to the Law of Iterative Logarithm (see also \cite{jamieson2014lil}), 
	our proof shows that the $\ln\ln \Delta^{-1}$ factor exists due to algorithmic reasons, 
	which is a new perspective.
\end{rem}
		

\newcommand{\ordlow} {\mathcal{L}}
\newcommand{\arment} {\mathsf{Ent}}

\section{On Almost Instance Optimality}
\label{sec:instanceopt}

Instance optimality (\cite{fagin2003optimal,afshani2009instance}) is the strongest possible notion of optimality in the theoretical computer science literature. Loosely speaking,
an algorithm $\alg$ is instance optimal if the running time 
of $\alg$ on instance $I$ is at most $O(L(I))$, where $L(I)$
is the lower bound required to solve the instance for any 
algorithm.

Let us first consider \sign\, (the two arms case).
As we mentioned, Farrell's lower bound \eqref{eq:2armlowerbound}
is not an instance-wise lower bound.	
On the other hand, 
it is impossible to obtain an $\Omega(\Delta^{-2} \ln\ln \Delta^{-2})$
lower bound for every instance, since we can design an algorithm 
that uses $o(\Delta^{-2} \ln\ln \Delta^{-2})$ samples for infinite number of 
instances. We provide a detailed discussion in Section~\ref{app:sign}. 
Combining this two fact, we can see
that it is impossible to obtain an instance optimal algorithm even for \sign.
Hence, it appears to be more hopeless to consider instance optimality for \bestarm.
However, based on our current understanding, 
we suspect that the two arms case is the only obstruction for an instance optimal 
algorithm, and modulo a $\Gap{2}^{-2}\ln\ln \Gap{2}$ additive term, 
we may be able to achieve instance optimality for \bestarm.

We propose an intriguing conjecture
concerning the instance optimality of \bestarm.
The conjecture provides an explicit formula for 
the sample complexity.
Interestingly, the formula involves 
an entropy term, which we call {\em the gap entropy},
which has not appeared in the bandit literature, 
to the best of our knowledge.
We assume that all reward distributions are Gaussian with variance 1.
In order to state the conjecture
formally, we need to define what is 
an instance-wise lower bound.
Our definition is inspired by that in \cite{afshani2009instance}.

\begin{defi}(Order-Oblivious Instance-wise Lower Bound)
	Suppose $\ordlow(I,\delta)$
	is a function which  
	maps a \bestarm\ instance $I$ with $n$ arms, and confidence parameter $\delta$ to a number.
	We say  $\ordlow(I,\delta)$ is 
	an instance-wise lower bound for $I$ if
	$$
	\ordlow(I,\delta) \le \inf_{\alg: \alg \text{ is } 
	\delta\text{-correct}} \,\,\frac{1}{n!} \cdot \sum_{\pi \in \mathrm{Sym}(n)} T_{\alg}(\pi \circ I).
	$$
\end{defi}

Now, we define the entropy term.

\begin{defi}\label{defi:arment} (Gap Entropy)
	Given a \bestarm\ instance $I$, 
	let 
	$$
	G_i = \{u \in [2,n] \mid  2^{-i} \le \Gap{u} < 2^{-i+1} \},\quad
	H_i = \sum_{u \in G_i} \Gap{u}^{-2},
	\quad\text{ and }\quad
	p_i = H_i/\sum_j H_j.
	$$
	 We can view $\{p_i\}$
	 as a discrete probability distribution.
	We define the following quantity as the {\em gap entropy} for the instance $I$
	$$
	\arment(I) = \sum_{G_i \ne \emptyset} p_i \log p_i^{-1}.
	$$
	Note that it is exactly the Shannon 
	entropy for the distribution defined by $\{p_i\}$.
\end{defi}

\begin{rem}
	We choose to partition the arms based on the powers of $2$. 
	There is nothing special about 2 and replacing it
	by any other constant only changes $\arment(I)$
	by a constant factor.	
\end{rem}

Now, we formally state our conjecture.
Let $H(I)=\sum_{i=2}^{n} \Gap{i}^{-2}$.

\begin{conj}
	\label{conj:optimal}
	For any \bestarm\ instance $I$
	and confidence $\delta\in (0,c)$ ($c$
	is a universal small constant), 
	let $$
	\ordlow(I,\delta) = \Theta\left(H(I)(\ln\delta^{-1} + \arment(I))\right).
	$$
	$\ordlow(I,\delta)$ is an instance-wise lower bound for $I$.
		
	Moreover, there is a \CORRECT\ algorithm for \bestarm\ with sample complexity
	$$
	O\left(\ordlow(I,\delta) +  \Gap{2}^{-2} \ln\ln\Gap{2}^{-1}\right).
	$$ 
	In other words, modulo the 
	$\Gap{2}^{-2}\ln\ln\Gap{2}^{-1}$ additive term, 
	the algorithm is instance optimal.
	We call such an algorithm an {\em almost instance optimal}
	algorithm.
\end{conj}


The conjectured sample complexity
consists of two terms, one matching an instance-wise
lower bound $\ordlow(I, \delta)$ and the other matching the 
optimal bound $\Gap{2}^{-2} \ln\ln\Gap{2}^{-1}$ for \sign.
	\footnote{
		From our previous discussion, we know it is impossible to obtain an instance optimal 
		algorithm for 2-arm instances, and the bound $\Gap{2}^{-2} \ln\ln\Gap{2}^{-1}$ is not improvable.
	}
Hence, a resolution of the conjecture
would provide a complete understanding 
of the sample complexity of \bestarm.

In fact, our proofs
of Theorem~\ref{thm:hard-case-exists}
and Theorem~\ref{theo:BESTARMPACD}
provide strong evidences for
Conjecture~\ref{conj:optimal} and we briefly discuss
the connections below.

First, the third additive term 
$\sum_{i=2}^{n} \Gap{i}^{-2}\ln\ln\min(n,\Gap{i}^{-1})$
in Theorem~\ref{theo:BESTARMPACD}
might appear to be an artifact of the algorithm or 
the analysis at first glance.
However, in light of Conjecture~\ref{conj:optimal},
it is a natural upper bound of $H(I)\arment(I)$,
as shown in the following proposition.
On one extreme, the maximum value $\arment(I)$ can get is 
$O(\ln\ln n)$.
This can be achieved by instances in which there are $\log n$ nonempty groups $G_i$ and
they have almost the same weight $H_i$.
On the other extreme where
there is only a constant number of nonempty groups
(i.e., the instance is clustered),
$\arment(I)=O(1)$, and our algorithm can achieve almost instance optimality in this case.
The proof of the proposition is somewhat tedious and we defer it to the end of this section.

\begin{prop}\label{prop:arment-bound}
	For any instance $I$, 
	$
	H(I)\arment(I)$ is upper bounded by  $$O\left(\sum_{i=2}^{n} \Gap{i}^{-2}(1+\ln\ln\min(n,\Gap{i}^{-1})) \right).
	$$
	In particular, $\arment(I) = O(\ln\ln n)$. Moreover, for any clustered instance $I$, we have $\arment(I) = O(1)$. 
\end{prop}




\eat{
Our previous results already showed that, even in the case for only two arms, we can not hope for an algorithm with sample complexity $O(\ordlow(I,\delta))$. Hence, there must be a gap between the achievable sample complexity and the instance-wise lower bound. 
And our conjecture consists of two part, the first part is that this gap is only an {\em additive} term of $O(\Gap{i}\ln\ln\Gap{i}^{-2})$. The second part is the lower bound is given by a simple entropy expression $\arment(I)$ just like in \cite{afshani2009instance}.
}

Besides the fact that our algorithm can achieve
optimal results for both extreme cases,
we have more reasons to believe why 
$\arment(I)$ should enter the picture.

\topic{Gap Entropy $\arment(I)$}
First, we motivate $\arment$ from the algorithmic side.
Consider an elimination-based algorithms 
(such as \cite{karnin2013almost} or our algorithm).
We must ensure that the best arm is not eliminated 
in any round. 
Recall that in the $r$-th round, we want to eliminate arms with gap $\Delta_r = \Theta(2^{-r})$, which is done by obtaining an approximation of the best arm, then take $O(\Delta_r^{-2} \ln \delta_r^{-1})$ samples from each arm and eliminate the arms with smaller empirical means. 
Roughly speaking, we need to assign the failure probability $\delta_r$ carefully to each round
(by union bound, we need $\sum_r\delta_r\leq \delta$).
The algorithm in \cite{karnin2013almost}
use $\delta_r = O(\delta \cdot r^{-2})$. and our algorithm uses a better way to assign $\delta_r$, based on the information we collected using \FRACTIONTEST. 
However, if one can assign
$\delta_r$s optimally (i.e., minimize 
$\sum_{r} H_r \ln \delta_r^{-1}$ subject to $\sum_{r} \delta_r \le \delta$),
one could achieve the entropy bound 
$\sum_r H_r \cdot (\ln\delta^{-1} + \arment(I))$
(by letting $\delta_r = \delta H_r/\sum_i H_i $).
Of course, this does not lead to
an algorithm directly,
as we do not know $H_i$s in advance.



We also have strong evidence from our lower bound result. 
In fact, it is possible to extend 
Theorem~\ref{thm:hard-case-exists} in the following way. 
\footnote{We omit the details in this version.
}
We can use different $I_{init}$ (e.g. choosing a different number of arms with gap $2^{-k}$ for each $k$) and show there is a similar instance $I_S$ such that $\alg$ requires at least $\Omega( H(I_S) \cdot \arment(I_S))$ samples. 
Even this does not prove an lower bound for every instance, it strongly suggests $\Omega( H(I) \cdot \arment(I))$ is the right lower bound. 

Now, we provide a proof of Proposition~\ref{prop:arment-bound}.
\vspace{-0.6cm}
	\begin{proofof}{Proposition~\ref{prop:arment-bound}}
		\newcommand{\mk}{m}
		In the following the base of $\log$ is $2$.
		Let $\mk$ denote the maximum index $i$ such that $G_{i}$ is non-empty, and $S_i = |G_i|$. 
		Clearly, $4^{i-1} S_i \le H_i \le 4^i S_i$. 
		
		We first prove the second claim, which is straightforward. 
		By definition, in a clustered instance,  
		the number of non-empty $G_i$ is bounded by a constant $C$, hence the corresponding entropy is bounded by a constant.
		
		For the first claim, we make use of 
		the non-negativity of KL divergence .
		We construct another probability distribution $q$. Recall that
		$\KL(p,q) = \sum (p_i \log q_i^{-1} - p_i \log p_{i}^{-1}) \ge 0$, which implies that $\arment(I) \le \sum p_i \log q_i^{-1} $.
		
		Now, we partition all the arms into blocks. The $t$-th block $B_t$ is the union of a consecutive segment of $G_{l_t},G_{l_t+1},\dots,G_{r_t}$. 
		The blocks are constructed one by one starting from the first block $B_1$.
		The $t$-th block $B_t$ is constructed as follows: let $l_t = r_{t-1}+1$ (if $t=1$, then $l_t = 1$), and $r_t$ be the first $k$ such that $\sum_{i=l_t}^{k} S_i \ge \frac{1}{2} \cdot \sum_{i=l_t}^{\mk} S_i$. We terminate w
		hen $r_t = \mk$. 
		Suppose there are $h$ blocks in total.
		Clearly $h = O(\log n)$.
		
		Now, we define the probability distribution $q$. For each $i$ such that $G_i$ belongs to $B_t$, we let $q_i = \frac{6}{\pi^2} \cdot t^{-2}\cdot 2^{i-r_t-1}$. Note that $\sum_{i=1}^{\mk} q_i = \frac{6}{\pi^2} \cdot \sum_{t=1}^{h} t^{-2} \sum_{i=l_t}^{r_t} 2^{i-r_t-1} \le \frac{6}{\pi^2} \cdot \sum_{t=1}^{h} t^{-2} \le 1$. In addition, we let $q_{\mk + 1} = 1 - \sum_{i=1}^{\mk} q_i$. 
		Hence, $q$ is a well defined distribution over
		$[1, \mk + 1]$.
		
		\eat{
			Then we have (we adopt the convention $0\log 0 = 0$):
			
			\begin{align*}
			\KL(p,q) &\ge 0 \\
			\sum_{i=1}^{\mk+1} p_i \log p_i^{-1} &\le \sum_{i=1}^{\mk+1} p_i \log q_i^{-1} \\
			\sum_{i=2}^{n} \Gap{i}^{-2} \arment(I) &\le \sum_{i=1}^{\mk} H_i \log q_i^{-1} \\
			\end{align*}
		}
		
		Let $C = \log \frac{\pi^2}{6}$.
		So now we only need provide an upper bound for $\sum_{i=1}^{\mk} H_i \log q_i^{-1}$.
		\footnote{
			For empty groups, we adopt the convention $0\log 0 = 0$.
		}
		\begin{equation}\label{eq:newbd}
		\sum_{i=1}^{\mk} H_i \log q_i^{-1} = \sum_{t=1}^{h} \sum_{i=l_t}^{r_t} H_i (2\log t + (r_t-i+1) + C). 
		\end{equation}
		Now, consider the following quantity $U$.	\begin{equation*}\label{eq:upbd}
		U=\sum_{t=1}^{h} \sum_{i=l_t}^{r_t} H_i(C + 1 + \sum_{j=1}^{i-1} 4^{j-i+1}(i-j+1) + 2\log t). 
		\end{equation*}
		
		We claim that 
		$\sum_{i=1}^{\mk} H_i \log q_i^{-1}\leq U$.
		We first see that the proposition is an 
		easy consequence of the claim. 
		Since $\sum_{j=1}^{i-1} 4^{j-i+1}(i-j+1)=O(1)$,
		we have  
		$U=O\left(\sum_{t=1}^{h} \sum_{i=l_t}^{r_t} H_i(1 + \log t))\right)$. 	
		Note that $t \le \min(h,i)$. So $U$ 
		can be further bounded by 
		$O(\sum_{i=1}^{\mk} H_i (1+\log\min(\log n,i)))$, which is exactly
		$O\left(\sum_{i=2}^{n} \Gap{i}^{-2}(1+\ln\ln\min(n,\Gap{i}^{-1})) \right) $. 
		Now, the only remaining task is to prove the 
		claim. We first see that
		\begin{align}
		&\sum_{t=1}^{h} \sum_{i=l_t}^{r_t}  H_i \left(C + 1  + \sum_{j=1}^{i-1} 4^{j-i+1}(i-j+1) + 2\log t \right) \notag\\
		\ge&\sum_{t=1}^{h} \sum_{i=l_t}^{r_t} \left( H_i(C + 1 + 2\log t) + \sum_{k=i+1}^{\mk} H_k \cdot 4^{i-k+1}(k-i+1) \right)\notag\\
		\ge&\sum_{t=1}^{h} \sum_{i=l_t}^{r_t} \left(H_i(C + 1  + 2\log t) + \sum_{k=i+1}^{\mk} S_k 4^{k-1}\cdot 4^{i-k+1}(k-i+1)\right)\notag\\
		\ge&\sum_{t=1}^{h} \sum_{i=l_t}^{r_t} \left(H_i(C + 1  + 2\log t) + 4^i\cdot\sum_{k=i+1}^{\mk} S_k(k-i+1) \right) \label{eq:last-line}
		\end{align}
		
		For each $i$ such that $l_t \le i < r_t$,
		by the construction of the blocks, we have $S_i \le \sum_{j=r_t}^{\mk} S_j$. 
		So we have $4^i\sum_{k=r_t}^{\mk} S_k(k-i+1) \ge 4^i \sum_{k=r_t}^{\mk} S_k(r_t-i+1) \ge  4^i S_i(r_t-i+1) \ge H_i(r_t-i+1) $. Hence,
		each term in \eqref{eq:last-line} is no less
		than the corresponding term in \eqref{eq:newbd}. (It is trivially true for $i = r_t$).
		This concludes the proof.
	\end{proofof}


\section{Missing Proofs in Section~\ref{sec:bestarmupperbound}}
\label{app:missingpf-upper}

\subsection{Proof for Lemma~\ref{lm:ESTIMATE-PROCEDURE}}

\vspace{0.2cm}
\noindent
{\bf Lemma~\ref{lm:ESTIMATE-PROCEDURE} }
(restated)
{\em
Suppose $\epsilon < 0.1$ and $t \in (\epsilon,1-\epsilon)$. With probability $1-\delta$,
the following hold:
\begin{itemize} [leftmargin=*]
	\item If {\em\FRACTIONTEST} outputs $\True$, then $|S^{>c_r}| < (1-t+\epsilon) |S|$
	(or equivalently $|S^{\le c_r}| > (t-\epsilon)|S|$).
	\item If {\em\FRACTIONTEST} outputs $\False$, then $|S^{<c_l}| < (t+\epsilon) |S|$
	(or equivalently $ |S^{\ge c_l}| > (1-t-\epsilon)|S|$).
\end{itemize}
Moreover the number of samples taken by the algorithm is $O(\ln\delta^{-1}\epsilon^{-2}\Delta^{-2} \ln \epsilon^{-1})$, in which $\Delta = c_r-c_l$.
}

\begin{proofof}{Lemma~\ref{lm:ESTIMATE-PROCEDURE}}
	
	Let $S_a = S^{<c_l}$, $ S_b = S^{>c_r}$, $N_a = |S_a|$, $N_b = |S_b|$, $N = |S|$.
	For each iteration $i$ and the arm $a_i$ in line 4, by Lemma~\ref{lm:UNIFORM-SAMPLE-PROCEDURE}, 
	we have $\Pr[|\hamean{a_i}-\amean{a_i}| \ge  (c_r-c_l)/2] \le \epsilon/3$. 
	Hence, if $\amean{a_i} < c_l$, then $\Pr[\hamean{a_i} < \frac{c_l+c_r}{2}] \ge 1-\epsilon/3$. 
	Similarly, if $\amean{a_i} > c_r$, then $\Pr[\hamean{a_i} < \frac{c_l+c_r}{2}] \le \epsilon/3$.
	
	Let $X_i$ be the indicator Boolean variable $\indicator\left\{\hamean{a_i} < \frac{c_l+c_r}{2} \right\}$. 
	Clearly $X_i$s are i.i.d.
	From the algorithm, we can see that $\cnt = \sum_{i=1}^{\tot} X_i$.
	Let $E=\Ex[X_i]$. Let $\hat{E} = \cnt/\tot$, which is the empirical value of $E$.
	By Chernoff bound, we can easily get that	
	$
	\Pr[|E-\hat{E}| \ge \epsilon/3] \le \delta.
	$
	
	In the rest of the proof, 
	we condition on the event that $|E-\hat{E}| < \epsilon/3$, which happens with probability at least $1-\delta$.
	Suppose $\hat{E} > t$. Then we have
	$E > t - \epsilon/3$.
	It is also easy to see that:
	\begin{align*}
	E\cdot N &= \sum_{a \in S} \Pr\left[\hamean{a} < \frac{c_l+c_r}{2}\right] 
	\le (N-N_b)\cdot 1 + (\epsilon/3)N_b = N - (1-\epsilon/3)N_b, 
	\end{align*}
	as $\epsilon < 0.1$, $\frac{1}{1-\epsilon/3} < (1+2\epsilon/3)$. So, we have proved the first claim:
	\[
	N_b \le \frac{(1-E)N}{1-\epsilon/3} \le \frac{(1-t+\epsilon/3)N}{1-\epsilon/3} < (1-t+\epsilon/3)(1+2\epsilon/3)N \le (1-t+\epsilon)N.
	\]
	
	The second claim is completely symmetric. Suppose $\hat{E} \le t$. Then we have $E \le t+\epsilon/3$.
	We also have $E\cdot N \ge N_a(1-\epsilon/3)$. So,
	\[
	N_a \le \frac{E \cdot N}{1-\epsilon/3}\le \frac{(t+\epsilon/3)N}{1-\epsilon/3} < (t+\epsilon/3)(1+2\epsilon/3)N \le (t+\epsilon) N.
	\]
	Finally, the upper bound for the number of samples can be verified by a direct calculation.
\end{proofof}

\subsection{Proof for Lemma~\ref{lm:ELIMINATION-PROCEDURE}}
\vspace{0.2cm}
\noindent
{\bf Lemma~\ref{lm:ELIMINATION-PROCEDURE} }
(restated)
{\em
	Suppose $\delta < 0.1$.
	Let $S' = \ELIMINATION(S,c_l,c_r,\delta)$. 
	Let $A_1$ be the best arm among $S$, with mean $\amean{A_1} \ge c_r$.
	Then with probability at least $1-\delta$, the following statements hold
	\begin{enumerate} [leftmargin=*]
		\item $A_1 \in S'$ \,\, (the best arm survives);
		\item $|S'^{\le c_l}| < 0.1 |S'|$ \,\, (only a small fraction of arms have means less than $c_l$);
		\item The number of samples is $O(|S| \ln \delta^{-1} \Delta^{-2})$, in which $\Delta = c_r-c_l$.
	\end{enumerate}
	Note that with probability at most $\delta$, there is no guarantee for any of the above statements.
}

Before proving Lemma~\ref{lm:ELIMINATION-PROCEDURE}, we first describe two events (which happen with high probability) 
that we condition on, in order to simplify the argument.

\begin{lemma}\label{GOOD1}
	With probability at least $1-\delta/2$, it holds that
	in all round r, \FRACTIONTEST\ outputs correctly, and $|\hamean{A_1}-\amean{A_1}| < \frac{c_r-c_m}{2}$.
\end{lemma}

\begin{proof}
	Fix a round $r$. 
	\FRACTIONTEST\ outputs incorrectly with probability at most $\delta_r$.
	By Theorem~\ref{lm:UNIFORM-SAMPLE-PROCEDURE}, $\Pr(|\hamean{A_1}-\amean{A_1}| \ge \frac{c_r-c_m}{2}) \le \delta_r$.
	
	The lemma follows from a simple union bound over all rounds: 
	$
	2\sum_{r=1}^{+\infty} \delta_r \le 2\delta\sum_{r=1}^{+\infty}0.1/2^r \le \delta/2.
	$
\end{proof}

\begin{lemma}\label{GOOD2}
	Let $N_r = |S_{r}^{\le c_m}|$. 
	Then with probability at least $1-\delta/2$, 
	for all rounds $r$ in which Algorithm~\ref{algo:ELIMINATION-PROCEDURE} does not terminate, $N_{r+1} \le \frac{1}{4} N_r$.
\end{lemma}

\begin{proof}
	Suppose $a \in S_r^{\le c_m}$. 
	By Theorem~\ref{lm:UNIFORM-SAMPLE-PROCEDURE}, we have that 
	$\Pr[|\hamean{a} - \amean{a}| \ge \frac{c_r-c_m}{2}] \le \delta_r$. 
	So $\Pr[\hamean{a} > \frac{c_m+c_r}{2}] \le \delta_r$.
	Then, we can see that $\Ex[N_{r+1}] \le \delta_rN_{r}$. 
	By Markov inequality, we can see that $\Pr(N_{r+1} > \frac{1}{4} N_r) \le \frac{\delta_r N_r}{\frac{1}{4}N_r} = 4\delta_r$.
	
	Again, the lemma follows by a simple union bound: 
	$
	\sum_{r=1}^{+\infty} 4\delta_r \le 4\delta\sum_{r=1}^{+\infty} 0.1/2^r \le \delta/2.
	$
\end{proof}

\begin{proofof}{Lemma~\ref{lm:ELIMINATION-PROCEDURE}}
	With probability at least $1-\delta$, both statements in Lemma~\ref{GOOD1} and Lemma~\ref{GOOD2} hold.
	Let that event be $\event_G$. Now we prove Lemma~\ref{lm:ELIMINATION-PROCEDURE} under the condition that $\event_G$ holds.
	Now we prove all the claims one by one.
	
	Consider the first claim.
	Note that for all $r$, conditioning on $\event_G$, 
	we have that $|\hamean{A_1} - \amean{A_1}| < \frac{c_r-c_m}{2}$, 
	or equivalently $\hamean{A_1} > c_r - \frac{c_r-c_m}{2} = \frac{c_m+c_r}{2}$. 
	Hence, $A_1$ survives all rounds and $A_1 \in S'$.
	
	For the second claim, we note that, conditioning on event $\event_G$, $\textrm{\FRACTIONTEST}$ always outputs correctly. 
	Suppose the algorithm terminates at round $r$, which means $\FRACTIONTEST(S_r,c_l,c_m,\delta_r,0.075,0.025)$ outputs $\False$. 
	By Lemma~\ref{lm:ESTIMATE-PROCEDURE}, we have $|S_{r}^{<c_l}| < (0.075+0.025) |S_r| = 0.1 |S_r|$. 
	Since $S'=S_r$, the claim clearly follows.
	
	Now, we prove the last claim. Again we condition on $\event_G$.
	Suppose the algorithm does not terminate at round $r$, 
	which means $\FRACTIONTEST(S_r,c_l,c_m,\delta_r,0.075,0.025)$ outputs $\True$. 
	By Lemma~\ref{lm:ESTIMATE-PROCEDURE}, we know $N_r = |S_r^{\le c_m}| > (0.075-0.025)|S_r| = 0.05|S_r|$. 
	Then, we have that
	$$
	|S_{r+1}| \le |S_r| - (|N_r| - |N_{r+1}|) \le |S_r| - \frac{3}{4} |N_r| \le 0.99 |S_r|.
	$$
	
	Suppose the algorithm terminates at round $r'$. 
	Let $c_1$ be a large enough constant (so that $c_1 \ln \delta_{r} \Delta^{-2} |S_r|$ is an upper bound for the samples taken by \UNIFORMSAMPLING \ in round $r$
	and $c_1 \ln \delta_{r} \Delta^{-2}$ is an upper bound for the samples taken by \FRACTIONTEST \ in round $r$).
	Then, the number of samples is bounded by:
	
	\begin{align*}
	\sum_{r=1}^{r'} c_1 ( \Delta^{-2} \ln \delta_{r} |S_r| +  \Delta^{-2} \ln\delta_{r})
	&\le 2c_1|S|\Delta^{-2}\sum_{r=1}^{r'} (\ln\delta^{-1} + r\ln2 + \ln 10) \cdot 0.99^{r-1}\\
	&\le 2c_1|S|\Delta^{-2}\sum_{r=1}^{r'} (\ln\delta^{-1}(r+1) + \ln 10) \cdot 0.99^{r-1}\\
	&\le 2c_1|S|\Delta^{-2}\left(\ln \delta^{-1}\sum_{r=1}^{+\infty}(r+1) \cdot 0.99^{r-1} + \sum_{r=1}^{+\infty} \ln 10 \cdot 0.99^{r-1} \right)\\
	&\le 2c_1|S| \Delta^{-2} \left(\ln \delta^{-1} \cdot 10100 + 100 \cdot \ln 10\right)
	\end{align*}    
	So the number of samples is $O(|S|\ln \delta^{-1} \Delta^{-2})$, which concludes the proof of the lemma.
	
\end{proofof}

\subsection{Proofs for Lemma~\ref{lm:case-1}, Lemma~\ref{lm:case-2} and Lemma~\ref{lm:case-3}}

For convenience, we let 
$$\curarm = \armset{r} \cap S_r, \curbigarm = \bigarmset{r+1} \cap S_r,\Ncur = |\curarm|,\Nbig = |\curbigarm|.$$

Then we have $|S_r| = \Ncur + \Nbig + \Nsmall + 1$. Also, recall that $\lx{s} = \ln(\min(\hI,s))$.

In order to prove these three lemmas, we need an important inequality for $\Pot{r}$. If $r \le \maxs$, we have:
\begin{align}
\Pot{r} - \Pot{r+1} &= c_2 \cdot \left(\sum_{s=r}^{+\infty} (\ln \delta^{-1} + \lx{s})\cdot \epsilon_{r}^{-2} |\armset{s}| 
+ (\ln\delta^{-1} + \lx{\maxs})\cdot \epsilon_{r}^{-2}\right) \notag \\
&\ge c_2 \cdot \epsilon_{r}^{-2}(\ln \delta^{-1} + \lx{r})(|\bigarmset{r}| + 1)\notag\\
&\ge c_2 \cdot \epsilon_{r}^{-2}(\ln \delta^{-1} + \lx{r})(\Ncur +\Nbig + 1). \label{eq:SUMCRS}
\end{align}

\vspace{0.2cm}
\noindent
{\bf Lemma~\ref{lm:case-1} }
(restated)
{\em
	When \MEDIANELIMINATION\ (line~\ref{line:ME1}) returns an $\epsilon_r/4$-approximation of the best arm $A_1$, and \FRACTIONTEST\ outputs $\True$. The expected number of samples taken at and after round $r$ is bounded by
	$$
	\epsilon_r^{-2}(\ln\delta^{-1} +  \lx{r})c_3 \Nsmall + \Pot{r}.
	$$
}
\begin{proofof}{Lemma~\ref{lm:case-1}}

Since by Lemma~\ref{lm:CORRECT-MEDELIM}, we have $|\smaarmset{r} \cap S_{r+1}| \le 0.1 |S_{r+1}|$, 
which means 
$
|\bigarmset{r+1} \cap S_{r+1}| + 1 \ge 0.9 |S_{r+1}|
$. So, we have that
$$
|\smaarmset{r} \cap S_{r+1}| \le \frac{1}{9} (|\bigarmset{r+1} \cap S_{r+1}|+1) \le \frac{1}{9}(\Nbig+1).
$$ 
Therefore, we can see the number of samples is bounded by:
\[
c_3\cdot|S_r|\epsilon_r^{-2}(\ln\delta^{-1} +  \lx{r}) + T\left(r+1,\frac{1}{9}(\Nbig+1)\right),
\]	
where the first additive term is the number of samples in this round, and is bounded by Lemma~\ref{lm:PULLING-BOUND}.

By the induction hypothesis, we have:
\begin{align*}
&T\left(r+1,\frac{1}{9}(\Nbig+1)\right) \\
\le&(\ln\delta^{-1}+\lx{r+1})\cdot c_1 \cdot \frac{1}{9}(\Nbig+1)\cdot\epsilon_{r+1}^{-2} +\Pot{r+1}\\ 
\le&(\ln\delta^{-1}+\lx{r+1})\cdot c_1 \cdot \frac{1}{9}(\Nbig+1)\cdot\epsilon_{r+1}^{-2}+\Pot{r} - c_2\cdot(\ln \delta^{-1}+\lx{r})(\Nbig+\Ncur+1)\epsilon_r^{-2} 
\tag{By \eqref{eq:SUMCRS}} \\ 
\le&(\ln\delta^{-1}+\lx{r})\cdot \epsilon_r^{-2}\left(c_1 \frac{5}{9}\Nbig + \frac{5}{9}c_1- c_2\Ncur - c_2\Nbig-c_2\right)+ \Pot{r} 
\end{align*}
Therefore, we can bound the expected number of samples by:
\begin{align*}
 &\epsilon_r^{-2}(\ln\delta^{-1} +  \lx{r})\left(c_3\cdot|S_r| + c_1 \frac{5}{9}\Nbig + \frac{5}{9}c_1- c_2\Ncur - c_2\Nbig -c_2\right) + \Pot{r} \\
\le&\epsilon_r^{-2}(\ln\delta^{-1} +  \lx{r})\left((c_3-c_2)\Ncur+(c_3+\frac{5}{9}c_1-c_2)\Nbig+c_3 \Nsmall +\frac{5}{9}c_1+ c_3-c_2\right)+ \Pot{r}  \\
\le&\epsilon_r^{-2}(\ln\delta^{-1} +  \lx{r})c_3 \Nsmall + \Pot{r} \tag{$\frac{5}{9}c_1+c_3 - c_2 < 0$ , $c_3+\frac{5}{9}c_1-c_2 < 0$} 
\end{align*}
In the first inequality, we use the fact that $|S_r| = \Ncur + \Nbig + \Nsmall + 1$.

\end{proofof}

\vspace{0.2cm}
\noindent
{\bf Lemma~\ref{lm:case-2} }
(restated)
{\em
	When 
	\MEDIANELIMINATION\ (line~\ref{line:ME1}) returns an $\epsilon_r/4$-approximation of the best arm $A_1$, and \FRACTIONTEST\ outputs $\False$. The expected number of samples taken at and after round $r$ is bounded by $\Pot{r}$.
}
\begin{proofof}{Lemma~\ref{lm:case-2}}
By Lemma~\ref{lm:CORRECT-MEDELIM}, we can see $|\bigarmset{r} \cap S_r| + 1 = \Ncur+\Nbig + 1 > 0.5 |S_r|$. So $\Nsmall < 0.5 |S_r|$, 
thus 
\begin{align}
\label{eq:DEM}
\Ncur+\Nbig +1 \ge \Nsmall.
\end{align} 
We can see that the total number of samples is bounded by:
$$
c_3 \cdot |S_r|\epsilon_{r}^{-2} + T(r+1,\Nsmall+\Ncur) \le 
(\ln \delta^{-1} + \lx{r})c_3 \cdot |S_r|\epsilon_{r}^{-2} + T(r+1,\Nsmall+\Ncur).
$$	
In the above, the first term is due to the number of samples in this round by Lemma~\ref{lm:PULLING-BOUND}. 
The second term is an upper bound for the number of samples 
starting at round $r+1$.
Since we do not eliminate any arm in this case, we have $|\smaarmset{r} \cap S_{r+1}|\le \Nsmall + \Ncur$.

\noindent
From the induction hypothesis, we have:
\begin{align*}
\label{eq:Trplus1}
& T(r+1,\Nsmall+\Ncur)\\
\le&(\ln\delta^{-1}+\lx{r+1})\cdot c_1 \cdot (\Nsmall+\Ncur)\cdot\epsilon_{r+1}^{-2} + \Pot{r+1}\notag\\
\le&(\ln\delta^{-1}+\lx{r})\cdot c_1 \cdot 5(\Nsmall+\Ncur)\cdot\epsilon_{r}^{-2}+\Pot{r}- c_2\cdot(\ln \delta^{-1}+\lx{r})(\Nbig+\Ncur+1)\epsilon_r^{-2} \tag{By (\ref{eq:SUMCRS})}\\
\le&(\ln\delta^{-1}+\lx{r})(5c_1 \Nsmall + 5 c_1 \Ncur - c_2\Ncur - c_2\Nbig-c_2)\epsilon_{r}^{-2} + \Pot{r}
\end{align*}

\noindent
Plugging it into the bound, we have the following bound for the expected number of samples:
\begin{align*}
&(\ln \delta^{-1} + \lx{r})c_3 \cdot |S_r|\epsilon_{r}^{-2} + (\ln\delta^{-1}+\lx{r})(5c_1 \Nsmall + 5 c_1 \Ncur - c_2\Ncur - c_2\Nbig-c_2)\cdot\epsilon_{r}^{-2}  + \Pot{r} \\
\le&(\ln \delta^{-1} + \lx{r})((5c_1 +c_3)\Ncur+c_3\Nbig+c_3+(5c_1+c_3) \Nsmall - c_2(\Ncur + \Nbig+1))\cdot\epsilon_{r}^{-2} + \Pot{r}\\
\le&(\ln \delta^{-1} + \lx{r})(c_3\Nbig+c_3+ (5 c_1+c_3) \Ncur - (c_2-5c_1-c_3)(\Ncur + \Nbig + 1))\cdot\epsilon_{r}^{-2} + \Pot{r}\\
\le& \Pot{r} \tag{$c_2-5c_1-c_3 > 5c_1+c_3 > c_3$}
\end{align*}
In the second inequality, we use the fact that $(5c_1+c_3) \Nsmall \le (5c_1+c_3)(\Ncur+\Nbig+1)$ due to \eqref{eq:DEM}.
\end{proofof}

\vspace{0.2cm}
\noindent
{\bf Lemma~\ref{lm:case-3} }
(restated)
{\em
	When 
	\MEDIANELIMINATION\ (line~\ref{line:ME1}) returns an arm which is not an $\epsilon_r/4$-approximation of the best arm $A_1$. The expected number of samples taken at and after round $r$ is bounded by
	$$
	\le(\ln\delta^{-1}+\lx{r})(c_3+5c_1) \cdot \Nsmall\epsilon_{r}^{-2} + \Pot{r}.
	$$
}
\begin{proofof}{Lemma~\ref{lm:case-3}}
In this case, we can simply bound it by:
\[
c_3\cdot|S_r|\epsilon_r^{-2}(\ln\delta^{-1}+\lx{r}) + T(r+1,\Nsmall+\Ncur).
\]
The first term is still due to the number of samples in this round by Lemma~\ref{lm:PULLING-BOUND}. The second term is an upper bound for the samples taken starting at round $r+1$, since $|\smaarmset{r} \cap S_{r+1}|\le \Nsmall + \Ncur$ in any case.
Then, we have that
\begin{align*}
& c_3\cdot|S_r|\epsilon_r^{-2}(\ln\delta^{-1}+\lx{r}) + T(r+1,\Nsmall+\Ncur) \notag\\
\le & (\ln\delta^{-1}+\lx{r})(c_3(\Nsmall+\Ncur+\Nbig+1)\epsilon_r^{-2} + 5c_1(\Nsmall+\Ncur) \epsilon_{r}^{-2} -c_2 \epsilon_{r}^{-2}(\Nbig+\Ncur+1)) + \Pot{r} 
\tag{By (\ref{eq:SUMCRS})}\\
\le & (\ln\delta^{-1} + \lx{r})((c_3+5c_1-c_2)\Ncur + (c_3-c_2)\Nbig +(c_3+5c_1)  \Nsmall + c_3 -c_2)\epsilon_{r}^{-2} + \Pot{r} 
\end{align*}

Since $c_3+5c_1-c_2 \le 0$ and $c_3-c_2 \le 0$, we have the following bound for the expected number of samples:
$$
(\ln\delta^{-1}+\lx{r})(c_3+5c_1) \cdot \Nsmall\epsilon_{r}^{-2} + \Pot{r}
$$
\end{proofof}

\section{More About \sign}
\label{app:sign}

In this section, we present a class of 
\CORRECT\ algorithms for \sign\ which 
needs  $o(\Delta^{-2}\ln\ln \Delta^{-1})$ samples for 
infinite number of instances.
In particular, we show the following stronger result.

\begin{theo}\label{thm:FAST-ALGO}
	For any function $T$ on $(0,1]$ such that $\limsup_{\Delta \to +0} T(\Delta)\Delta^{2} = +\infty$ and 
	for any fixed constant $\delta >0$, there exists a \CORRECT\ algorithm $\alg$ for \sign,       
	such that 
	$$
	\liminf_{\Delta \to +0} \frac{T_{\alg}(\Delta)}{T(\Delta)} = 0.
	$$ 
\end{theo}

Now, we begin our description of the algorithm, 
which is in fact a simple variant of the \EXPGAPELIMINATION\ algorithm in \cite{karnin2013almost}. 
Our algorithm takes an infinite sequence $\seqn=\{\refseq_i\}_{i=1}^{+\infty}$ as input, which we call
the {\em reference sequence}.

\begin{defi}
	We say an infinite sequence $\seqn=\{\refseq_i\}_{i=1}^{+\infty}$ is a {\em reference sequence} if 
	the following statements hold:    
	\begin{enumerate}
		\item $0<\refseq_i<1$, for all $i$.
		\item There exists a constant $0<c<1$ such that for all $i$, $\refseq_{i+1} \le c \cdot \refseq_{i}$.
	\end{enumerate}
\end{defi}

Our algorithm \TESTSIGN\ takes a confidence level $\delta$ and the reference 
sequence $\{ \refseq_i \}$ as input. It runs in rounds.
In the $r$th round, 
the algorithm takes a number of samples (the actual number depends on $r$, and can be found in 
Algorithm~\ref{algo:FAST-ALGO-SIGN}) from the arm and let $\widehat{\mu}^r$ be the empirical mean.
If $\widehat{\mu}^r\in \xi\pm \refseq_r/2$, we decide that the gap $\Delta$ is smaller than the reference gap $\refseq_r$
and we should proceed to the next round with a smaller reference gap.
If $\widehat{\mu}^r$ is larger than $\xi + \refseq_r/2$,
we decide $\mu>\xi$.
If $\widehat{\mu}^r$ is smaller than $\xi - \refseq_r/2$,
we decide $\mu<\xi$.
The pseudocode can be found in algorithm~\ref{algo:FAST-ALGO-SIGN}.

\begin{algorithm}[H]
	\LinesNumbered
	\setcounter{AlgoLine}{0}
	\caption{\TESTSIGN($A,\delta,\{ \refseq_i \}$)}
	\label{algo:FAST-ALGO-SIGN}
	\KwData{The single arm $A$ with unknown mean $\mu \ne \xi$, confidence level $\delta$, the reference 
		sequence $\{ \refseq_i \}$.}
	\KwResult{Whether $\mu > \xi$ or $\mu < \xi$.}
	\smallskip
	\For{r = 1 to $+\infty$}{
		$\epsilon_r = \refseq_r/2$
		
		$\delta_{r} = \delta/10r^2$
		
		Pull $A$ for $t_r = 2\ln(2/\delta_r)/\epsilon_r^2$ times. Let $\widehat{\mu}^{r}$ denote its average reward.
		
		\uIf{$\widehat{\mu}^r > \xi + \epsilon_r$}{
			{\bf Return} $\mu > \xi$
		}
		
		\uIf{$\widehat{\mu}^r < \xi - \epsilon_r$}{
			{\bf Return} $\mu < \xi$
		}
	}
\end{algorithm}

The algorithm can achieve the following guarantee.
The proof is somewhat similar to the analysis of the \EXPGAPELIMINATION\ algorithm in \cite{karnin2013almost}
(in fact, simpler since there is only one arm).

\begin{lemma}\label{thm:behavior-fast-algo}
	Fix a confidence level $\delta > 0$ and an arbitrary reference sequence $\seqn=\{\refseq_i\}_{i=1}^{\infty}$. Suppose that the given instance has a gap $\Delta$.
	Let $\terround$ be the smallest $i$ such that $\refseq_i\le\Delta$.
	With probability at least $1-\delta$,
	\TESTSIGN\ determines whether 
	$\mu > \xi$ or $\mu < \xi$ correctly and uses
	$O((\ln\delta^{-1}+\ln \terround)\refseq_{\terround}^{-2})$ samples in total.
\end{lemma}

\begin{proof}
	For any round $r$, by Hoeffding's inequality (Lemma~\ref{lm:hoeff}), we have that
	\begin{equation}\label{HOFFEFFDING-1}
	\Pr\left(|\widehat{\mu}^{r} - \mu| \ge \epsilon_{r}\right) \le 2\exp(-\epsilon_r^2/2\cdot t_r) = \delta_r.
	\end{equation}
	Then, by a union bound, with probability 
	$1-\sum_{i=1}^{+\infty} \delta_r = 1 - \delta \cdot \sum_{i=1}^{+\infty} 1/10r^2\ge 1-\delta$, 
	we have $|\widehat{\mu}^{r} - \mu| < \epsilon_{r}$ for all $r$. 
	Denote this event by $\event$.
	Then we prove that conditioning on $\event$, Algorithm~\ref{algo:FAST-ALGO-SIGN} is correct.
	
	Let $k$ be the round the algorithm returns the answer.
	By the definition of $\terround$, we know that $\refseq_{\terround} \le \Delta$. 
	Then on round $\terround$, we have $|\widehat{\mu}^{\terround} - \mu| < \refseq_{\terround}/2 \le \Delta/2$.
	Thus, $|\widehat{\mu}^{\terround} - \xi| \ge |\mu - \xi| - |\widehat{\mu}^{\terround} - \mu| > \Delta/2 \ge \epsilon_{r}$.
	Therefore, we can see that $k \le \terround$, which shows that 
	the algorithm terminates on or before round $\terround$.
	On round $k$, if we have $\widehat{\mu}^{k} > \xi + \epsilon_r$, 
	we must have $\mu > \xi$ since $|\widehat{\mu}^{k} -\mu|<\epsilon_r$. 
	The case $\widehat{\mu}^{k} < \xi - \epsilon_r$ is completely symmetric,
	which proves the correctness.
	
	Now, we analyze the number of samples.
	It is easy to see that the total number of samples is at most:
	\[
	\sum_{r=1}^{\terround} t_r = \sum_{r=1}^{\terround} 8\ln(2/\delta_r)/\refseq_{r}^2.
	\]
	By the definition of the reference sequence,
	$\refseq_{r} \ge c^{r-\terround} \refseq_{\terround}$ for $1 \le r \le \terround$.
	Hence, we have that
	\begin{align*}
	\sum_{r=1}^{\terround} 8(\ln \delta^{-1} + \ln 20 +\ln r)/\refseq_{r}^2 \, 
	&\le\,\refseq_{\terround}^{-2}\sum_{r=1}^{\terround} c^{2(\terround-r)} \cdot 8(\ln \delta^{-1} + \ln 20 +\ln r) \\
	&=\, O\left((\ln \delta^{-1} + \ln \terround) \refseq_{\terround}^{-2}\right)
	\end{align*}
	This finishes the proof of the theorem.
\end{proof}

Finally, we prove Theorem~\ref{thm:FAST-ALGO}.

\begin{proofof}{Theorem~\ref{thm:FAST-ALGO}}
	First, we can easily construct a reference sequence $\{ \refseq_i \}$ such that
	\begin{enumerate}
		\item $0 < \refseq_i < 1$ and $\refseq_{i+1} \le \refseq_i/2$ for all $i$.
		\item $T(\refseq_i) \ge i\cdot \refseq_i^{-2}$ for all $i$
		(this is possible since $\limsup_{\Delta \to +0} T(\Delta)/\Delta^{-2} = +\infty$).
	\end{enumerate}
	With this reference sequence, we can see 
	that with probability $1-\delta$,
	the algorithm outputs the correct answer and 
	runs in $O(\ln\delta^{-1}+\ln \kappa)\refseq_\kappa^{-2}$ time.
	However, there is a subtlety here:
	the expected running time is not bounded since we do not have 
	a bound with probability $\delta$ (when the good event $\event$ in Lemma~\ref{thm:behavior-fast-algo} does not happen).
	In Theorem~\ref{theo:TRANSFORM1}, we provide a general transformation that can produce a 
	\CORRECT\ algorithm whose expected running time is
	$O(\ln\delta^{-1}+\ln \kappa)\refseq_\kappa^{-2}$.
	Let the algorithm be $\alg$. 
	
	For any fixed $\delta$, we can see that
	\[
	\lim_{i \to +\infty} \frac{T_\alg(\refseq_i)}{T(\refseq_i)} \le 
	\lim_{i \to +\infty} \frac{C(\ln \delta^{-1} + \ln i) \refseq_i^{-2}}{i \refseq_i^{-2}} = 0,
	\]
	where $C$ is some large constant. 
	This implies that
	$
	\liminf_{\Delta \to +0}  T_\alg(\Delta)/T(\Delta) = 0.
	$
\end{proofof}

\begin{rem}
	If we use the reference sequence $\{ \refseq_i = e^{-i} \}$, we have an \CORRECT\ algorithm 
	for \sign, and it takes $O(\Delta^{-2}(\ln\delta^{-1} +  \ln\ln\Delta^{-1}))$ samples in expectation on instance with gap $\Delta$.
\end{rem}

\begin{rem}
	Recall Farrell's (worse case) lower bound 
	\eqref{eq:2armlowerbound}
	is $\Omega(\Delta^{-2}\ln\ln \Delta^{-1})$.
	Together with
	Theorem~\ref{thm:FAST-ALGO}, they
	imply that 
	it is impossible to obtain an instance optimal algorithm for \sign. 
\end{rem}

\newcommand{\eventf}{\mathcal{F}}

\section{\CORRECT\ Algorithms and Parallel Simulation}
\label{sec:simult}

In Corollary~\ref{cor:final-time-bound},
we show that our algorithm can output
the correct answer with probability at least $1-\delta$, 
and the conditional expected running time is upper bounded.
However, with probability $\delta$, there is no guarantee on its behavior (e.g., it could potentially run forever). 
\footnote{
	Many previous algorithms for pure exploration multi-armed bandits 
	belong to this kind, such as \cite{even2006action,karnin2013almost,gabillon2012best}.
	Some previous work has noticed the issue as well, but we are not aware of a systematic way to deal with it.
	For example, \cite{kalyanakrishnan2012pac} stated that 
	{\em ``However, their (i.e., \cite{even2006action}) elimination algorithm could incur high
	sample complexity on the $\delta$-fraction of the runs on
	which mistakes are made---we think it unlikely that
	elimination algorithms can yield an expected sample
	complexity bound smaller than $...$ ".} 
	}
Hence, the expected running time may be unbounded.
It is preferable to have an algorithm with a bounded expected running time,
and being correct with probability at least $1-\delta$.
In this section, we provide such a transformation that given an algorithm for the former
kind, produces one of the later (the time bound only increases by a constant factor).

Now, we formally define the two kinds of algorithms.

\begin{defi}\label{PRE-PAC}
	Let $\alg$ be an algorithm for some problem $\Problem$.
	$\alg$ takes an additional input $\delta$ as the confidence level. 
	Let $\mathcal{I}$ be the set of all valid instances for $\Problem$.
	We write $\alg_\delta$ to denote the algorithm $\alg$ with a fixed confidence level $\delta$.
	
	\begin{enumerate}
		
		\item
		We call $\alg$ an \EXPCORRECTS{T} algorithm iff there exists $\delta_0 \in (0,1)$ such that for any $\delta \in (0,\delta_0)$ and instance $I \in \mathcal{I}$:
		\[
		T_{\alg_\delta}[I] \le T(\delta,I) \quad \text{ and } \quad
		\Pr[\alg_\delta \text{ returns the correct answer on } I] \ge 1-\delta.
		\]
		
		\item
		We call $\alg$ a \WEAKEXPCORRECTS{T} algorithm iff there exists $\delta_0 \in (0,1)$ such that for any $\delta \in (0,\delta_0)$ and instance $I \in \mathcal{I}$, there exists an event $\event$ that		
		\[
		\Pr_{\alg_\delta,I}[\event] \ge 1-\delta \wedge \Ex_{\alg_\delta,I}[\tau\ |\  \event] \le T(\delta,I) \,\,
		\text{ and }
		\,\,
		\Pr[\alg_\delta \text{ returns the correct answer on } I \ |\ \event]=1.
		\]
		We call the above event $\event$ a {\em good event}.
	\end{enumerate}
\end{defi}

We need a mild assumption on the running times for our general transformation.

\begin{defi}\label{defi:GOOD-FUNCTION}
	We say a function $T:(0,1) \times \mathcal{I} \to \R$ is a {\em reasonable time bound}, 
	if there exists $0<\delta_0<1$ such that for all $0<\delta' < \delta < \delta_0$ and $I \in \mathcal{I}$ we have that
	\[
	T(\delta',I) \le \frac{\ln \delta'^{-1}}{\ln \delta^{-1}}T(\delta,I).
	\]
\end{defi}

\begin{rem}
	The running times of most previous algorithms are of the form $\alpha(I)+\beta(I)\ln \delta^{-1}$,
	where $\alpha$ and $\beta$ only depend on $I$ (e.g., see Table 1).
	Such running time bounds are obviously reasonable.
\end{rem}

Suppose $\alg$ is a \WEAKEXPCORRECTS{T} algorithm.
Our strategy to produce a \EXPCORRECTS{O(T)} algorithm
is to simulate many copies of $\alg$ with different 
confidence levels.
Now we formally define the parallel simulation method.
Suppose we have a class of algorithms $\{\alg_i\}$ (possibly infinite) 
for the same problem $\Problem$. 
We want to construct a new algorithm $\newalg$ for problem $\Problem$
which simulates  $\{\alg_i\}$ in a parallel fashion.
The details of the construction are specified as follows. 

\begin{defi}(Parallel Simulation)\label{defi:sim-independent}
	The new algorithm $\newalg$ 
	simulates $\alg_i$ with rate $r_i \in \mathbb{N}^{+}$, for all $i$.
	More specifically, $\newalg$ runs in rounds.
	In the $r$-th round, $\newalg$ simulates each algorithm $\alg_i$ such that $r_i$ divides $r$ (i.e., $r_i|r$) for one step.
	If there are more than one such algorithms, $\newalg$ simulates them in the increasing order of their indices.
	If any such $\alg_i$ requires a sample, $\newalg$ takes a fresh sample and feeds it to $\alg_i$.

	$\newalg$ terminates whenever any $\alg_i$ terminates and $\newalg$ outputs what $\alg_i$ outputs.	
	We denote this new algorithm as $\newalg=\SIM(\{\alg_i\},\{r_i\})$.
\end{defi}

Now we prove the main result of this section.

\begin{theo}\label{theo:TRANSFORM1}
	Suppose $T$ is a reasonable time bound.
	If $\alg$ is a \WEAKEXPCORRECTS{T} algorithm
	for the problem $\Problem$, then there exists an algorithm $\newalg$ which is \EXPCORRECTS{O(T)}.
	Moreover, $\newalg$ can be constructed explicitly from $\alg$.
\end{theo}

\begin{proof}
	The construction of $\newalg$ is very simple:
	Let $\newalg_{\delta} = \SIM(\{\alg_i\},\{r_i\})$, in which $\alg_i = \alg_{\delta/2^i}$ 
	(that is algorithm $\alg$ with confidence parameter $\delta/2^i$)
	and $r_i = 2^i$. 
	
	Now we prove $\newalg$ is an \EXPCORRECTS{O(T)} algorithm for problem $\Problem$. Suppose the given instance is $I$. 
	Let $\event_i$ be a good event for $\alg_i$ on instance $I$.
	
	First, by a simple union bound, we can see that
	the probability that $\newalg_{\delta}$ outputs the correct answer is at least $1-\sum_{i=1}^{+\infty}\delta/2^i = 1-\delta$
	($\newalg_{\delta}$ returns the correct answer if all $\alg_i$ returns the correct answer).
	
	Now, assume $\delta < \delta_0'= \min(0.1,\delta_0)$. 
	Let us analyze the running time of $\newalg_{\delta}$ on instance $I$.
	
	For ease of argument, we can think that $\newalg$ executes
	in a slight different way.
	$\newalg$ does not terminate the same way as before, but keeps simulating all algorithms until all of them terminate 
	(or run forever). The output of $\newalg$ is still the same as the $\alg_i$ that terminates first, and the running time of $\newalg$ is determined by the first terminated $\alg_i$.
	
	Partition this probability space into disjoint events $\{ \eventf_i \}$, 
	in which $\eventf_i$ is the event that all events $\event_j$ with $j<i$ do not happen, and $\event_{i}$ happens.
	Note that $\{ \eventf_i \}$ is indeed a partition of the probability space ($\Pr[\cup_i \eventf_i]=1$).
	This is simply because $\lim_{i\rightarrow+\infty} \Pr[\event_i]\geq \lim_{i\rightarrow+\infty} 1-\delta/2^i=1$.
	
	Let $\tau$ be the running time of $\newalg$.
	Since each $\alg_i$ uses its own independent samples, we have that
	$$
	\Ex_{\alg_{i},I}[\tau \mid \eventf_i]=\Ex_{\alg_{i},I}[\tau \mid \event_i]  \le T(\delta/2^i,I).
	$$
	Moreover, for $\alg_i$ to run one step, for any $j \ne i$, $\alg_j$ runs at most $2^{i-j}$ steps. 
	Thus, the running time for $\newalg_{\delta}$ conditioning on $\eventf_{i}$ is bounded by:
	\[
	\Ex_{\newalg_{\delta},I}[\tau\ |\ \eventf_i] \le 
	T(\delta/2^i,I) \cdot \sum_{j=1}^{+\infty} 2^{i-j} \le 
	T(\delta/2^i,I) \cdot 2^{i}.
	\]
	Furthermore, by the independence of different $\alg_i$s, we note that 
	$$
	\Pr[\eventf_i] \le \prod_{k=1}^{i-1} \delta/2^k \le \delta^{i-1}.
	$$
	Now, we can bound the expected running time of $\newalg$ as follows:
	\begin{align*}
	\Ex_{\newalg_{\delta},I}[\tau] &= \sum_{i=1}^{+\infty} \Pr[\eventf_i] \cdot \Ex_{\alg_{\delta}',I}[\tau\ |\ \eventf_i] \\
	&\le \sum_{i=1}^{+\infty} \delta^{i-1} \cdot T(\delta/2^i,I)\cdot 2^i.\\
	&\le 2\sum_{i=1}^{+\infty} (2\delta)^{i-1} \cdot T(\delta,I)\left(\frac{\ln \delta^{-1} + i\ln2 }{\ln \delta^{-1}}\right).\\
	&\le 2\sum_{i=1}^{+\infty} (2\delta)^{i-1}(1+i) \cdot T(\delta,I).\\
	&\le 2\sum_{i=1}^{+\infty} (0.2)^{i-1}(1+i) \cdot T(\delta,I).\\	&\le 6 T(\delta,I).
	\end{align*}	
	In the second inequality, we use the fact that $T$ is a reasonable time bound.
	To summarize, we can see that $\newalg$ is an \EXPCORRECTS{O(T)} algorithm.
\end{proof}

\eat{
	\begin{cor}	
		Suppose $T$ is a reasonable time bound.
		If there is a \WEAKCORRECTS{T} algorithm $\alg$ for problem $\Problem$, there exists an \EXPCORRECTS{O(T)} algorithm $\newalg$ for $\Problem$.
		Moreover, $\newalg$ can be constructed explicitly from $\alg$.
	\end{cor}
}

\end{document}